\documentclass{article}
\usepackage[T1]{fontenc}    
\usepackage[utf8]{inputenc} 
\usepackage{lmodern}

\usepackage{PRIMEarxiv}
\usepackage{hyperref}       
\usepackage{url}            
\usepackage{booktabs}       
\usepackage{amsfonts}       
\usepackage{mathrsfs}
\usepackage{nicefrac}       
\usepackage{microtype}      
\usepackage{enumitem}      
\usepackage{subcaption}    
\usepackage{lipsum}
\usepackage{fancyhdr}       
\usepackage{graphicx}       
\graphicspath{{media/}}     
\usepackage{amsthm,amsmath,amssymb}
\usepackage{tikz}
\usetikzlibrary{arrows.meta,positioning}

\newtheorem{corollary}{Corollary}[section]

\newtheorem{proposition}{Proposition}[section]
\newtheorem{conjecture}{Conjecture}[section]

\usepackage[
backend=biber,
style=alphabetic,
sorting=ynt,
sortlocale=en
]{biblatex}
\title{Toward a Physical Theory of Intelligence}

\author{
  Peter David Fagan \\
  School of Informatics \\
  University of Edinburgh \\
  United Kingdom \\
  \texttt{p.d.fagan@ed.ac.uk} \\
}

\begin{document}
\raggedbottom
\maketitle

\begin{abstract}
While often treated as abstract algorithmic properties, intelligence and computation are ultimately physical processes constrained by conservation laws. We introduce the Conservation-Congruent Encoding (CCE) framework as a unified, substrate-neutral physical framework for studying intelligence. We propose that information processing emerges when open systems undergo irreversible transitions, carving out macroscopic states from underlying reversible micro-dynamics. Generalizing Landauer's principle to arbitrary conserved quantities via metriplectic flows, we derive a universal bound for macroscopic computation. This yields physical metrics for intelligence and an operational analogue for consciousness, quantifying an agent's ability to extract work from the environment while minimizing its own dissipative dynamics. Applying CCE to the limits of physical observation, we model measurement as an active coarse-graining process rather than a passive projection. At the quantum scale, CCE recovers the Lindblad Master Equation, consistent with modelling decoherence as the dissipative exhaust required to record a measurement. Scaling to cosmological limits, we explore the hypothesis that gravity emerges as the macroscopic geometric footprint of these bounds. We show that, under this hypothesis, measurement-induced dissipation is consistent with a volumetric phase-space collapse, offering a dynamical route to the Bekenstein-Hawking area law. Equating the Landauer exhaust of this coarse-graining to horizon deformation outlines a limiting-case recovery of the Einstein Field Equations. Ultimately, by establishing a substrate-neutral link between thermodynamic dissipation, quantum measurement, and spacetime geometry, CCE provides physical constraints for understanding both natural and artificial intelligence.
\end{abstract}

\section{Introduction}
\label{sec:Introduction}

Intelligence is typically studied through behavioural, information-theoretic, or algorithmic lenses, yet all intelligent systems are ultimately physical systems. They evolve in time, consume resources, store and transform information, and interact with their environments through structured flows of energy and matter. Despite this, there is no unified account linking physical laws to the computational properties associated with intelligence. Existing accounts rely primarily on behavioural benchmarks or information-theoretic abstractions, and therefore lack an explicit grounding in physical law, making principled comparison across heterogeneous systems and the derivation of fundamental limits difficult. Furthermore, existing physical theories have historically remained largely confined to thermodynamics, linking computation to the export of heat rather than unifying concepts across all conserved quantities. Consequently, fields such as neuroscience, artificial intelligence, and quantum mechanics have developed disparate theoretical languages to describe how systems encode memory, process information, and exert causal influence over their environments.

This paper proposes a unified physical framework for studying intelligence, grounded in the dynamics of conserved quantities and irreversible information processing. We connect abstract computation to explicit physical costs by introducing the Conservation-Congruent Encoding (CCE) condition. Under CCE, informational distinctions correspond to metastable basins of attraction whose separability is enforced by fundamental conservation laws. Crucially, we show that the irreversible formation or collapse of these dynamically distinguishable equivalence classes requires the export of a conserved quantity, thereby generalising Landauer's bound across arbitrary physical domains. Whether an agent encodes information in the chemical potential of biological wetware, the electrical charge of a silicon GPU, or the angular momentum of a quantum spin, we show that the minimal energetic cost of irreversible computation is governed by the product of the channel's conjugate force ($\mathcal{F}$) and its fundamental characteristic scale ($\alpha$).

Within this physically grounded framework, we model an agent and its environment as coupled dynamical systems. We represent goals by macroscopic work performed over spatiotemporal trajectories. We define two complementary, scale- and boundary-relative physical efficiencies: Intelligence ($\chi$) measures the macroscopic goal-directed work produced per nat of irreversibly processed information, while consciousness ($\kappa$)---framed here strictly as an operational property---measures the goal-directed work supported per nat of preserved internal information. Furthermore, rather than assuming fixed system architectures, we treat the specification of the system boundary itself as a dynamic optimization problem, identifying the physically meaningful agent as the minimal boundary that maximises these efficiencies.

To rigorously validate the substrate neutrality of the CCE framework, we stress-test it at both the microscopic and macroscopic limits of physical observation. At the quantum scale, by framing the double-slit experiment as an information-extraction problem requiring the physical instantiation of a CCE, we map our generalised metriplectic dynamic equations directly to Hilbert space. This recovers the Lindblad Master Equation, suggesting that the destruction of quantum interference can be modelled as a consequence of the generalised dissipative exhaust required by a macroscopic detector to record a classical bit.

Scaling to the cosmological limit, we then explore the hypothesis that gravity emerges as the macroscopic geometric footprint of these informational bounds. We address the fundamental geometric paradox between volume-preserving microscopic dynamics and area-bounded macroscopic limits by showing that, under this hypothesis, measurement-induced dissipation is consistent with a volumetric phase-space collapse. By equating the localised Landauer exhaust of this coarse-graining to causal horizon deformation, we recover the Einstein Field Equations in this limiting construction via the Raychaudhuri integration. Extending this to the extreme limits of physical measurement, we establish a strict polynomial bound on macroscopic observation, suggesting that the structural fuel required for an autonomous observer to record a black hole's microstate may trigger mutual gravitational collapse if the observer crosses a critical threshold of approximately $2.52 R_S$.

Beyond these core definitions, we extend these metrics to coupled multi-component systems, yielding a quantitative notion of emergent intelligence as a coupling-dependent change in efficiency across scales. We analyse how reversible structure-preserving dynamics can reduce the irreversible information required for sustained goal-directed work, with oscillatory and near-critical regimes providing illustrative biological examples. We also develop a compositional theory of continuous dynamical circuits using ports and attractor semantics under CCE, and outline a physics-grounded perspective on AI safety based on strict constraints on irreversible information flow, entropy production, and structural homeostasis.

Taken together, the results propose a substrate-neutral theory that seeks to relate notions of intelligence directly to fundamental physical law, making its constraints and trade-offs explicit across all scales of nature.

\section{Mathematical Framework for Assessing Intelligence}
\label{sec:FormalModel}

This section develops a mathematical framework that renders intelligence a physically measurable property of coupled agent--environment dynamics. The central objective is to distinguish information transport from information use, and to identify when informational transformations incur unavoidable physical cost. To this end, both the agent and the environment are modelled as stochastic dynamical systems whose states evolve in time and are coupled through explicit ports. These ports define the physical interface across which information is exchanged and goal-directed behaviour emerges.

The framework isolates the minimal structural elements required to quantify intelligence as a physically
grounded efficiency. It is developed in four stages:
\begin{enumerate}[label=(\roman*)]
\item formal specification of the coupled agent--environment dynamics;
\item definition of useful work completed by the agent under a spatiotemporal framing;
\item introduction of a physically grounded information formulation that links computational irreversibility to physical cost;
\item definition of scale- and boundary-relative intelligence metrics derived from these elements.
\end{enumerate}

\subsection{Agent--Environment Dynamics Model}
\label{sec:AgentEnvironmentModel}

Let $(X,\mathscr{X})$ and $(E,\mathscr{E})$ denote the measurable state spaces of the internal (agent) and external (environment) systems, respectively. Define measurable port maps
\begin{align}
\label{eq:auto:0001}
\Pi_X : X \to \Xi_X, 
\qquad
\Pi_E : E \to \Xi_E,
\end{align}
where $\Xi_X$ and $\Xi_E$ are measurable port spaces representing the agent--environment interface variables at the chosen system boundary and scale. The coupled one-step transition dynamics are defined by measurable kernels
\begin{align}
\label{eq:auto:0002}
X_{t+1} &\sim K_X(\,\cdot\,\mid X_t,\,\Pi_E(E_t)), \\
E_{t+1} &\sim K_E(\,\cdot\,\mid E_t,\,\Pi_X(X_t)),
\end{align}
where $K_X$ and $K_E$ govern the evolution of internal and external states, respectively. Each system depends on the other only through its port variables, satisfying the conditional
independence relations
\begin{align}
\label{eq:auto:0003}
X_{t+1} &\perp\!\!\!\perp E_t \mid (X_t,\,\Pi_E(E_t)), \\
E_{t+1} &\perp\!\!\!\perp X_t \mid (E_t,\,\Pi_X(X_t)).
\end{align}
Hence, the ports $(\Pi_X(X_t),\,\Pi_E(E_t))$ form a Markov blanket that separates internal and external states, following the formalisms established in graphical models \cite{pearl2014probabilistic} and subsequently applied to the variational foundations of self-organising systems \cite{friston2013life, parr2022active}. Let $S_t := (X_t, E_t)$ denote the joint process with state space $(X \times E,\;\mathscr{X} \otimes \mathscr{E})$, and let $\mu_t := \mathsf{Law}(S_t) \in \mathcal{P}(X \times E)$ denote its probability law. The joint transition kernel $K$ on $(\mathscr{X} \otimes \mathscr{E})$ factorises via the ports as
\begin{align}
\label{eq:auto:0004}
K(dx',de' \mid x,e)
  &= K_X(dx' \mid x,\,\Pi_E(e))\;
     K_E(de' \mid e,\,\Pi_X(x)).
\end{align}
This factorization expresses that, conditioned on the current joint state, the next internal and external states are independent, each responding only through its respective port. The evolution of the joint measure is then given by
\begin{align}
\label{eq:auto:0005}
\mu_{t+1}(dx',de') 
  &= \!\int K(dx',de' \mid x,e)\,\mu_t(dx,de),
\end{align}
which defines a valid measure-valued dynamical system on $(X \times E,\;\mathscr{X} \otimes \mathscr{E})$. Under standard regularity assumptions (standard Borel spaces, measurable kernels, and optional Feller continuity), this recursion guarantees the existence of a unique probability law for the process $\{S_t\}_{t \ge 0}$ given an initial distribution $\mu_0$. The pushforward
\begin{align}
\label{eq:auto:0006}
\bar{\mu}_t  := (\Pi_X, \Pi_E)_\# \mu_t
\end{align}
defines the joint distribution over the agent–environment interface, capturing the measurable exchange between them. Time dependence of all informational and energetic quantities is thus implicit in
the evolution of $\mu_t$ under $K$.

\subsection{Useful Work as a Spatiotemporal Causal Footprint}
\label{sec:useful_work_causal_footprint}

Intelligent agents typically operate in open, spatiotemporally extended environments. Small, targeted
interventions at the agent--environment interface can unlock large downstream changes in the environment
(e.g.\ flipping a single bit that triggers a large automated control action, initiating a reaction by
introducing a catalyst, or issuing a command that coordinates distributed actuators). Evaluating
intelligence therefore requires a macroscopic notion of \emph{causal footprint}: the total downstream
structural consequences causally enabled by the agent's interventions.

Because explicitly simulating the environment $E$ at all microscopic scales is intractable, we evaluate
downstream dynamics through a spatiotemporal initial--boundary value problem (IBVP) constrained by
relativistic causality and local conservation
laws.

Let $\mathcal M$ denote a spacetime manifold foliated by spatial hypersurfaces $\Sigma_t$. We formalise the
agent's spatial boundary as a closed, codimension-one interface
$\Gamma_{\mathrm{port}}\subset \Sigma_t$. Over the horizon $t\in[0,T]$, this interface sweeps out a timelike
world-tube
\begin{align}
\label{eq:auto:0007}
\mathcal W_{\mathrm{port}} := \Gamma_{\mathrm{port}}\times [0,T]\subset \mathcal M.
\end{align}
This provides the physical bridge to the stochastic port model of \S\ref{sec:AgentEnvironmentModel}: the
abstract port variables $\Pi_X(X_t)$ and $\Pi_E(E_t)$ manifest physically as time-dependent boundary
data exchanged across $\mathcal W_{\mathrm{port}}$.

We designate causal past and future using standard notation. The agent's interventions at the port can only
influence the environment within its future domain of dependence $J^+(\mathcal W_{\mathrm{port}})$. The
historical data determining the initial state of the environment is strictly contained within the causal
past $J^-(\mathcal W_{\mathrm{port}})$. At time $t\ge 0$, the downstream region of influence is the spatial
volume
\begin{align}
\label{eq:auto:0008}
\Omega(t) := \Sigma_t \cap J^+(\mathcal W_{\mathrm{port}}).
\end{align}
Its boundary $\partial\Omega(t)$ is partitioned into two disjoint surfaces: (i) the active port interface
$\Gamma_{\mathrm{port}}$, and (ii) the expanding causal horizon $\Gamma_{\mathrm{LC}}(t)$ separating the
influenced region from the rest of the universe.

To remain substrate-neutral, we avoid committing to a single reference state. Instead, we model the continuum within $\Omega(t)$ via a generalised structural potential density $a(\mathbf r,t)$ and its spatial flux $\mathbf J_a(\mathbf r,t)$. To ensure dimensional consistency with the mesoscopic informational costs developed in subsequent sections, we formally parameterise this potential across the relevant physical channels as
\begin{align}
\label{eq:auto:0009}
a(\mathbf r,t) = \sum_c (\mathcal{F}_c \alpha_c) \rho_c(\mathbf r,t),
\end{align}
where $\rho_c(\mathbf r,t)$ is the local density of accessible structural distinctions (a dimensionless volumetric density) for channel $c$. The parameter $\alpha_c$ is the characteristic physical scale of that channel (e.g., $k_B$ for thermal limits), and $\mathcal{F}_c$ is the corresponding intensive conjugate force. 

This multiplexed formulation preserves complete substrate independence. It defines $a(\mathbf r,t)$ as the locally available structural capacity (a generalised free energy) evaluated at the chosen operational scale, while ensuring that the macroscopic causal footprint remains physically commensurate with the agent's internal computational costs. Let $\sigma_a(\mathbf r,t)\ge 0$ denote the corresponding local rate density of irreversible structural dissipation across the integrated channels.

To isolate the agent's causal influence, the environmental evolution is posed as an IBVP with two
constraints:
\begin{enumerate}[label=(\alph*)]
    \item \emph{Initial Cauchy data:} the state on $\Omega(0)$ is determined by the universe history in
    $J^-(\mathcal W_{\mathrm{port}})\cap\{t\le 0\}$, so inherited pre-existing dynamics are explicitly
    conditioned upon.
    \item \emph{Port boundary condition:} along $\mathcal W_{\mathrm{port}}$, the agent prescribes the
    active fluxes injected into the environment.
\end{enumerate}
We define the \emph{absolute enabled work} over horizon $T$ as the integral of all downstream generalised
dissipation within the causal region,
\begin{align}
\label{eq:auto:0010}
W_{\mathrm{enabled},T}
  &:= \int_0^T dt \int_{\Omega(t)} \sigma_a(\mathbf r,t)\, d^3\mathbf r.
\end{align}

Local conservation laws and irreversible degradation imply a continuity balance
\begin{align}
\label{eq:auto:0011}
\partial_t a + \nabla\cdot \mathbf J_a = -\sigma_a.
\end{align}
Applying the Reynolds transport theorem and Gauss's divergence theorem to the expanding volume $\Omega(t)$
relates volumetric structural decay to boundary fluxes:
\begin{align}
\label{eq:auto:0012}
\frac{d}{dt}\int_{\Omega(t)} a(\mathbf r,t)\,d^3\mathbf r
&=
\int_{\Gamma_{\mathrm{port}}} \mathbf J_a\cdot \mathbf n_{\mathrm{in}}\, dA
-
\int_{\Gamma_{\mathrm{LC}}(t)}\!\!\left(\mathbf J_a - a\,\mathbf v_{\mathrm{LC}}\right)\cdot \mathbf n_{\mathrm{out}}\, dA
-
\int_{\Omega(t)} \sigma_a(\mathbf r,t)\, d^3\mathbf r,
\end{align}
where $\mathbf n_{\mathrm{in}}$ is the inward normal from the port and $\mathbf v_{\mathrm{LC}}$ is the
expansion velocity of the causal horizon. Reynolds transport gives
\begin{align}
\label{eq:auto:0013}
\frac{d}{dt}\int_{\Omega(t)} a\,d^3\mathbf r
  = \int_{\Omega(t)} \partial_t a\,d^3\mathbf r
    + \int_{\partial\Omega(t)} a\,\mathbf v_{\partial\Omega}\cdot \mathbf n_{\mathrm{out}}\,dA,
\end{align}
with $\mathbf v_{\partial\Omega}=\mathbf 0$ on the fixed port boundary and
$\mathbf v_{\partial\Omega}=\mathbf v_{\mathrm{LC}}$ on the moving boundary $\Gamma_{\mathrm{LC}}(t)$.
Substituting the local balance $\partial_t a=-\nabla\cdot \mathbf J_a-\sigma_a$ and applying Gauss's theorem
yields
\begin{align}
\label{eq:auto:0014}
\frac{d}{dt}\int_{\Omega(t)} a\,d^3\mathbf r
&= -\int_{\partial\Omega(t)} \mathbf J_a\cdot \mathbf n_{\mathrm{out}}\, dA
   +\int_{\partial\Omega(t)} a\,\mathbf v_{\partial\Omega}\cdot \mathbf n_{\mathrm{out}}\,dA
   -\int_{\Omega(t)} \sigma_a\, d^3\mathbf r,
\end{align}
which, after splitting $\partial\Omega(t)=\Gamma_{\mathrm{port}}\cup \Gamma_{\mathrm{LC}}(t)$ and using
$\mathbf n_{\mathrm{in}}=-\mathbf n_{\mathrm{out}}$ on $\Gamma_{\mathrm{port}}$, gives the stated form.
The first boundary term evaluates to the physical power injected by the agent, $\dot W_{\mathrm{port}}(t)$.
Let $\Phi_{\mathrm{LC}}(t)$ denote the net flux escaping across the expanding causal horizon. Under causal isolation of the region, there is no incoming flux across the horizon, so the
horizon term reduces to outward leakage only (and in the idealised perfectly isolating limit,
$\Phi_{\mathrm{LC}}(t)=0$). Integrating over $[0,T]$ yields the exact relation
\begin{align}
W_{\mathrm{enabled},T}
&= W_{\mathrm{port},T} - \Delta A_{\mathrm{env}} - \int_0^T \Phi_{\mathrm{LC}}(t)\,dt,
\label{eq:auto:0015}
\end{align}
where $W_{\mathrm{port},T}:=\int_0^T \dot W_{\mathrm{port}}(t)\,dt$ and $\Delta A_{\mathrm{env}} :=
A_{\mathrm{env}}(T)-A_{\mathrm{env}}(0)$ is the net change in total bulk structural potential within
$\Omega(t)$.

The initial data on $\Omega(0)$ already encode gradients inherited from the past light cone. Autonomous
relaxation of these pre-existing gradients will produce structural change and flux across
$\Gamma_{\mathrm{LC}}(t)$ even if the agent behaves as a strictly passive object. To avoid conflating active
causal influence with deterministic background drift, define a baseline using the same past-light-cone
conditioned initial data but imposing a passive boundary condition at the port world-tube (e.g.\
$\mathbf J_a\cdot \mathbf n=0$ on $\mathcal W_{\mathrm{port}}$). Let
$W_{\mathrm{enabled},T}^{\mathrm{default}}$ and $\Delta A_{\mathrm{env}}^{\mathrm{default}}$ denote the
enabled work and bulk structural change under this passive baseline.

By domain-of-dependence, the state on and outside the
expanding causal horizon $\Gamma_{\mathrm{LC}}(t)$ cannot depend on boundary conditions applied at
$\mathcal W_{\mathrm{port}}$ during $t\in[0,T]$. Therefore, the horizon flux $\Phi_{\mathrm{LC}}(t)$ is
identical in the actual and default dynamics. Defining the purely causally enabled structural footprint as
\begin{align}
\label{eq:auto:0016}
W_{\mathrm{causal},T}
  &:= W_{\mathrm{enabled},T}^{\mathrm{actual}} - W_{\mathrm{enabled},T}^{\mathrm{default}},
\end{align}
the unwieldy light-cone terms cancel exactly:
\begin{align}
\label{eq:auto:0017}
W_{\mathrm{causal},T}
&= W_{\mathrm{port},T} - \Big(\Delta A_{\mathrm{env}}^{\mathrm{actual}}-\Delta A_{\mathrm{env}}^{\mathrm{default}}\Big)
 = W_{\mathrm{port},T} - \Delta A_{\mathrm{env}}^{\mathrm{causal}}.
\end{align}
This boundary-value formulation isolates the agent's net structural footprint without assuming an
``unperturbed'' external vacuum, and ensures an agent receives no efficiency credit for merely witnessing
pre-existing cascades (a large autonomous background event yields
$\Delta A_{\mathrm{env}}^{\mathrm{causal}}=0$).

Often we are interested not in total causal footprint but in influence aligned with a task or purpose at
the chosen operational scale. To account for this a projection operator $\pi_{\mathcal G}$ that extracts the
goal-relevant component of the interface-level footprint can be leveraged. We define the expected work over horizon $T$ as
\begin{align}
\label{eq:auto:0018}
W_{\mathrm{goal},T}
&:= \mathbb{E}_{\mu_{0:T}}\!\left[\pi_{\mathcal G}\!\left(W_{\mathrm{port},T}-\Delta A_{\mathrm{env}}^{\mathrm{causal}}\right)\right],
\end{align}
where $\mu_{0:T}$ is the law of the coupled process on $[0,T]$.

\subsection{Information Processing}
\label{sec:InformationMetrics}

Intelligent systems process information through the evolution of physical dynamics. Throughout, we adopt the language of information processing as a descriptive convenience. The underlying physical primitive is entropy production, which can be related (at an appropriate level of description) to information-theoretic quantities, but is not identical to them. To relate abstract notions of information to measurable quantities, we formalise information processing at the level of probability distributions over system trajectories. In this framework, information corresponds to dynamically preserved distinctions in the system’s evolution, while irreversible information loss corresponds to the collapse of such distinctions under time evolution. 

In practice, irreversible information processing may be diagnosed via divergence between forward and time-reversed trajectory distributions. Such divergences provide bounds on irreversible loss of information, becoming informative only after projection onto a dynamically induced quotient space that determines which distinctions are physically meaningful.

While this path-space formulation applies to general dynamical systems, relating informational irreversibility to explicitly measurable physical costs typically requires additional structure. In particular, not every operational coarse-graining admits a clean metastable encoding at a given scale. We therefore introduce the \emph{Conservation-Congruent Encoding} (CCE) framework, which specifies a class of systems in which dynamically preserved distinctions admit metastable physical realizations constrained by conservation laws. In this regime, information-theoretic measures of irreversibility can be directly related to entropy production and the dissipation of conserved quantities such as heat or angular momentum, recovering and generalising Landauer-type bounds \cite{landauer1961irreversibility,bennett1982thermodynamics,parrondo2015thermodynamics}.

\subsubsection{Path-Space Information Measures for Dynamical Systems}
\label{sec:PathMeasures}

Information processing arises from the irreversible evolution of coupled dynamical systems. In this work, we take the base state spaces $(X \times E, \mathscr{X} \otimes \mathscr{E})$ to represent the absolute microscopic reality of the agent and environment, governed fundamentally by deterministic, reversible dynamics. Consequently, the stochasticity present in the induced agent transition kernel $K_X(dx'|x)$ is not an axiomatic assumption, but a strict mathematical artifact of the port maps $\Pi_X$ and $\Pi_E$. Because the agent interacts with the environment exclusively through these lower-dimensional boundaries, the unresolved microscopic degrees of freedom of the environment act as a bath, driving the irreversible stochastic evolution of the open subsystem $X$.

While instantaneous distributions $\mu_t^X = \mathrm{Law}(X_t)$ describe the subsystem at a given time, questions of reversibility and irreversible information loss are inherently temporal and are therefore most naturally formulated at the level of entire trajectories. Let 
\begin{align}
\label{eq:auto:0019}
\mathcal{T}_T := \{\tau = (x_0, x_1, ..., x_T) | x_t \in X\}
\end{align}
denote the space of agent trajectories over a finite horizon $T$. Given an initial distribution $\mu_0^X \in \mathcal{P}(X)$ and the induced stochastic transition kernel $K_X$, the dynamics generate a probability measure over trajectories, $\mathbb{P}_{\mu_0^X}$, defined on $\mathcal{T}_T$.

Operational notions of information and computation may not depend on all microscopic distinctions between internal states. Furthermore, in complex systems, information processing is often multiplexed across multiple spatiotemporal scales and macroscopic physical channels. We therefore define a multiplexed operational interface via a projection $\Phi := \{\phi_t: X \rightarrow Z\}_{t=0}^T$ where the target $Z$ acts as a quotient space grouping microscopic states into macroscopic equivalence classes. To align with the macroscopic properties of the environment, the operational space is strictly a Cartesian product of distinct physical channels, $Z = Z_1 \times Z_2 \times \dots \times Z_C$. Each component projection $\phi^{(c)}_t: X \rightarrow Z_c$ isolates the dynamical equivalence classes associated with a specific macroscopic conserved quantity (e.g., thermal heat, chemical gradients, or angular momentum). 

This projection induces a corresponding trajectory-level mapping $\pi_\Phi: \mathcal{T}_T \rightarrow \mathcal{Z}_T$. The induced statistics of coarse-grained trajectories are captured by the pushforward measure $\mathbb{P}_{\mu_0^X}^\Phi := (\pi_\Phi)_\# \mathbb{P}_{\mu_0^X}$. To assess temporal asymmetry at this operational level, we introduce a time-reversal operator $\Theta$ and define the time-reversed induced path measure $(\mathbb{P}_{\mu_0^X}^{rev})^\Phi$.

We define the operational (or encoding-level) irreversibility over horizon $T$ as the statistical distinguishability between the forward and time-reversed induced path measures:
\begin{align}
\label{eq:auto:0020}
\mathcal{I}_\Phi(T) := D_{KL}(\mathbb{P}_{\mu_0^X}^\Phi \,\|\, (\mathbb{P}_{\mu_0^X}^{rev})^\Phi).
\end{align}
Operational reversibility corresponds to the vanishing of this quantity ($\mathcal{I}_\Phi(T) = 0$). This formulation ensures invariance under internal permutations, reparameterizations, or spatial translations within the system. For example, if an encoded predictive memory is physically routed from one spatial region of the agent to another, the projection identifies these distinct micro-trajectories as belonging to the same preserved equivalence class, correctly yielding zero irreversible computational cost for pure spatial transport.

By the data-processing inequality, encoding-level irreversibility is bounded above by the raw trajectory-level asymmetry of the open system:
\begin{align}
\label{eq:auto:0021}
\mathcal{I}_\Phi(T) \le \mathcal{I}(T).
\end{align}
Crucially, because $X$ is an open system coupled to an unobserved environment via the port maps, its induced microscopic evolution is non-unitary. Therefore, the base trajectory asymmetry $\mathcal{I}(T)$ is strictly positive, representing the total physical entropy exported to the environmental bath. The data-processing inequality thus asserts a strict physical hierarchy: the irreversibility at the abstract, multiplexed operational level ($\Phi$) cannot exceed the total irreversibility of the open physical system ($X$).

\subsubsection{Conservation-Congruent Encoding (CCE)}
\label{sec:EncodingPhysicalState}

The trajectory-level framework developed in \S\ref{sec:PathMeasures} characterises information in terms of dynamically induced equivalence classes, grouping microscopic configurations that yield statistically indistinguishable macroscopic trajectories under a chosen multiplexed coarse-graining $\Phi$. These equivalence classes are abstract objects; they specify which internal histories are informationally indistinguishable across various macroscopic channels, but do not, by themselves, determine how such distinctions are physically realised. In this subsection, we identify conditions under which these abstract equivalence classes admit metastable physical realizations, making them easier to measure and study. We refer to such realizations as \emph{conservation-congruent encodings}.

Let $X$ denote the internal physical state space of the agent and let $C$ denote a set of externally adjustable
control parameters (e.g., barrier heights, biases, or couplings). The chosen projection $\Phi$ maps microscopic states to a composite quotient space $Z = Z_1 \times Z_2 \times \dots \times Z_C$, whose elements label the abstract equivalence classes induced by the dynamics across distinct physical channels. A physical
system realises a composite coarse-grained state $z = (z_1, \dots, z_C) \in Z$ through a probability distribution over its internal physical states under a
given control configuration $c \in C$. This is captured by a representation map
\begin{align}
\label{eq:auto:0022}
f : Z \times C \to \mathcal{P}(X),
\end{align}
where $f(z,c)$ denotes the distribution over microscopic configurations $x \in X$ corresponding to the realization of
the joint macroscopic state $z$ under control $c$. At this level, no assumptions are made regarding the stability, uniqueness, or cost of such realizations.

To connect abstract equivalence classes to explicit physical cost, we impose the
\emph{Conservation-Congruent Encoding (CCE)} conditions, which specify when the realizations $f(z,c)$ are
metastable and controllably manipulable:

\begin{enumerate}[label=(\roman*)]
    \item \textbf{Metastable realization:}
    Each composite equivalence class $z \in Z$ is realised by a metastable basin $B_z \subset X$, such that basins
    corresponding to distinct $z$ are separated by activation barriers associated with the effective, scale-dependent conserved
    physical quantities that emerge under the chosen projection $\Phi$. Local equilibration within each basin occurs on timescales short compared to those
    governing inter-basin transitions or control-induced deformation of the basin structure
    \cite{kramers1940brownian,hanggi1990reaction}. 

    \item \textbf{Quasistatic, conservation-respecting control:}
    Control trajectories $c_t$ deform the basin structure $\{B_z\}$ quasistatically relative to internal
    relaxation. Transitions between basins proceed along paths constrained by the conserved-quantity
    structure of the dynamics, rather than through arbitrary or
    nonphysical jumps \cite{callen1985thermodynamics,sekimoto1998langevin}.
\end{enumerate}

Under the CCE assumptions, the representation map $f$ is locally invertible on each metastable basin:
physical configurations identify equivalence classes up to fluctuations set by the characteristic scales
$\alpha_c$ of the specific environmental baths to which the channel couples (e.g., $k_B$ for thermal channels, or corresponding intensive parameters for chemical or mechanical channels). As a result,
operations that irreversibly merge dynamically distinguishable equivalence classes necessarily entail
irreversible physical transformations of these scale-dependent conserved quantities.

It is critical to explicitly distinguish the physical state variables that instantiate an encoding from the conserved quantities exported to erase it. Under CCE, these need not be the same physical property. For example, in a classical CMOS gate, the information is encoded electromagnetically via spatial charge accumulation, but the conserved quantity exported during irreversible erasure is thermal energy (heat). The encoding is termed \emph{conservation-congruent} because its metastable geometry---the energetic barrier separating its states---is strictly stabilised by, and mathematically coupled to, the specific conservation laws governing the exhaust channel. The irreversible collapse of the encoding state ($x$) strictly forces a proportional export of entropy via the specific conserved quantities via the dissipative flow.

The CCE conditions are sufficient but not necessary for irreversibility to entail entropy export. Crucially, this relationship is scale-dependent. At the absolute microscopic limit of a closed system, dynamic evolution is purely reversible and volume-preserving. Irreversibility and generalised dissipation strictly emerge only at mesoscopic or macroscopic scales where the system is open, requiring the export of conserved quantities into unresolved environmental degrees of freedom. Therefore, for open systems at these higher scales, the collapse of dynamically induced equivalence classes inevitably requires dissipation into degrees of freedom not retained by the system, even when no clean metastable CCE description exists. The CCE framework identifies a broad and practically important class of systems---particularly engineered and experimentally tractable ones---for which this macroscopic cost can be more easily isolated and measured.

In complex systems, information processing is often multiplexed across multiple spatiotemporal scales, meaning distinct conservation-congruent encodings can simultaneously emerge under different coarse-grainings. Consequently, finding strict joint metastability across all physical channels within the raw state space is highly improbable. In such multi-scale regimes, the relevant effective conserved quantities (and their corresponding metastable basins) can be formally understood as distinct slow eigenfunctions of the system's underlying transfer operator, each associated with a different relaxation timescale.  The chosen projection $\Phi$ thus acts as a tunable filter: it sets the specific scale of observation, isolating the functionally relevant metastable encodings at that level from both faster, high-entropy microscopic fluctuations and slower background drift.

\subsubsection{Irreversibility and Conserved-Quantity Cost}
\label{sec:IrreversibilityCost}

Within the quotient-based framework of \S\ref{sec:PathMeasures}, irreversible information loss corresponds
to a many-to-one mapping between dynamically induced equivalence classes defined by statistically indistinguishable macroscopic trajectories under a chosen multiplexed coarse-graining. Such a collapse of equivalence classes
necessarily entails the destruction of physically realizable distinctions and therefore requires entropy
export to degrees of freedom not retained by the system. This establishes a general connection between
irreversibility and entropy production that is independent of any particular encoding scheme.

The Conservation-Congruent Encoding (CCE) conditions introduced in \S\ref{sec:EncodingPhysicalState}
identify a broad class of systems for which this entropy cost can be made explicit. Under CCE, abstract
equivalence classes admit metastable physical realizations, and irreversible information loss occurs when
control protocols deform the metastable basin structure so that previously distinct basins merge. During
such operations, the physical state no longer encodes which equivalence class the system originated from,
and restoring this information necessarily requires exporting entropy via the specific conserved quantities governing the dissipative flow to the environment
\cite{landauer1961irreversibility,parrondo2015thermodynamics}.

Encoding-level reversibility concerns whether the induced mapping between equivalence classes remains
bijective at the operational level of description. Conserved-quantity reversibility, by contrast, concerns
whether the underlying physical evolution preserves the relevant conserved measures and produces no net
entropy. Under the CCE assumptions, these notions coincide: reversible manipulation of encoded information must correspond to reversible evolution with respect to the conserved quantities of the environmental baths coupled to the operational channels
\cite{bennett1973logical,bennett1982thermodynamics}.

Consider an irreversible operation in which a set of equivalence classes
$\{\ell_i\}_{i=1}^{m}$ is merged into a single physically indistinguishable state. Under the CCE
assumptions, the minimal entropy export required to perform this operation is bounded below by the Shannon
entropy of the erased distribution,
\begin{align}
\label{eq:auto:0023}
S_{\min}^{\mathrm{CCE}}
= -\alpha \sum_{i=1}^{m} p_i \ln p_i,
\end{align}
where $p_i$ denotes the prior probability of class $\ell_i$, and $\alpha$ is the characteristic entropy
scale associated with the conserved quantity (e.g., $k_B$ for thermal channels or $\hbar$ for angular
momentum). For equiprobable inputs this bound reduces to $S_{\min}^{\mathrm{CCE}} = \alpha \ln m$,
recovering Landauer-type limits as a special case
\cite{landauer1961irreversibility,bennett1982thermodynamics}.

While the metrics developed in \S\ref{sec:InformationMetrics} fundamentally compare macroscopic goal-directed influence to abstract information-processing costs, it is often instructive to evaluate them in a common physical currency. When we seek to frame these measures explicitly as dimensionless physical efficiencies (e.g., Joules of structural work per Joule of computational cost), we require an energetic conversion. This is naturally provided by standard conjugacy relations: the irreversible export of a conserved quantity across a boundary occurs against an intensive conjugate force $\mathcal{F}$ fixed by the environment. Multiplying the scaled informational erasure by this conjugate force yields the minimal energetic cost of the computation, allowing for direct comparison against the macroscopic physical footprint \cite{callen1985thermodynamics}.

Under CCE, the Shannon term in $S_{\min}^{\mathrm{CCE}}$ is precisely the dimensionless irreversibility of
the induced many-to-one map on equivalence classes, while the prefactor sets the physical scale of the
associated export. The corresponding generalised energetic Landauer bound is
\begin{align}
E_{\min}^{\mathrm{CCE}}
  &= \mathcal{F}\,S_{\min}^{\mathrm{CCE}}
   = -\big(\mathcal{F}\alpha\big)\sum_{i=1}^{m} p_i \ln p_i .
\label{eq:auto:0024}
\end{align}
This makes explicit the bridge from abstract irreversibility to energetic cost: the trajectory-level
asymmetry measures of \S\ref{sec:PathMeasures} supply the dimensionless log-likelihood ratio (hence the
Shannon term), and CCE supplies the conversion factor $(\mathcal{F}\alpha)$ into Joules.

For thermal entropy export, $\mathcal{F}=T$ and $\alpha=k_B$, yielding the usual Landauer form
\begin{align}
\label{eq:auto:0025}
E_{\min}
  &= k_B T\,\Big(-\sum_{i=1}^{m} p_i \ln p_i\Big).
\end{align}
For an encoding stabilised by angular momentum, $\mathcal{F}=\omega$ and $\alpha=\hbar$, giving
\begin{align}
\label{eq:auto:0026}
E_{\min}
  &= \hbar\omega\,\Big(-\sum_{i=1}^{m} p_i \ln p_i\Big).
\end{align}
Other conserved channels admit analogous expressions with their corresponding conjugate forces.

Real agents typically flux multiple conserved quantities. If an irreversible operation exports entropy
through channels indexed by $a$, each with characteristic entropy scale $\alpha_a$ and conjugate force
$\mathcal{F}_a$, then the unified energetic lower bound is additive across channels,
\begin{align}
E_{\min}^{\mathrm{CCE}}
  &= \sum_a \mathcal{F}_a\,S_{\min,a}^{\mathrm{CCE}}
   = -\Big(\sum_a \mathcal{F}_a\alpha_a\Big)\sum_{i=1}^{m} p_i \ln p_i ,
\label{eq:auto:0027}
\end{align}
where the final form applies when the same erased distribution $\{p_i\}$ is implemented across channels
under CCE. Interpreting $\boldsymbol{\alpha}=(\alpha_a)_a$ and $\mathbf{F}=(\mathcal{F}_a)_a$, the
prefactor is the inner product $\mathbf{F}\cdot\boldsymbol{\alpha}$.

In cases where the chosen operational interface fails to align with a metastable encoding, the irreversible collapse of dynamically induced equivalence classes still necessitates physical export. Crucially, however, this lack of a clean encoding is an epistemic artifact of the measurement, not a physical breakdown. Because the chosen coarse-graining does not adequately track the relevant conserved quantities, the physical costs leak into unmonitored degrees of freedom. Consequently, while the fundamental conservation laws remain exact, the observed cost becomes highly distributed and evades a clean, closed-form expression at that specific level of description.

\subsection{System Boundaries and Structured Information Processing}
\label{sec:SystemBoundaries}

The information-processing framework developed in \S\ref{sec:InformationMetrics} characterises irreversible computation in terms of path-level distinguishability under a chosen multiplexed operational coarse-graining $\Phi$. Such quantities are well defined only once the physical limits of the system have been specified. In this section, we formalise the notion of a \emph{system boundary} in a manner consistent with the joint agent--environment dynamics of \S\ref{sec:AgentEnvironmentModel} and the path-space framework of \S\ref{sec:InformationMetrics}, treating both the boundary and the internal encodings as derived, scale-dependent objects rather than primitive assumptions.

Throughout, the base joint dynamics of the universe $(X, E)$ are assumed to be strictly microscopic and deterministic. The stochastic process $(X_t, E_t)$ evolving under a joint transition kernel, as introduced in \S\ref{sec:AgentEnvironmentModel}, strictly emerges as a mathematical consequence of marginalising over unobserved degrees of freedom. We now treat the choice of these boundary variables---alongside the choice of internal operational variables---as a rigorous joint search problem driven by the optimization of physical efficiencies.

We take $X$ to denote the candidate microscopic system variables. A fundamental requirement is that $X$ itself forms a \emph{dynamically connected system}. At a given multiplexed coarse-graining $\Phi : X \rightarrow Z_1 \times \dots \times Z_C$, this means that the internal degrees of freedom cannot be decomposed into a nontrivial partition whose components evolve independently while preserving the same path-level informational and thermodynamic structure. This requirement excludes purely abstract collections of variables and ensures that $X$ corresponds to a single cohesive physical system rather than a disjoint aggregate. 

Given a connected system $X$, defining the agent as an open thermodynamic entity requires specifying two distinct mappings: the physical boundary projections $(\Pi_X, \Pi_E)$ that mediate environmental interaction, and the internal multiplexed projection $\Phi$ that defines the operational state. These maps are not assumed \emph{a priori}. Instead, they are jointly selected so that the induced internal dynamics satisfy a strict set of criteria, ultimately serving to optimise the system's operational goals:

\begin{enumerate}[label=(\roman*)]
    \item \textbf{Dynamical Closure and Noise Generation:} Environmental influence on the internal evolution must enter only through the boundary variables. Because the boundary projections $\Pi_X: X \rightarrow \Xi_X$ and $\Pi_E: E \rightarrow \Xi_E$ map the deterministic microscopic reality down to lower-dimensional interface ports, they strictly define the unresolved degrees of freedom that act as the environmental bath. Variations in microscopic environmental degrees of freedom orthogonal to $\Pi_E(E_t)$ are marginalised, generating the stochastic noise floor of the induced transition kernel $K_X$. The boundary must be selectively permeable, constraining exactly which physical fluxes (and thus which conjugate forces) are allowed to perturb the system.
    
    \item \textbf{Encoding Preservation (Maximising the Numerator):} Beyond closure, admissibility requires the preservation of structured information processing against the boundary-induced noise. The internal operational projection $\Phi$ must map to a composite quotient space $Z$ that aligns with the system's Conservation-Congruent Encodings (CCE). The joint selection of the interface triplet $(\Pi_X, \Pi_E, \Phi)$ must preserve the multidimensional metastable basins required for meaningful prediction and control. Coarsening the boundary or the internal mapping too aggressively collapses these task-relevant macroscopic distinctions, directly reducing the macroscopic goal-directed work $W_{\mathrm{causal},T}$ that the system can perform.
    
    \item \textbf{Minimality (Minimising the Denominator):} Among closed and preserving architectures, the optimal interface $(\Pi_X, \Pi_E, \Phi)$ is selected by minimising unnecessary physical cost. Boundaries that expose superfluous internal variables ($\Pi_X$) or track irrelevant environmental fluctuations ($\Pi_E$) force the system to process and eventually erase non-predictive noise. This, along with internal encodings ($\Phi$) that track irrelevant physical channels, introduces unneeded physical dissipation. It inflates the system's irreversible computational cost, captured by the generalised multi-channel Landauer sum $E_{\min}^{\mathrm{CCE}} = \sum_c \mathcal{F}_c \alpha_c S_{\min}^{(c)}$, without improving prediction or control, thereby degrading its overall efficiency.
\end{enumerate}

Under this definition, identifying a macroscopic agent becomes a formal optimization problem. We treat the identification of the system and its computational states as a fundamental modelling decision: the most physically meaningful description of an agent is precisely the joint choice of boundary $(\Pi_X, \Pi_E)$ and internal multiplexed encoding $\Phi$ that simultaneously maximises its intelligence ($\chi$) and consciousness ($\kappa$) metrics. Once this optimal interface is identified, the stochastic port-based formulation of \S2.1 and the path-space thermodynamics of \S2.3 click together to provide a faithful, energetically bounded representation of the agent.

\subsection{From Communication to Intelligence}
\label{sec:CommVsIntel}

Shannon's communication theory established the fundamental limits on how much information can be transmitted across a physical channel subject to constraints such as noise and available signal power. Crucially, the theory treats communication as a process whose objective is faithful transmission rather than irreversible transformation. Shannon's bounds therefore quantify how efficiently a physical medium may transport information, not how efficiently a system may use information to guide its own evolution \cite{shannon1948communication}.

Irreversible information processing is physically distinct from communication. Whenever encodings are collapsed, a physical system must export entropy associated with a conserved quantity. This generalised form of Landauer’s insight establishes that information use has intrinsic physical cost \cite{landauer1961irreversibility,bennett1982thermodynamics}. The Conservation-Congruent Encoding (CCE) assumptions introduced in \S\ref{sec:EncodingPhysicalState} ensure that such costs are well defined across all substrates.

The present framework unifies these perspectives by relating communication to the irreversible use of information in performing goal-directed work. An intelligent system does not just receive or relay information; it irreversibly transforms information so that it constrains future behaviour and enables goal-directed physical work. Let $W_{\mathrm{causal}}$ denote the useful work performed on the environment (as derived via the spatiotemporal boundary value problem in \S\ref{sec:useful_work_causal_footprint}), and let $I_{\mathrm{irr}}$ denote the total irreversible information processed by the agent (bounded by the CCE limits in \S\ref{sec:IrreversibilityCost}). We define the \emph{intelligence} as
\begin{align}
\label{eq:auto:0028}
\chi = \frac{W_{\mathrm{causal}}}{I_{\mathrm{irr}}},
\end{align}
representing the amount of goal-directed work extracted per nat of irreversible information processing. This ratio is well defined whenever $I_{\mathrm{irr}} > 0$. In the limiting case $I_{\mathrm{irr}} = 0$, the agent performs no irreversible transformations, preserves all encodings, and cannot generate any persistent, time-asymmetric influence on its environment. Such a fully reversible process therefore cannot be assigned a meaningful finite intelligence under this metric. Systems with larger $\chi$ convert their physically costly informational transformations into useful work more effectively. This definition does not measure task success or reward accumulation directly, but rather the physical efficiency with which irreversible informational transformations are converted into work aligned with externally specified goals.

In time-dependent settings, intelligence admits a natural flux formulation. Let $\dot I_{\mathrm{irr}}(t)$ and $\dot W_{\mathrm{causal}}(t)$ denote the instantaneous rates of irreversible information processing and work, respectively. We define the \emph{instantaneous intelligence} as
\begin{align}
\label{eq:auto:0029}
\chi(t) = \frac{\dot W_{\mathrm{causal}}(t)}{\dot I_{\mathrm{irr}}(t)},
\end{align}
which characterises the local efficiency with which irreversible information processing is converted into useful work at time $t$. The physically meaningful measure of intelligence over a finite horizon $T$ is the ratio of integrated fluxes,
\begin{align}
\label{eq:auto:0030}
\chi_T = \frac{\int_0^T \dot W_{\mathrm{causal}}(t)\,dt}{\int_0^T \dot I_{\mathrm{irr}}(t)\,dt},
\end{align}
which quantifies the total amount of goal-directed work extracted per nat of irreversibly processed information. Note that $\chi_T$ is defined as a ratio of integrated fluxes and does not, in general, equal the time integral of $\chi(t)$. This reflects the fact that intelligence is not an instantaneous property of a system, but an emergent efficiency defined over irreversible transformations accumulated across time.

In this sense, intelligence extends Shannon’s program; it concerns not only how efficiently information can be communicated across a channel, but how efficiently it can be irreversibly transformed within a system to produce physically meaningful work.

\section{Quantitative Measures of Intelligence}
\label{sec:QuantitativeMetrics}

Having established the coupled agent--environment dynamics, the representation of goals, and the notion of
admissible system boundaries, we now formalise the core quantitative measures of the present framework.
These measures characterise intelligence as a physical efficiency arising from irreversible and preserved
information processing, rather than as a behavioural or algorithmic property.

We introduce two complementary quantities. \emph{Intelligence} quantifies how efficiently irreversible
information processing is converted into goal-directed physical work. \emph{Consciousness} quantifies how
efficiently preserved, structure-maintaining information supports the same process, capturing the role of
persistent internal organization in enabling sustained, long-horizon behaviour. Together, these quantities
separate the contributions of irreversible intervention and reversible structural reuse in physically embodied
intelligent systems.

Both measures are defined directly in terms of physically measurable quantities---goal-directed work and
irreversibly or reversibly processed information---and are evaluated relative to a chosen admissible system
boundary and operational coarse-graining, as defined in \S\ref{sec:SystemBoundaries}.

\subsection{Intelligence}
\label{sec:Intelligence}

Given the macroscopic causal footprint developed in \S\ref{sec:useful_work_causal_footprint}, intelligence is primarily evaluated with respect to the expected causal work $W_{\mathrm{causal},T}$ performed on the environment. Over a finite horizon $T$, this is quantified by the expected net extraction of structural capacity over the interval:
\begin{align}
\label{eq:auto:0031}
W_{\mathrm{causal},T} = \mathbb{E}_{\mu_{0:T}}[W_{\mathrm{port},T} - \Delta A_{\mathrm{env}}^{\mathrm{causal}}]
\end{align}
Let $I_{\mathrm{irr},T}$ denote the total amount of information irreversibly processed by the agent over the same horizon:
\begin{align}
\label{eq:auto:0032}
I_{\mathrm{irr},T} = \int_0^T \dot{I}_{\mathrm{irr}}(t)\,dt
\end{align}
where $\dot{I}_{\mathrm{irr}}(t)$ is the instantaneous irreversible information rate defined in \S\ref{sec:InformationMetrics}. The \emph{cumulative intelligence} over $[0,T]$ is then defined as:
\begin{align}
\label{eq:auto:0033}
\chi_T = \frac{W_{\mathrm{causal},T}}{I_{\mathrm{irr},T}}
\end{align}
which quantifies the amount of causal structural work extracted per nat of irreversibly processed information. This ratio is well defined whenever $I_{\mathrm{irr},T} > 0$ and characterises the fundamental physical efficiency with which irreversible computation is expressed as macroscopic causal influence, independent of the absolute timescale of the underlying dynamics.

For analyses conducted at the level of instantaneous fluxes, the causal work rate is defined via the expected instantaneous power:
\begin{align}
\label{eq:auto:0034}
\dot{W}_{\mathrm{causal}}(t) := \mathbb{E}_{\mu_t}[\dot{W}_{\mathrm{port}}(t) - \dot{A}_{\mathrm{env}}^{\mathrm{causal}}(t)]
\end{align}
The corresponding instantaneous intelligence is:
\begin{align}
\label{eq:auto:0035}
\chi(t) = \frac{\dot{W}_{\mathrm{causal}}(t)}{\dot{I}_{\mathrm{irr}}(t)}
\end{align}
which serves as a local diagnostic of how irreversible information processing is expressed as causal influence at a given moment in time.

While raw causal intelligence $\chi_T$ measures the absolute scale of an agent's structural influence, evaluating intelligence in the context of a specific task requires isolating the goal-aligned component of this footprint. To account for this, a projection operator $\pi_{\mathcal G}$ that extracts the goal-relevant component of the interface-level footprint can be leveraged. As introduced in \S\ref{sec:useful_work_causal_footprint}, we define the expected goal-directed work as:
\begin{align}
\label{eq:auto:0036}
W_{\mathrm{goal},T} := \mathbb{E}_{\mu_{0:T}}\!\left[\pi_{\mathcal G}\!\left(W_{\mathrm{port},T}-\Delta A_{\mathrm{env}}^{\mathrm{causal}}\right)\right]
\end{align}
This yields a complementary \emph{goal-directed intelligence} metric:
\begin{align}
\label{eq:auto:0037}
\chi_{\mathrm{goal},T} = \frac{W_{\mathrm{goal},T}}{I_{\mathrm{irr},T}}
\end{align}
This subsidiary metric characterises how efficiently a system converts irreversible computation specifically into work aligned with externally specified goals, isolating functional utility from undirected macroscopic disruption. The definition of intelligence is agnostic to the specific control, inference, or decision mechanisms employed by the agent. It depends only on the physical evolution of the joint measure $\mu_t$, the resulting boundary fluxes, and the irreversible information transformations generated by that evolution.

\subsection{Consciousness}
\label{sec:ConsciousnessMetric}

The intelligence metric $\chi_T$ in \S\ref{sec:Intelligence} quantifies how efficiently \emph{irreversible} information processing is converted into causal work. A complementary quantity characterises the extent to which this macroscopic footprint is supported by \emph{preserved} internal informational structure. We refer to this quantity as a \emph{consciousness} metric in the present framework.

Consciousness, as defined here, is an operational and physical quantity rather than a phenomenological one. It measures the contribution of internal informational distinctions that persist over time and continue to constrain future behaviour without requiring repeated irreversible reconstruction. Such preserved structure may include stable encodings (maintained as CCE metastable basins), memories, predictive internal models, or invariant dynamical organization. Whereas intelligence measures the efficiency of distinction-destroying computation, consciousness measures the efficiency with which \emph{preserved distinctions} support causal influence over extended horizons.

Let $\dot{I}_{\mathrm{rev}}(t)$ denote the instantaneous rate of \emph{reversible information processing} within the system, defined as the rate at which internal distinctions are transformed while remaining recoverable without additional irreversible cost. Equivalently, $\dot{I}_{\mathrm{rev}}(t)$ quantifies information processing that preserves membership in dynamically induced equivalence classes under the chosen admissible system boundary and operational coarse-graining. By contrast, $\dot{I}_{\mathrm{irr}}(t)$ accounts for irreversible collapse of such distinctions.

Over a finite horizon $T$, define the cumulative preserved (reversible) information as:

\begin{align}
\label{eq:auto:0038}
I_{\mathrm{rev},T} = \int_0^T \dot{I}_{\mathrm{rev}}(t)\,dt
\end{align}

We then define the \emph{cumulative consciousness} over $[0,T]$ as:
\begin{align}
\label{eq:auto:0039}
\kappa_T = \frac{W_{\mathrm{causal},T}}{I_{\mathrm{rev},T}}
\end{align}
which quantifies the amount of causal structural work extracted per nat of preserved information. This ratio is well defined whenever $I_{\mathrm{rev},T} > 0$. An instantaneous analogue, utilising the power fluxes defined in \S\ref{sec:Intelligence}, may likewise be defined as:

\begin{align}
\label{eq:auto:0040}
\kappa(t) = \frac{\dot{W}_{\mathrm{causal}}(t)}{\dot{I}_{\mathrm{rev}}(t)}
\end{align}

serving as a local diagnostic of how preserved internal structure contributes to the extraction of causal structural capacity at time $t$.

Just as we isolated the goal-aligned component of intelligence, we can project the consciousness metric onto a task-specific subspace to evaluate functional utility. Utilizing the expected goal-directed work $W_{\mathrm{goal},T}$, we define the \emph{goal-directed consciousness} as:

\begin{align}
\label{eq:auto:0041}
\kappa_{\mathrm{goal},T} = \frac{W_{\mathrm{goal},T}}{I_{\mathrm{rev},T}}
\end{align}

and its corresponding instantaneous rate:

\begin{align}
\label{eq:auto:0042}
\kappa_{\mathrm{goal}}(t) = \frac{\dot{W}_{\mathrm{goal}}(t)}{\dot{I}_{\mathrm{rev}}(t)}
\end{align}

This evaluates how effectively preserved internal structure specifically supports work aligned with external operational goals, distinct from general environmental disruption.

The pair $(\chi_T,\kappa_T)$ separates two complementary informational contributions to physical behaviour. Intelligence $\chi_T$ captures the efficiency of \emph{irreversible}, distinction-destroying computation, while consciousness $\kappa_T$ captures the efficiency with which \emph{preserved internal structure} supports the same macroscopic footprint. In systems capable of sustained, long-horizon behaviour, high $\kappa_T$ indicates that substantial causal influence is enabled by reuse of internal informational organization rather than by continual irreversible reconstruction. In this sense, consciousness quantifies the physical role of persistent informational structure in enabling intelligence over extended horizons.

\section{The Microscopic Limits of Observation: Deriving Quantum Decoherence}
\label{sec:QuantumMeasurement}

To examine the substrate neutrality of the Conservation-Congruent Encoding (CCE) framework, we apply it to a canonical limit of physical observation: the quantum mechanical double-slit experiment. In standard quantum theory, the collapse of the wavefunction upon measurement is often introduced as an abstract mathematical postulate. However, viewed through the lens of our framework, the destruction of the interference pattern can be modelled as a consequence of the energetic constraints required to establish a CCE, as derived in \S\ref{sec:EncodingPhysicalState}.

Extracting a classical bit of information from a quantum system requires a macroscopic detector to dynamically instantiate a metastable internal memory state. Under CCE, carving out this dynamically isolated equivalence class requires the irreversible export of a conserved quantity. This framework thus links the abstract concept of quantum ``observation'' directly to the physical dissipation required to record information.

\subsection{The Unobserved System: Pure Reversible Flow}
When a particle is fired at the double slits and unobserved, it exists in a coherent spatial superposition. In the language of the present framework, the particle's state evolves entirely within the reversible, non-dissipative subspace of the universe's dynamics. Because no macroscopic distinctions are recorded by the environment, the physical state does not collapse into dynamically isolated equivalence classes.

In this regime, the dynamics are entirely structure-preserving. Within the metriplectic decomposition introduced earlier, the generalised entropy gradient vanishes ($\nabla\Xi = 0$). The evolution of the system's density matrix $\rho_t$ is driven solely by the skew-symmetric Poisson structure $J(\rho_t)$, generating the standard unitary von Neumann equation:
\begin{align}
\label{eq:auto:0043}
\dot{\rho}_t = J(\rho_t)\nabla H(\rho_t) = -\frac{i}{\hbar}[\hat{H}, \rho_t],
\end{align}
where $\hat{H}$ is the Hamiltonian operator. This reversible flow preserves all off-diagonal phase coherences, yielding the classic wave interference pattern on the detector screen.

\subsubsection{The Detector and the Formation of a CCE}

The physical situation changes fundamentally when a macroscopic detector is placed at the slits to determine which path the particle took. To operate as a continuous measuring device, the detector cannot simply entangle with the particle once and freeze; it must function as an open macroscopic system, continuously tracking the particle's state while maintaining a finite memory capacity. Within our framework, this means the detector must actively instantiate, and subsequently erase and reset, Conservation-Congruent Encodings (CCEs). It must drive its own internal state into one of two metastable macroscopic basins (representing the logical bit $L \in \{\text{Slit A}, \text{Slit B}\}$), and then clear that basin to register the next interaction.

While the initial entanglement between the particle and the detector's pointer state can be reversible (unitary), the requirement to continuously reset the detector's operational interface to maintain a steady measurement rate is strictly irreversible. To clear and stably reset $1$ bit (or $\ln 2$ nats) of ``which-path'' information, the detector must export a conserved quantity to the environmental bath. For a typical quantum-optical or mechanical channel, where the fundamental scale is the reduced Planck constant ($\alpha = \hbar$) and the conjugate force driving the reset is the characteristic measurement frequency ($\mathcal{F} = \omega_{\mathrm{meas}}$), the minimum physical cost required to irreversibly erase this specific CCE basin is defined by our generalised Landauer relation:
\begin{align}
\label{eq:auto:0044}
E_{\min}^{\mathrm{CCE}} = \hbar \omega_{\mathrm{meas}} \ln 2.
\end{align}

Observation is therefore not a passive mathematical projection. Sustaining a macroscopic non-demolition measurement record requires continuous, irreversible physical erasure of a CCE. The physical action dissipated to pay this reset cost can be interpreted as the environmental exhaust that suppresses the reversible phase geometry of the quantum superposition.

\subsubsection{Recovering the Lindblad Master Equation}

\begin{figure}[h]
    \centering
    \includegraphics[width=0.8\textwidth]{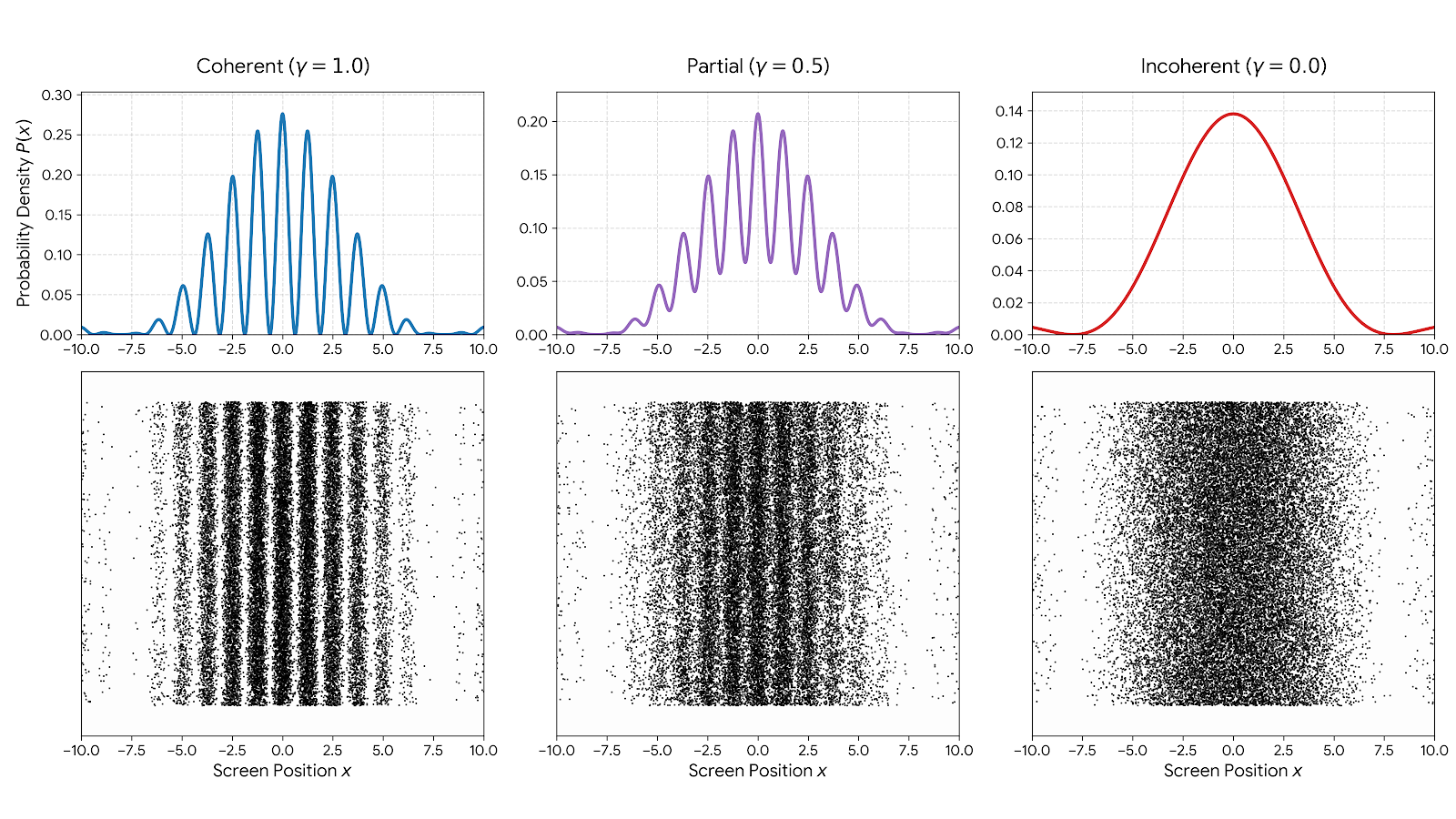}
    \caption{\small Solutions of the double-slit experiment in our CCE setup. A macroscopic which-path detector instantiates and then irreversibly erases a CCE over $L\in\{\text{Slit A},\text{Slit B}\}$, incurring the generalised Landauer reset cost. In this formulation the resulting decoherence rate $\gamma$ is fixed by the detector's CCE reset power and is not a free parameter.}
    \label{fig:example_image}
    
\end{figure}

This observer effect is captured naturally by the full metriplectic equation. In standard quantum mechanics, the state of the system is given by the density matrix $\rho_t$, and the kinematics are governed by a free-energy functional $F = H - \Xi$, where $H$ is the Hamiltonian and $\Xi$ is the generalised entropy potential associated with the coherence of the state. The CCE equation of motion is defined by the competing reversible and irreversible gradients:

\begin{align}
\label{eq:auto:0045}
\dot{\rho}_t = J(\rho_t)\nabla H(\rho_t) - \mathcal{R}(\rho_t)\nabla\Xi(\rho_t).
\end{align}
The reversible flow is generated by the skew-symmetric Poisson tensor $J(\rho_t)$. Mapping this to the Hilbert-Schmidt inner product $\langle A, B \rangle = \mathrm{Tr}(A^\dagger B)$, the geometric action of $J$ on the Hamiltonian gradient ($\nabla H = \hat{H}$) yields the standard von Neumann commutator:

\begin{align}
\label{eq:auto:0046}
J(\rho_t)\nabla H(\rho_t) = -\frac{i}{\hbar}[\hat{H}, \rho_t],
\end{align}
which exactly describes the unobserved, unitary evolution of the wave superposition. When the detector is active, the physical requirement to continuously form and reset Conservation-Congruent Encodings activates the dissipative metric tensor $\mathcal{R}(\rho_t)$. To extract ``which-path'' information without absorbing the particle's kinetic energy (a non-demolition measurement), the macroscopic detector couples to the spatial observable of the particle, represented by a Hermitian projector $\hat{A}$. To ensure that the resulting irreversible flow represents a memoryless, Completely Positive, Trace-Preserving (CPTP) Markovian process, we invoke the Born-Markov approximation. This requires a strict separation of timescales: the macroscopic detector's reset cycle—the time required to irreversibly erase the CCE and dump the resulting exhaust into the thermal bath—must be orders of magnitude shorter than the characteristic coherence time of the observed quantum system ($\tau_{\mathrm{reset}} \ll \tau_{\mathrm{sys}}$). Under this condition, the environment retains no memory of the interaction. By the Gorini-Kossakowski-Sudarshan-Lindblad (GKSL) theorem, any such Markovian CPTP evolution must geometrically take a double-commutator form. Therefore, to correctly map this physical constraint into the metriplectic framework, the metric superoperator $\mathcal{R}(\rho_t)$ must be defined by its geometric action on a generic gradient $X$:

\begin{align}
\label{eq:auto:0047}
\mathcal{R}(\rho_t) \circ X = -\frac{\gamma}{2\lambda} \big[ \hat{A}, [\hat{A}, X] \big],
\end{align}
where $\gamma$ is the measurement coupling rate, and $\lambda$ is a scaling constant. Under continuous observation, the driving force pulling the quantum system toward a classical mixture is the physical penalty on off-diagonal coherence. It is crucial here to distinguish the thermodynamic fuel of the computation from its kinematic trajectory. While the absolute minimum generalised Landauer cost required to reset the CCE is bounded by the von Neumann entropy (as established in \S 2.3.3), the geometric flow of decoherence through state space is driven dynamically by the decay of state purity. Therefore, we define the generalised metric potential driving this kinematic relaxation using the Tsallis linear entropy (purity), $\Xi(\rho_t) = -\frac{\lambda}{2}\mathrm{Tr}(\rho_t^2)$. The corresponding Fréchet derivative yields the generalised entropy gradient:

\begin{align}
\label{eq:auto:0048}
\nabla\Xi(\rho_t) = -\lambda \rho_t.
\end{align}
Contracting the metric tensor with this specific informational gradient directly yields the irreversible flow. The scaling constant $\lambda$ cancels, yielding the CPTP dissipator:

\begin{align}
\label{eq:auto:0049}
- \mathcal{R}(\rho_t) \circ \nabla\Xi(\rho_t) = -\frac{\gamma}{2} \big[ \hat{A}, [\hat{A}, \rho_t] \big].
\end{align}
Combining the reversible and irreversible flows back into the CCE equation of motion seamlessly embeds the Lindblad Master Equation:

\begin{align}
\label{eq:auto:0050}
\dot{\rho}_t = -\frac{i}{\hbar}[\hat{H}, \rho_t] - \frac{\gamma}{2} \big[ \hat{A}, [\hat{A}, \rho_t] \big].
\end{align}
Here lies the central result of this mapping: within standard quantum mechanics, the decoherence rate $\gamma$ is introduced as an empirical fitting parameter. Crucially, under the CCE framework, it is strictly bound by physical exhaust. For the detector to continuously resolve the particle's path, it must sustain a continuous rate of CCE erasure. From the generalised thermodynamic bounds established in \S\ref{sec:IrreversibilityCost}, the continuous physical power $\mathcal{P}_{\mathrm{det}}$ exhausted by the detector to drive this macroscopic reset is strictly bounded by the CCE energetic relation:

\begin{align}
\label{eq:auto:0051}
\mathcal{P}_{\mathrm{det}}(t) \ge \mathcal{F} \alpha \dot{I}_{\mathrm{irr}}(t),
\end{align}
where $\alpha$ is the characteristic fundamental scale of the channel and $\mathcal{F}$ is the intensive conjugate force. Because the double-commutator dissipator governs the exponential decay of the off-diagonal interference terms, the temporal rate of this phase destruction ($\gamma$) must perfectly match the rate at which the detector physically clears its classical bits. Therefore, assuming $\hat{A}$ is normalised to the slit paths, the coupling rate is explicitly defined by the detector's generalised thermodynamic exhaust:

\begin{align}
\label{eq:auto:0052}
\gamma(t) = c \frac{\mathcal{P}_{\mathrm{det}}(t)}{\mathcal{F} \alpha},
\end{align}
where $c$ is a dimensionless proportionality constant relating the detector's physical reset efficiency to the quantum phase-damping rate. For a detector multiplexing across multiple physical channels, the denominator naturally generalises to the vector inner product $\mathbf{F} \cdot \boldsymbol{\alpha}$ derived previously. To explicitly demonstrate how this physical exhaust destroys the quantum interference, we compare the formal solutions of the density matrix. Let the particle's spatial state be defined in the basis of the two paths, $|0\rangle$ and $|1\rangle$.  The measurement operator distinguishing the paths is $\hat{A} = |0\rangle\langle 0| - |1\rangle\langle 1|$. The action of the double-commutator on the off-diagonal coherence term $|0\rangle\langle 1|$ evaluates strictly to:
\begin{align}
\label{eq:auto:0053}
\big[\hat{A}, [\hat{A}, |0\rangle\langle 1|]\big] = 4|0\rangle\langle 1|.
\end{align}
The differential equation for the off-diagonal elements under the full metriplectic flow thus becomes:
\begin{align}
\label{eq:auto:0054}
\dot{\rho}_{01}(t) = -i\omega_{01}\rho_{01}(t) - 2\gamma(t)\rho_{01}(t).
\end{align}
Integrating this equation yields the formal solution for the coherence term at the time of screen impact $T$:
\begin{align}
\label{eq:auto:0055}
\rho_{01}(T) = \rho_{01}(0) \exp\left(-i\int_0^T \omega_{01} dt\right) \exp\left(-2\int_0^T \gamma(t) dt\right).
\end{align}
Substituting the strict physical constraint that the coupling rate is driven by the generalised dissipative exhaust ($\gamma(t) = c \frac{\mathcal{P}_{\mathrm{det}}(t)}{\mathcal{F} \alpha}$), the total decay of the interference term is governed directly by the cumulative physical energy dissipated by the detector ($E_{\mathrm{det},T}$):

\begin{align}
\label{eq:auto:0056}
\exp\left(-2c \int_0^T \frac{\mathcal{P}_{\mathrm{det}}(t)}{\mathcal{F} \alpha} dt\right) = \exp\left(-2c \frac{E_{\mathrm{det},T}}{\mathcal{F} \alpha}\right) = \exp\left(-2c I_{\mathrm{irr},T}\right).
\end{align}
The resulting spatial probability distribution on the screen under observation is therefore:
\begin{align}
\label{eq:auto:0057}
P_{\mathrm{CCE}}(x) = \langle x | \rho_T | x \rangle = \frac{1}{2}|\psi_0(x)|^2 + \frac{1}{2}|\psi_1(x)|^2 + e^{-2c I_{\mathrm{irr},T}} \mathrm{Re}\left(\psi_0(x)\psi_1^*(x)\right).
\end{align}
This formal solution makes the mechanism of measurement mathematically explicit. The visibility of the quantum interference fringe is strictly bounded by $e^{-2c I_{\mathrm{irr},T}}$. If the macroscopic detector performs zero irreversible erasures ($I_{\mathrm{irr},T} = 0$), the exponential term evaluates to $1$, and the interference pattern is perfectly preserved. If the detector completes a sufficient number of erasures to establish a perfectly distinguishable and continuous CCE record, the exponential suppression physically drives the cross-term to zero, yielding the classical addition of independent probabilities. The Lindblad dissipator is thus recovered as a mathematically consistent representation of the physical power required to reset an open system's boundary, rather than as an ad-hoc term used only to model experimental fuzziness. In this interpretation, the disappearance of the quantum interference pattern can be modelled as arising from the continuous flow of irreversible exhaust ($\dot{I}_{\mathrm{irr}}$), which provides the dissipative friction ($\gamma$) that breaks the skew-symmetric flow required for superposition.

\section{The Macroscopic Limits of Observation: Emergent Spacetime and Epistemic Collapse}
\label{sec:GravitationalLimits}
In \S\ref{sec:QuantumMeasurement}, we established that the extraction of classical information from a quantum superposition necessitates the export of a conserved quantity, mathematically recovering the Lindblad Master Equation. We now extend our analysis from microscopic measurement to macroscopic by analysing the measurement of black holes. If microscopic dynamics are fundamentally reversible and volume-preserving, one interpretive possibility is that gravity is not a fundamental microscopic mechanism. Within the Conservation-Congruent Encoding (CCE) framework, we therefore explore the hypothesis that gravity emerges at the macroscopic level as the geometric footprint of the universe's informational constraints. On this reading, gravity is treated as a macroscopic dynamical property related to the physical capacity of localised systems to measure, record, and erase information. This perspective aligns with and extends the foundational thermodynamic and entropic gravity programs pioneered by Jacobson \cite{Jacobson1995} and Verlinde \cite{Verlinde2011}. However, while standard theoretical treatments often rely on phenomenological concepts of heat and thermal baths, the CCE framework generalises this relationship. We show that, within this interpretive framework, thermodynamic bounds on spacetime are consistent with substrate-neutral information processing limits. Instantiating a CCE to record a macroscopic state requires a generalised dissipation flux that can backreact on the metric. To ground this hypothesis, we mathematically bridge the microscopic informational friction of the metriplectic flow with the macroscopic curvature of spacetime ($g_{\mu\nu}$).

\subsection{Informational Limits and the Bekenstein-Hawking Area Law}
\label{sec:emergence_bh}

We must first resolve a fundamental geometric discrepancy: the absolute microscopic dynamics of a closed universe are phase-space volume preserving, yet macroscopic gravitational information is bounded strictly by the area of its bounding causal horizons.
Consider the exact, un-coarse-grained microscopic state of a closed system. Its temporal evolution is governed entirely by the reversible, Hamiltonian vector field $v_{\mathrm{rev}} = J\nabla H$. The evolution of an arbitrary microscopic phase-space volume $\mathcal{V}_{\mathrm{micro}}$ is dictated by the Lie derivative of the volume form $\Omega$, determined by the flow's divergence:

\begin{align}
 \label{eq:auto:0058}
\mathcal{L}_{v_{\mathrm{rev}}} \Omega = (\nabla \cdot v_{\mathrm{rev}}) \Omega 
\end{align}

By Liouville's Theorem, the divergence of this skew-symmetric Hamiltonian flow vanishes identically ($\nabla \cdot (J\nabla H) = 0$). Consequently, a universe described purely at the exact microscopic limit possesses no mechanism to constrain information density to an area; its raw structural capacity scales volumetrically. The Bekenstein-Hawking geometric area law ($I_{\mathrm{max}} \propto A$) must therefore emerge strictly as a macroscopic phenomenon, originating from the irreversible metric tensor $\mathcal{R}$.
To understand why macroscopic measurement mathematically necessitates the activation of this tensor, we must view observation as an inherently coarse-graining operation. Recording a unit of information related to a macroscopic property (instantiating a CCE) requires mapping multiple distinct microscopic initial states into a single dynamically isolated equivalence class. Because the exact microscopic flow $v_{\mathrm{rev}}$ is strictly bijective, it is geometrically impossible to form this many-to-one macroscopic basin under pure reversible dynamics. Instead, the macroscopic detector must trace over unobserved environmental degrees of freedom. This coarse-graining induces an effective, open-system mesoscopic state $x_t$ that evolves under the full metriplectic decomposition:
\begin{align}
 \label{eq:auto:0059}
\dot{x}_t = v_{\mathrm{rev}}(x_t) + v_{\mathrm{irr}}(x_t) = J(x_t)\nabla H(x_t) - \mathcal{R}(x_t)\nabla \Xi(x_t) 
\end{align}
Here, the contractive vector field required to merge trajectories and form the measurement record is uniquely provided by the dissipative flow $v_{\mathrm{irr}} = -\mathcal{R}\nabla\Xi$. Because $\mathcal{R}$ is positive semi-definite and $\Xi$ acts as a convex entropy potential near these equilibria, the divergence of this irreversible measurement flow is strictly negative:
\begin{align}
 \label{eq:auto:0060}
\nabla \cdot v_{\mathrm{irr}} = -\nabla \cdot (\mathcal{R}\nabla \Xi) \equiv -\Lambda(x) < 0 
\end{align}
Integrating this divergence over the macroscopic operational phase-space volume yields a strict exponential geometric collapse of the effective state space:
\begin{align}
 \label{eq:auto:0061}
\frac{d\mathcal{V}_{\mathrm{macro}}}{dt} = -\int_{\mathcal{V}_{\mathrm{macro}}} \Lambda(x) \Omega \implies \mathcal{V}_{\mathrm{macro}}(t) \le \mathcal{V}_{\mathrm{macro}}(0) e^{-\bar{\Lambda} t} 
\end{align}
where $\bar{\Lambda} = \inf_{x \in \mathcal{V}_{\mathrm{macro}}} \Lambda(x) > 0$.
This dissipative contraction dynamically foliates the state space, projecting the continuous mesoscopic phase-space volume onto a lower-dimensional attractor. Crucially, this abstract phase-space collapse rigorously constrains physical spatial geometry. In a local field theory, microscopic degrees of freedom scale extensively with the 3D physical volume. However, macroscopic observation across a causal horizon requires continuously tracing out the unobservable internal bulk. Because General Relativity is fundamentally a diffeomorphism-invariant gauge theory, the physical bulk Hamiltonian is composed entirely of constraints that vanish on-shell. Consequently, when the interior microstates are integrated out, the surviving measurable energy and informational capacity of the subregion mathematically reduce purely to a boundary integral (analogous to the ADM or Brown-York quasi-local mass). The physical capacity to store dynamically distinct CCE basins can no longer scale volumetrically; it is strictly projected onto, and bounded by, the discrete geometric elements of the 2D causal horizon. But why is the minimal geometric extent of a CCE basin exactly proportional to the Planck area ($l_p^2 = \frac{G\hbar}{c^3}$)? Within this framework, this is not an arbitrary quantum gravity postulate, but a strict thermodynamic derivation.
To irreversibly record a macroscopic distinction—effectively tracing out a microscopic degree of freedom as it crosses the causal horizon—the exterior observer must pay the generalised Landauer cost, $\delta Q = k_B T_U \delta I_{\mathrm{irr}}$. For a localised macroscopic observer, the vacuum establishes the ambient thermal state at the Unruh temperature, $T_U = \frac{\hbar\kappa}{2\pi k_B c}$, where $\kappa$ is the local surface gravity. The minimal physical heat transferred to erase $\delta I_{\mathrm{irr}}$ nats of microscopic information is therefore:
\begin{align}
 \label{eq:auto:0062}
\delta Q = \frac{\hbar\kappa}{2\pi c} \delta I_{\mathrm{irr}} 
\end{align}
Simultaneously, the classical First Law of Black Hole Mechanics states that the physical energy required to geometrically deform a causal horizon by an area $\delta A$ is $dE = \frac{\kappa c^2}{8\pi G} \delta A$. 

Because the macroscopic horizon acts as the physical causal interface that must absorb this irreversible measurement exhaust to instantiate the coarse-grained state, we equate the Landauer energy cost to the geometric deformation ($\delta Q = dE$):

\begin{align}
 \label{eq:auto:0063}
\frac{\hbar\kappa}{2\pi c} \delta I_{\mathrm{irr}} = \frac{\kappa c^2}{8\pi G} \delta A 
\end{align}

The local kinematic variable $\kappa$ perfectly cancels from both sides, revealing a universal, observer-independent geometric conversion rate for information. Rearranging for the area required to store one unit of irreversibly coarse-grained information ($\delta I_{\mathrm{irr}}$) yields:
\begin{align}
 \label{eq:auto:0064}
\delta A = \left(\frac{4G\hbar}{c^3}\right) \delta I_{\mathrm{irr}} = 4 l_p^2 \delta I_{\mathrm{irr}} 
\end{align}
The Planck area ($l_p^2$) is therefore the exact geometric equivalent of the generalised Landauer cost in a gravitational channel. The total macroscopic informational capacity of the local vacuum ($I_{\mathrm{max}}$) is simply the total boundary area divided by this fundamental CCE basin size, directly recovering the exact Bekenstein-Hawking entropy relation:
\begin{align}
 \label{eq:auto:0065}
S_{BH} = k_B I_{\mathrm{max}} = k_B \frac{A}{4 l_p^2} = \frac{k_B c^3 A}{4G\hbar} 
\end{align}
Because this derivation mathematically establishes the localised heat-to-area relation ($\delta Q \propto \delta A$) purely from fundamental informational erasure limits, it provides a candidate microscopic mechanism for Jacobson's thermodynamic formulation of gravity. Within the CCE framework, the Einstein Field Equations can be interpreted as macroscopic thermodynamic equations of state governing this $\mathcal{R}$-driven informational channel. Crucially, this bounding horizon is not a static mathematical wall; it is a dynamically maintained operational interface. To sustain these macroscopic equivalence classes against the continuous entanglement of infalling quantum vacuum fluctuations, the horizon operates via the continuous action of the macroscopic dissipative tensor $\mathcal{R}$.
However, by the Fluctuation-Dissipation Theorem, any system exhibiting continuous macroscopic dissipation is not expected to remain a purely absorptive sink; the irreversible tensor $\mathcal{R}$ is accompanied by a conjugate stochastic noise term. While the horizon continuously absorbs the Landauer exhaust of coarse-grained infalling modes, this dynamic dissipation is coupled to microscopic thermal emission. Viewed through this lens, Hawking radiation aligns naturally with the stochastic fluctuation counterpart to $\mathcal{R}$-driven dissipation. Furthermore, because the asymptotic vacuum of our universe acts as a lower-temperature environment, this continuous thermal emission is consistent with a net outward flux. Since this shed energy is drawn from the structural mass of the spacetime itself, it is likewise consistent with net black hole evaporation. Holographic bounds and black hole radiance are thus interpreted here as geometric and non-equilibrium thermodynamic consequences of measurement-associated phase-space contraction.

\subsection{Outlining a Recovery of the Einstein Field Equations}

To rigorously map the discrete, scalar thermodynamic relationship of \S\ref{sec:emergence_bh} into the continuous tensor geometry of a spacetime manifold, we must formalise the connection between the abstract metriplectic dissipation of the agent and the macroscopic stress-energy tensor, $T_{\mu\nu}$. 

In the CCE framework, the irreversible formation of an equivalence class is driven by the dissipative flow $v_{\mathrm{irr}} = -\mathcal{R}\nabla\Xi$. The local rate of generalised physical dissipation (entropy production) is governed exactly by the inner product of the generalised entropy gradient and this flow:

\begin{align}
\label{eq:auto:0066}
\dot{S}_{\mathrm{phys}}(x,t) = \langle \nabla\Xi(x,t), \mathcal{R}(x,t) \nabla\Xi(x,t) \rangle
\end{align}

When coarse-grained over a macroscopic spatial volume, this continuous irreversible information processing manifests physically as a localised power density of measurement exhaust, $\mathcal{P}_{exh}$. Scaling this by the channel's conjugate force $\mathcal{F}$ and characteristic scale $\alpha$, we first define an effective mass-equivalent exhaust tensor $\widetilde{T}_{\mu\nu}$. The canonical stress-energy tensor used in Einstein's equations is then recovered by $T_{\mu\nu} = c^2 \widetilde{T}_{\mu\nu}$. Thus, $T_{\mu\nu}$ appears only after restoring the mass--energy conversion factor.

Therefore, the flux of energy crossing a local null horizon generated by the vector field $k^\mu$ is strictly the spacetime projection of the observer's continuous coarse-graining:

\begin{align}
\label{eq:auto:0067}
\widetilde{T}_{\mu\nu} k^\mu k^\nu \propto \mathcal{F}\alpha \langle \nabla\Xi, \mathcal{R} \nabla\Xi \rangle
\end{align}

This formal equivalence allows us to build upon the thermodynamic formulation of spacetime pioneered by Ted Jacobson, adapting his framework to accommodate continuous informational dissipation as the strict physical source term.

In \S\ref{sec:emergence_bh}, we established that localised generalised Landauer exhaust directly deforms the causal horizon according to the scalar energy relation $\delta Q = \frac{\kappa c^2}{8\pi G} \delta A$. We now map this discrete thermodynamic statement into the continuous tensor geometry of a spacetime manifold, building upon \cite{Jacobson1995} while keeping this $c^2$ scaling explicit at the heat-flux level.
By the Equivalence Principle, we consider a local, approximately flat Minkowski patch of spacetime. Within this patch, an observer undergoing uniform proper acceleration generates a local Rindler horizon, $\mathcal{H}$. This horizon is a null hypersurface generated by a congruence of null vectors $k^\mu$. Let $\lambda$ be the affine parameter along these null geodesics, parameterised such that $\lambda = 0$ corresponds to the observer's present spatial slice, extending into the past ($-\epsilon \le \lambda \le 0$). In our metriplectic framework, $\widetilde{T}_{\mu\nu}$ is the effective continuum source for Landauer exhaust, with $T_{\mu\nu}=c^2\widetilde{T}_{\mu\nu}$ reserved for the canonical Einstein source. The localised heat flux $\delta Q$ crossing the horizon is the integral of this exhaust tensor contracting with the horizon's generator. Near the horizon, the approximate boost Killing vector field is $\chi^\mu \approx -\kappa \lambda k^\mu$, where $\kappa$ is the geometric surface gravity. The energy flux across the horizon surface element $d\Sigma^\nu = k^\nu d\lambda dA$ is therefore:
\begin{align}
 \label{eq:auto:0068}
\delta Q = \int_{\mathcal{H}} \widetilde{T}_{\mu\nu} \chi^\mu d\Sigma^\nu = \kappa \int d^2A \int_{-\epsilon}^{0} \widetilde{T}_{\mu\nu} k^\mu k^\nu (-\lambda) d\lambda 
\end{align}
Simultaneously, the geometric deformation of the horizon's area ($\delta A = \int_{\mathcal{H}} \theta d\lambda dA$) is governed by the continuous expansion scalar $\theta$ of the null congruence. The evolution of this expansion is strictly dictated by the Raychaudhuri equation:
\begin{align}
 \label{eq:auto:0069}
\frac{d\theta}{d\lambda} = -\frac{1}{2}\theta^2 - \sigma_{\mu\nu}\sigma^{\mu\nu} - R_{\mu\nu} k^\mu k^\nu 
\end{align}
where $\sigma_{\mu\nu}$ is the shear tensor and $R_{\mu\nu}$ is the spacetime Ricci curvature tensor. Although the underlying metriplectic process is dynamically dissipative, we treat the local, unperturbed macroscopic vacuum at $\lambda = 0$ as an instantaneous reference state of quasi-static thermodynamic equilibrium. Thus, we isolate the horizon's strictly perturbative geometric response by assuming the initial expansion and shear vanish ($\theta = 0$, $\sigma_{\mu\nu} = 0$). Over an infinitesimal affine interval, treating the curvature as locally constant, the Raychaudhuri equation linearises to $\frac{d\theta}{d\lambda} = -R_{\mu\nu} k^\mu k^\nu$. Integrating this yields the expansion at any point along the generator:
\begin{align}
 \label{eq:auto:0070}
\theta(\lambda) = - \lambda R_{\mu\nu} k^\mu k^\nu 
\end{align}
Substituting this expansion into the area integral, the total geometric deformation of the horizon becomes:
\begin{align}
 \label{eq:auto:0071}
\delta A = \int d^2A \int_{-\epsilon}^{0} R_{\mu\nu} k^\mu k^\nu (-\lambda) d\lambda 
\end{align}
We now apply the strict physical constraint derived from the CCE framework: the macroscopic metric must physically absorb the Landauer exhaust of coarse-graining. Equating the informational heat flux to the geometric area deformation via the same scalar bound derived in \S\ref{sec:emergence_bh} ($\delta Q = \frac{\kappa c^2}{8\pi G} \delta A$), we substitute our tensor integrals:
\begin{align}
 \label{eq:auto:0072}
\kappa \int d^2A \int_{-\epsilon}^{0} \widetilde{T}_{\mu\nu} k^\mu k^\nu (-\lambda) d\lambda = \frac{\kappa c^2}{8\pi G} \int d^2A \int_{-\epsilon}^{0} R_{\mu\nu} k^\mu k^\nu (-\lambda) d\lambda 
\end{align}
The local surface gravity $\kappa$ strictly cancels, reflecting the universality of the measurement bound regardless of the observer's specific acceleration. Because this thermodynamic equality must hold for all localised observers—and therefore for all null vector fields $k^\mu$ at any point in spacetime—the integrands must be identical for any null direction:
\begin{align}
 \label{eq:auto:0073}
\widetilde{T}_{\mu\nu} k^\mu k^\nu = \frac{c^2}{8\pi G} R_{\mu\nu} k^\mu k^\nu 
\end{align}
Multiplying by $c^2$ and restoring the canonical stress-energy tensor ($T_{\mu\nu}=c^2\widetilde{T}_{\mu\nu}$) yields the standard relativistic coupling:
\begin{align}
\label{eq:auto:0074}
T_{\mu\nu} k^\mu k^\nu = \frac{c^4}{8\pi G} R_{\mu\nu} k^\mu k^\nu
\end{align}
A standard tensor identity implies that if the contraction of a symmetric tensor with $k^\mu k^\nu$ vanishes for all null vectors ($g_{\mu\nu} k^\mu k^\nu = 0$), that tensor must be proportional to the metric tensor $g_{\mu\nu}$. Therefore, we can strip the null vectors by introducing an arbitrary scalar function $f$:
\begin{align}
 \label{eq:auto:0075}
R_{\mu\nu} + f g_{\mu\nu} = \frac{8\pi G}{c^4} T_{\mu\nu} 
\end{align}
To determine the scalar $f$, we invoke the fundamental macroscopic requirement of local energy-momentum conservation ($\nabla^\mu T_{\mu\nu} = 0$). Taking the covariant derivative of both sides and applying the contracted Bianchi identity ($\nabla^\mu R_{\mu\nu} = \frac{1}{2}\nabla_\nu R$), we uniquely constrain $f = -\frac{1}{2}R + \Lambda$, where $R$ is the Ricci scalar. Crucially, the cosmological constant $\Lambda$ emerges here naturally as a constant of integration, conceptually representing the baseline informational capacity (or zero-point equilibrium entropy) of the unperturbed vacuum itself. Substituting this back yields the full Einstein Field Equations:
\begin{align}
 \label{eq:auto:0076}
R_{\mu\nu} - \frac{1}{2}R g_{\mu\nu} + \Lambda g_{\mu\nu} = \frac{8\pi G}{c^4} T_{\mu\nu} 
\end{align}
The emergence of General Relativity within this framework is therefore complete, but its ontological meaning is reframed. Within the CCE framework, the Einstein Field Equations can be interpreted as macroscopic equations of state that continuously balance the generalised Landauer exhaust of observation against the informational capacity of the local vacuum. The spacetime metric can then be modelled as curving because the irreversible metriplectic tensor $\mathcal{R}$ continuously exports the heat of coarse-graining into the geometry of the manifold.

This framework motivates a shift in how the physical world we experience is interpreted. In classical mechanics, 3D space is treated as a fundamental, preexisting container. However, the exact microscopic reality of a closed system is a vast, high-dimensional phase space evolving reversibly and preserving its volume. A macroscopic observer cannot interact with this raw, un-coarse-grained reality. By instantiating continuous CCEs, the observer traces over microscopic degrees of freedom, producing a dissipative, mathematically irreversible collapse of this vast phase-space volume onto a lower-dimensional attractor. Under this hypothesis, the classical 3D spatial bulk we perceive—and the objects within it—behaves mathematically as this attractor. Our experience of continuous 3D space can then be interpreted not as a direct view of the fundamental universe, but as the geometric shadow of our macroscopic equivalence classes, structured by the capacity required to hold them.

\subsection{Bounding the Physical Cost of Horizon Measurement}
\label{sec:horizon_measurement}
Having established that gravity acts as the physical bookkeeping mechanism for generalised information erasure, we now apply this to the measurement of a Schwarzschild black hole of initial mass $M_{\mathrm{init}}$. Any physical observer must be situated at a finite radial coordinate $R_{\mathrm{obs}} \geq R_S$, where $R_S = \frac{2GM_{\mathrm{init}}}{c^2}$. Due to gravitational redshift, the effective CCE conjugate product at the detector is governed by the localised Tolman relation:
\begin{align}
\label{eq:auto:0077}
(\mathcal{F}\alpha)_{\mathrm{loc}} = \frac{\hbar c^3}{8\pi G M \sqrt{1 - \frac{2GM}{R_{\mathrm{obs}}c^2}}}
\end{align}
To record the information emitted by the black hole as it evaporates, the detector must continuously instantiate metastable CCE basins. The differential number of nats extracted for a given mass loss is $|dI_{\mathrm{irr}}| = -\frac{8\pi G M}{\hbar c} dM$. Following the generalised Landauer bound, the minimum irreversible physical energy required to establish and reset these encodings locally is $dE_{\mathrm{min}}^{\mathrm{CCE}}(R_{\mathrm{obs}}) = (\mathcal{F}\alpha)_{\mathrm{loc}} \, |dI_{\mathrm{irr}}|$. Substituting $(\mathcal{F}\alpha)_{\mathrm{loc}}$ and $|dI_{\mathrm{irr}}|$, the fundamental quantum and thermodynamic scales ($\hbar$ and $\alpha$) cancel perfectly. While the local differential bound remains explicitly dependent on the macroscopic geometric scale parameterised by $G$, the exact local measurement cost per unit mass simplifies to:
\begin{align}
\label{eq:auto:0078}
dE_{\mathrm{min}}^{\mathrm{CCE}}(R_{\mathrm{obs}}) = - \frac{c^2}{\sqrt{1 - \frac{2GM}{R_{\mathrm{obs}} c^2}}} dM.
\end{align}
Integrating this measurement energy from $M=M_{\mathrm{init}}$ down to $M=0$ yields the total structural mass-energy demanded by the complete measurement process:
\begin{align}
\label{eq:auto:0079}
E_{\mathrm{diss, total}}(R_{\mathrm{obs}}) = \int_{0}^{M_{\mathrm{init}}} \frac{c^2}{\sqrt{1 - \frac{2GM}{R_{\mathrm{obs}}c^2}}} dM = 2 M_{\mathrm{init}} c^2 \left( \frac{R_{\mathrm{obs}}}{R_S} \right) \left( 1 - \sqrt{1 - \frac{R_S}{R_{\mathrm{obs}}}} \right).
\end{align}
Evaluating this exact analytical integral yields a physical corollary regarding observer proximity, namely a horizon penalty. If an observer attempts to record the evaporation locally, hovering exactly at the initial event horizon ($R_{\mathrm{obs}} \to R_S$), the integrated total energy converges despite the locally divergent Tolman temperature:
\begin{align}
\label{eq:auto:0080}
E_{\mathrm{diss, total}}(R_S) = 2 M_{\mathrm{init}} c^2 (1) \left( 1 - \sqrt{1 - 1} \right) = 2 M_{\mathrm{init}} c^2.
\end{align}
Attempting to measure the information locally imposes a large gravitational penalty driven by extreme blueshifted resetting costs, corresponding in this limit to twice the initial rest mass of the black hole. Conversely, evaluating the integral asymptotically ($R_{\mathrm{obs}} \to \infty$) causes the geometric scale $G$ to vanish entirely, yielding exactly $E_{\mathrm{diss, total}} \to M_{\mathrm{init}} c^2$. In this far-field limit, the ultimate measurement bound is substrate-neutral and purely relativistic. This suggests a physical equivalence: the minimal irreversible energy required to record the complete information of a Schwarzschild black hole equals the rest mass of the black hole itself. Crucially, the observer cannot power this continuous observation using the incoming target signal itself. The black hole emits its $M_{\mathrm{init}}c^2$ energy as maximum-entropy Hawking radiation, which provides pure heat but zero free energy. To thermodynamically extract the pure work necessary to power Landauer resets, the observer would have to physically absorb and thermalise the incoming photons—an act that destroys the exact quantum microstates they are attempting to measure. Therefore, in this model, non-destructive measurement requires the observer to bring an independent, equally massive external reservoir of pure structural potential to complete the observation.

\subsection{Gravitational Backreaction: The Metric as a Measurement Constraint}
Because the detector must continuously reset its operational interface to maintain a steady measurement rate $\dot{I}_{\mathrm{irr}}(\tau)$, it radiates a continuous macroscopic exhaust into the local spacetime. To causally intersect and scramble the subsequent outgoing Hawking signals the observer seeks to decode, a fraction of this measurement exhaust must unavoidably be radiated inward, plunging toward the black hole. In General Relativity, the dynamic backreaction of this ingoing flux of stress-energy $T_{\mu\nu}^{(\mathrm{exh})}$ on the spacetime geometry is naturally modeled by the ingoing Vaidya metric:
\begin{align}
\label{eq:auto:0081}
ds^2 = - \left( 1 - \frac{2G \mathcal{M}(v)}{r c^2} \right) c^2 dv^2 + 2c \, dv \, dr + r^2 d\Omega^2,
\end{align}
where $v$ is the advanced time coordinate, and the dynamic mass profile is modelled as driven by the observer's inward generalised dissipation rate ($\mathcal{P}_{\mathrm{exh}}^{\mathrm{in}}$): 
\begin{align}
\label{eq:auto:0082}
\frac{d\mathcal{M}}{dv} = \frac{\mathcal{P}_{\mathrm{exh}}^{\mathrm{in}}}{c^2}
\end{align}
Because this exhaust geometrically plunges inward, it crosses the paths of the subsequent outgoing Hawking modes. Consequently, the local stress-energy of the causal intersection satisfies the Null Energy Condition (NEC) with respect to the outgoing null congruence $k^\mu$, yielding $T_{\mu\nu}^{(\mathrm{exh})} k^\mu k^\nu > 0$. This introduces a geometric focusing term in the Raychaudhuri equation. For a hypersurface-orthogonal null congruence (where the twist tensor vanishes, $\omega_{\mu\nu} = 0$), the expansion evolves as:
\begin{align}
\label{eq:auto:0083}
\frac{d\theta}{d\lambda} = -\frac{1}{2}\theta^2 - \sigma_{\mu\nu}\sigma^{\mu\nu} - \frac{8\pi G}{c^4} T_{\mu\nu}^{(\mathrm{exh})} k^\mu k^\nu.
\end{align}
Crucially, information extraction under the CCE framework requires the discrete, irreversible collapse of equivalence classes. Because these continuous resets occur as discrete physical events at localised detector elements across the observer's shell, the resulting exhaust is not spherically symmetric. Instead, it contains angularly inhomogeneous, high-frequency stochastic fluctuations ($\delta T_{\mu\nu}(\theta, \phi)$). 

This spatial asymmetry breaks global spherical symmetry, generating localised Weyl curvature that can drive the transverse spatial shear tensor ($\sigma_{\mu\nu}$) of local null congruences away from zero. Outgoing null geodesics traversing this $C^2$-smooth but highly perturbed geometry can then experience a chaotic combination of transverse spatial shear and stochastic Shapiro time delays. 

In this framework, this geometric turbulence can destroy the phase coherence of subsequent quantum superpositions propagating through the region. The macroscopic CCE exhaust thus acts as a gravitational Lindblad-like dissipator. The observer's inevitable measurement exhaust physically alters the geodesic trajectories of subsequent outgoing information, scrambling the very causal structure and delicate phase relationships the observer seeks to decode. Observation at this scale is therefore modelled not as a passive mathematical projection, but as the active generation of gravitational backreaction.

\subsection{Exploring Epistemic Collapse at the Bekenstein and Hoop Limits}

A rigorous physical formulation of epistemic limits must account for the dynamic nature of spacetime. A phenomenological pitfall in bounding computational limits is the assumption that an observer naively ``hoards'' the energetic exhaust of computation. In reality, generalised thermodynamic exhaust radiates outward as a dynamic flux and does not accumulate indefinitely within the observer's local radius.

To definitively bound the geometric consequences of measurement without violating relativistic dynamics, we must distinguish between the static structural mass required to store the extracted information and the energetic reservoir required to fuel the irreversible CCE erasures. We evaluate the epistemic collapse on the initial Cauchy hypersurface ($t=0$) under two bounding physical architectures: a dynamically fueled (tethered) observer, and an autonomous (causally isolated) observer.

\begin{conjecture}[Gravitational Epistemic Collapse]
No physically realizable macroscopic observer can instantiate a complete Conservation-Congruent Encoding (CCE) of a black hole's microstate without undergoing gravitational collapse. The minimum safe observation distance is strictly bounded by the Golden Ratio for externally powered observers, and $2.52 R_S$ for autonomous observers.
\end{conjecture}

\begin{proof}
Let a macroscopic observer be localised within a spherical shell at physical radius $R_{\mathrm{obs}} > R_S$. To store the complete microstate of the target black hole, the observer's physical memory substrate must obey the universal Bekenstein bound:
\begin{align}
\label{eq:auto:0084}
S_{\mathrm{obs}} \le \frac{2\pi k_B R_{\mathrm{obs}} M_{\mathrm{mem}} c}{\hbar}
\end{align} 
Even when empty at $t=0$, this memory substrate must possess sufficient structural mass-energy to support the state space. Setting the observer's structural entropy capacity to equal the black hole's initial entropy establishes the minimum required mass of the memory substrate itself:
\begin{align}
\label{eq:auto:0085}
M_{\mathrm{mem}} \ge M_{\mathrm{init}} \left( \frac{R_S}{R_{\mathrm{obs}}} \right).
\end{align}

\textbf{Case 1: The Tethered Observer (The Memory Limit)} \\
Assume an idealised scenario where the dynamic thermodynamic requirements are perfectly outsourced: the required measurement fuel is streamed to the observer from spatial infinity, and the resulting generalised exhaust is perfectly radiated outward without local accumulation. The only static mass strictly required to reside at $R_{\mathrm{obs}}$ at $t=0$ is the target black hole and the memory substrate. The total ADM mass enclosed is $M_{\mathrm{tot}} = M_{\mathrm{init}} + M_{\mathrm{mem}}$.

By the Hoop Conjecture, for the observer to avoid being engulfed by the collective event horizon, their radial coordinate must strictly exceed the Schwarzschild radius of the combined system: $R_{\mathrm{obs}} > \frac{2 G M_{\mathrm{tot}}}{c^2}$. Substituting the memory mass constraint yields:
\begin{align}
\label{eq:auto:0086}
R_{\mathrm{obs}} > R_S \left( 1 + \frac{R_S}{R_{\mathrm{obs}}} \right).
\end{align}

Defining the dimensionless distance parameter $x = R_{\mathrm{obs}}/R_S$, this geometric safety condition reduces to:
\begin{align}
\label{eq:auto:0087}
x > 1 + \frac{1}{x} \implies x^2 - x - 1 > 0.
\end{align}

The positive root of this characteristic polynomial is exactly the Golden Ratio, $\varphi = \frac{1 + \sqrt{5}}{2} \approx 1.618$. Therefore, even under perfectly idealised, zero-accumulation energy routing, the sheer structural mass required to physically store the target's microstates guarantees that any observer crossing $R_{\mathrm{obs}} \le 1.618 R_S$ will instantly trigger mutual gravitational collapse.

\textbf{Case 2: The Autonomous Observer (The Fuel Limit)} \\
A physically autonomous observer cannot rely on infinite cosmic tethers; they must act as a causally closed system at the onset of observation. Therefore, they must carry the structural potential required to drive the generalised Landauer exhaust over the entire measurement horizon. At $t=0$, this unspent fuel reservoir is a local, physical mass. By the equivalence principle, gravity couples identically to this unspent potential.

As derived in \S\ref{sec:horizon_measurement}, the required initial energetic reservoir to execute the full measurement is:
\begin{align}
\label{eq:auto:0088}
M_{\mathrm{fuel}} \ge 2 M_{\mathrm{init}} \left( \frac{R_{\mathrm{obs}}}{R_S} \right) \left( 1 - \sqrt{1 - \frac{R_S}{R_{\mathrm{obs}}}} \right).
\end{align}

The ADM mass of the combined autonomous system at the onset of measurement is $M_{\mathrm{tot}} = M_{\mathrm{init}} + M_{\mathrm{mem}} + M_{\mathrm{fuel}}$. Substituting these constraints into the collapse condition $R_{\mathrm{obs}} > \frac{2 G M_{\mathrm{tot}}}{c^2}$ and defining $y = R_S/R_{\mathrm{obs}}$, the survival condition reduces to:
\begin{align}
\label{eq:auto:0089}
1 > y + y^2 + 2\left( 1 - \sqrt{1 - y} \right).
\end{align}

Isolating the radical term to eliminate it via squaring ($2\sqrt{1-y} > y^2 + y + 1$) yields:
\begin{align}
\label{eq:auto:0090}
4(1-y) > (y^2 + y + 1)^2.
\end{align}

Expanding the right side ($(y^2 + y + 1)^2 = y^4 + 2y^3 + 3y^2 + 2y + 1$) and moving all terms to one side simplifies perfectly to the characteristic polynomial condition:
\begin{align}
\label{eq:auto:0091}
y^4 + 2y^3 + 3y^2 + 6y - 3 < 0.
\end{align}

Analysing the roots of this polynomial reveals that it is only satisfied for $y < 0.3965$, corresponding to a critical initial distance of $R_{\mathrm{obs}} \approx 2.52 R_S$. If an autonomous observer attempts to localise the necessary memory and fuel to instantiate a complete CCE of a black hole at any radius $R_{\mathrm{obs}} \le 2.52 R_S$, the combined mass-energy mathematically guarantees that the collective event horizon expands to instantly engulf the observer before a single bit can be measured.
\end{proof}

\begin{conjecture}[The Hoop Limit on Computation]
When the structural mass-energy required to operate a localised computational memory exceeds the available spatial geometry, the expansion of outgoing null geodesics becomes strictly negative ($\theta < 0$), and the computational space mathematically pinches off to form a trapped surface.
\end{conjecture}

The Hoop Conjecture is thus interpreted here as a topological threshold where the phase-space contraction required to store and compute a measurement exceeds the available geometric degrees of freedom in the local vacuum. At this limit, phase-space volume contracts faster than information can causally propagate outward. 

In this informational context, a black hole acts effectively as a physical firewall, dynamically limiting the localised divergence of information density. Epistemic incompleteness is then consistent with metric enforcement: perfectly measuring the universe requires a memory substrate and a localised fuel reservoir so dense that their mere initial presence collapses the very causal structure the agent attempts to observe.

\section{Emergence}
\label{sec:EmergentIntelligence}

Complex organisation in nature rarely arises from isolated units. From chemical reaction networks to neural circuits and ecosystems, intelligent behaviour is expressed through the collective dynamics of many interacting subsystems. Each component obeys relatively simple local laws, yet their interaction gives rise to coherent structure, predictive control, and adaptive behaviour that are not attributable to any component in isolation. While such phenomena are ubiquitous across biological systems, a precise characterization of emergence grounded in the physics of information processing and conserved quantities has remained elusive.

In this section we provide such a characterization. Building on the framework developed in the preceding sections, we extend the agent--environment formalism to composite systems composed of many coupled subsystems. Emergence is treated here not as a new primitive, but as a macroscopic physical property: how coupling alters the thermodynamic efficiency with which irreversible and reversible computation are deployed, and in some cases, how the admissible system boundary itself is dynamically redefined to maximise this efficiency. We begin by giving a quantitative definition of emergent intelligence and then identify distinct emergent regimes corresponding to different patterns of change in macroscopic goal-directed work, preserved CCE structure, and system boundary geometry.

\subsection{Multi-Component Agent--Environment Dynamics Model}
\label{sec:MultiAgentEnv}

Let $(X^{(i)},\mathscr{X}^{(i)})$ for $i=1,\dots,N$ denote measurable state spaces of $N$ internal subsystems, and let $(E,\mathcal E)$ denote the measurable state space of the external environment. The joint internal state is
\begin{align}
\label{eq:auto:0092}
X := X^{(1)}\!\times\!\cdots\!\times\!X^{(N)},\qquad
\mathscr{X} := \mathscr{X}^{(1)}\!\otimes\!\cdots\!\otimes\mathscr{X}^{(N)}.
\end{align}

Each subsystem exposes a measurable port
\begin{align}
\label{eq:auto:0093}
\Pi_{X,i}:X^{(i)}\to\Xi_{X},\qquad
\Pi_X(x)
    :=\big(\Pi_{X,1}(x^{(1)}),\dots,\Pi_{X,N}(x^{(N)})\big),
\end{align}
and the environment exposes a port $\Pi_E:E\to\Xi_{E}$. The pair $(\Pi_X(X_t),\Pi_E(E_t))$ forms the physical interface through which conserved-quantity exchange and irreversible information flow occur.

For each subsystem $i$, let $\mathcal M_i$ denote the indices of the environment channels relevant to $i$, and let $\mathcal N_i$ denote the indices of internal components to which $i$ is directly coupled. The one-step dynamics are specified by measurable kernels
\begin{align}
\label{eq:auto:0094}
X_{t+1}^{(i)}
    &\sim
     K_{X^{(i)}}\!\left(
        \cdot\,\middle|\,
        \Pi_E^{\mathcal M_i}(E_t),\,
        \Pi_X^{\mathcal N_i}(X_t)
     \right), \\
E_{t+1}
    &\sim
     K_E\!\left(
        \cdot\,\middle|\,
        E_t,\Pi_X(X_t)
     \right),
\end{align}
with conditional independences ensuring that all coupling occurs strictly through the designated ports.

The joint kernel factorises as
\begin{align}
\label{eq:auto:0095}
K(\mathrm dx',\mathrm de'\mid x,e)
    =
      \Big[
         \prod_{i=1}^N
         K_{X^{(i)}}(\mathrm dx^{(i)\prime}
             \mid \Pi_E^{\mathcal M_i}(e),
                  \Pi_X^{\mathcal N_i}(x))
      \Big]
      K_E(\mathrm de'\mid e,\Pi_X(x)),
\end{align}
with measure recursion
\begin{align}
\label{eq:auto:0096}
\mu_{t+1}(\mathrm dx',\mathrm de')
    =\!\int K(\mathrm dx',\mathrm de'\mid x,e)\,
           \mu_t(\mathrm dx,\mathrm de).
\end{align}
The interface law is $\bar\mu_t:=(\Pi_X,\Pi_E)_\#\mu_t$.

\subsubsection{The Separable Baseline via Zero-Flux Boundary Conditions}
To rigorously quantify emergence, we must define a separable baseline that isolates the subsystems from one another without altering their individual relationships with the external environment. Analogous to the passive boundary condition used to isolate causal footprint in \S\ref{sec:useful_work_causal_footprint}, we define the separable ensemble by imposing a strict zero-flux boundary condition across all inter-component ports ($\mathcal{N}_i$).

Physically, this severs the exchange of conserved quantities and irreversible information between subsystems. Within the metriplectic formulation, this corresponds to setting all cross-system components of the Poisson and dissipative tensors to zero ($J_{ij} = 0, R_{ij} = 0$ for $i \neq j$). 

The coupled joint kernel is therefore replaced by the completely separable kernel $K^{\mathrm{sep}}$:
\begin{align}
\label{eq:auto:0097}
K^{\mathrm{sep}}(\mathrm dx',\mathrm de'\mid x,e)
    =
      \Big[
         \prod_{i=1}^N
         K_{X^{(i)}}(\mathrm dx^{(i)\prime}
             \mid \Pi_E^{\mathcal M_i}(e))
      \Big]
      K_E(\mathrm de'\mid e,\Pi_X(x)),
\end{align}
where $K_{X^{(i)}}$ now operates under the strictly enforced boundary condition that no active flux enters from any internal neighbor $j \in \mathcal{N}_i$. The resulting measure $\mu_t^{\mathrm{sep}}$ defines the uncoupled reference trajectory. Any deviation from this baseline—manifesting as an increase in efficiency ($\Delta\chi_T$) or preserved CCE structure ($\Delta\kappa_T$)—is therefore strictly attributable to the macroscopic physical coupling of the subsystems.

\subsection{Quantitative Definition of Emergent Intelligence and Consciousness}
\label{sec:EmergentMetrics}

Within the composite setting above, causal intelligence and consciousness are evaluated as defined in \S\ref{sec:Intelligence} and \S\ref{sec:ConsciousnessMetric} using the joint law of the coupled ensemble. To isolate the physical contribution of coupling, we compare the coupled system to the corresponding separable baseline.

Crucially, because intelligence is defined as a ratio of integrated fluxes rather than the time integral of an instantaneous ratio, emergent contributions over a finite horizon $[0, T]$ must be computed directly from the difference of the cumulative causal efficiencies:
\begin{align}
\label{eq:auto:0098}
\Delta\chi_T := \chi_T - \chi_T^{\mathrm{sep}} = \frac{W_{\mathrm{causal},T}}{I_{\mathrm{irr},T}} - \frac{W_{\mathrm{causal},T}^{\mathrm{sep}}}{I_{\mathrm{irr},T}^{\mathrm{sep}}}
\end{align}

\begin{align}
\label{eq:auto:0099}
\Delta\kappa_T := \kappa_T - \kappa_T^{\mathrm{sep}} = \frac{W_{\mathrm{causal},T}}{I_{\mathrm{rev},T}} - \frac{W_{\mathrm{causal},T}^{\mathrm{sep}}}{I_{\mathrm{rev},T}^{\mathrm{sep}}}
\end{align}
Thus, $\Delta\chi_T > 0$ indicates net emergent causal intelligence over $[0, T]$, and $\Delta\kappa_T > 0$ indicates net emergent causal consciousness over $[0, T]$.

\paragraph{Goal-Directed Emergence}
When evaluating emergence with respect to a specific task, we apply the goal-directed projection established previously to yield the functional emergent metrics:
\begin{align}
\label{eq:auto:0100}
\Delta\chi_{\mathrm{goal},T} := \chi_{\mathrm{goal},T} - \chi_{\mathrm{goal},T}^{\mathrm{sep}} = \frac{W_{\mathrm{goal},T}}{I_{\mathrm{irr},T}} - \frac{W_{\mathrm{goal},T}^{\mathrm{sep}}}{I_{\mathrm{irr},T}^{\mathrm{sep}}}
\end{align}

\begin{align}
\label{eq:auto:0101}
\Delta\kappa_{\mathrm{goal},T} := \kappa_{\mathrm{goal},T} - \kappa_{\mathrm{goal},T}^{\mathrm{sep}} = \frac{W_{\mathrm{goal},T}}{I_{\mathrm{rev},T}} - \frac{W_{\mathrm{goal},T}^{\mathrm{sep}}}{I_{\mathrm{rev},T}^{\mathrm{sep}}}
\end{align}
Throughout this section, all quantities are understood to be evaluated relative to the same admissible system boundary and operational coarse-graining, so that changes reflect genuine efficiency gains from coupling rather than a change of description.

For convenience, we also introduce the dimensionless horizon-$T$ indices:
\begin{align}
\label{eq:auto:0102}
\mathcal{E}_T := \frac{\Delta\chi_T}{\chi_T^{\mathrm{sep}}}, \quad \mathcal{C}_T := \frac{\Delta\kappa_T}{\kappa_T^{\mathrm{sep}}}.
\end{align}
Their signs distinguish the types of emergence:
\begin{align}
\label{eq:auto:0103}
\mathcal{E}_T > 0 &\iff \text{emergent intelligence}, \\
\mathcal{C}_T > 0 &\iff \text{emergent consciousness}.
\end{align}
These quantities measure the relative strength of the emergent effects and allow comparison across physically heterogeneous systems and scales.

\subsection{Characteristic Regimes of Emergent Intelligence}

Emergent phenomena may be classified according to which of the quantities $\chi$, $\kappa$, and the boundary $\Pi_X(t)$ change under coupling, and on what timescale. Three characteristic regimes are distinguished.

For the regime classification below, we compare the coupled system to its separable baseline by evaluating the sign and magnitude of the cumulative emergent efficiencies $\Delta\chi_T$ and $\Delta\kappa_T$ over the horizon $[0, T]$, alongside the dynamical evolution of the admissible boundary $\Pi_X(t)$.

\subsubsection{Fast emergent coordination}
In this regime, the admissible system boundary $\Pi_X$ remains fixed and the amount of preserved internal structure does not increase significantly over the horizon considered. Emergence manifests as a rapid increase in intelligence due to more effective real-time routing of irreversible information processing:
\begin{align}
\label{eq:auto:0104}
\Delta\chi_T>0\,\qquad
\Delta\kappa_T\approx0,\qquad
\Pi_X(t)=\Pi_{X,0}.
\end{align}

Such events occur on short timescales and correspond to sudden coordination, synchronisation, or acute adaptive responses. In neural systems, this includes reflexive learning and extreme-stimulus avoidance, where intelligence increases transiently by pooling irreversible control interventions without immediate consolidation of long-term structure.

\subsubsection{Slow structural consolidation}
Here the system boundary remains fixed, but coupling leads to a gradual increase in preserved internal structure. The coupled subsystems collectively establish metastable CCE basins that would be thermodynamically inaccessible to individual components. Over long horizons, this reduces redundant irreversible computation and thereby increases intelligence:
\begin{align}
\label{eq:auto:0105}
\Delta\chi_T>0\,\qquad
\Delta\kappa_T>0,\qquad
\Pi_X(t)=\Pi_{X,0}.
\end{align}
This regime corresponds to memory formation, skill acquisition, and long-term learning, in which the generalised Landauer cost is invested once to construct stable structure that supports cheaper goal-directed behaviour thereafter.

\subsubsection{Boundary-changing emergence}
In this regime, coupling triggers the boundary optimization process described in \S\ref{sec:SystemBoundaries}, leading to a dynamic redefinition of the admissible system boundary itself:
\begin{align}
\label{eq:auto:0106}
\Pi_X(t):\ \Pi_{X,0} \;\longrightarrow\; \Pi_X'(t),
\end{align}
giving rise to a new, integrated macroscopic system at a higher level of organisation. Over a horizon $[0,T]$, this corresponds to a time-dependent admissible boundary $\Pi_X(t)$. The integrated emergent contributions are therefore
\begin{align}
\label{eq:auto:0107}
\Delta\chi_T>0,\qquad
\Delta\kappa_T>0,
\end{align}
and must be interpreted relative to this evolving description.

Such events include incorporative processes such as endosymbiosis, as well as recombinative processes such as sexual reproduction, in which information from multiple systems is combined to instantiate a new dynamically closed system. In these cases, emergence is not merely an efficiency gain within a fixed system, but a physical transition in system individuation—a shift in the geometric boundary of the agent itself to maximise intelligence.

Together, these regimes demonstrate that emergence within the present framework is neither monolithic nor mysterious. It arises through well-defined changes in efficiency, preserved CCE structure, and optimal system boundaries, governed strictly by the same physical principles that underlie single-system intelligence.

\section{Intelligence and Generalised Dissipation: The Physical Drive for Cognition}
\label{sec:IntelligenceEntropy}

Traditional discussions treat intelligence and physical decay as opposing forces: intelligence as organisation and purpose, dissipation as disorder and breakdown. Within the physical framework developed here, these concepts instead form a unified structure. Intelligence measures the efficiency with which a system converts irreversible information processing into goal-directed work, while generalised dissipation quantifies the strict physical cost of that computation under conservation laws. 

To understand why intelligent systems emerge and evolve in the physical universe, we must look beyond classical thermodynamics to the generalised dynamics of open, far-from-equilibrium systems constrained by the limits of physical encoding. Related connections between information processing, entropy production, and physical efficiency have been explored in nonequilibrium thermodynamics and the thermodynamics of computation \cite{landauer1961irreversibility,bennett1982thermodynamics,parrondo2015thermodynamics}.

\subsection{Intelligence as Structured Reversible Flow with Selective CCE Collapse}

The internal dynamics of an intelligent agent admit a reversible-dissipative decomposition of the form:
\begin{align}
\label{eq:auto:0108}
\dot x_t = J(x_t,t)\,\nabla H(x_t,t) - \mathcal{R}(x_t,t)\,\nabla \Xi(x_t,t)
\end{align}
where the reversible flow $J\nabla H$ preserves internal informational structure, while the dissipative flow $\mathcal{R}\nabla\Xi$ induces irreversible contraction associated with conserved-quantity dissipation. This decomposition makes explicit the physical distinction between information-preserving internal dynamics and the irreversible processes required to exert control and influence over the environment.The dissipative component generates a generalised dissipation rate within the specific conserved-quantity channel stabilising the system:
\begin{align}
\label{eq:auto:0109}
\dot S_{\mathrm{phys}}(t) = \langle \nabla\Xi(x_t,t),\, \mathcal{R}(x_t,t)\,\nabla\Xi(x_t,t)\rangle \ge 0
\end{align}
with a corresponding irreversible-information rate:
\begin{align}
\label{eq:auto:0110}
\dot I_{\mathrm{irr}}(t) = \dot S_{\mathrm{phys}}(t) / \alpha
\end{align}
where $\alpha$ denotes the fluctuation scale associated with the relevant conserved quantity (e.g., $k_B$ for thermal baths, $\hbar$ for quantum mechanical channels). This irreversible-information flow represents the strict physical exhaust of interventions that break time-reversal symmetry to collapse a Conservation-Congruent Encoding (CCE).The Causal Power DecompositionMacroscopic causal work arises from the interaction between the agent's internal dynamics and the environment via the boundary fluxes. Let $\mathcal{F}_{\mathrm{causal}} := \nabla_x {W}_{\mathrm{causal}}$ denote the effective causal gradient, representing the sensitivity of the instantaneous macroscopic causal power to changes in the agent's internal physical state. The active contribution of the agent's dynamics to the causal work rate decomposes linearly over the internal metriplectic flow:
\begin{align}
\label{eq:auto:0111}
\dot{W}_{\mathrm{causal}}(t) = \langle \mathcal{F}_{\mathrm{causal}},\, J(x_t,t)\nabla H(x_t,t) \rangle - \langle \mathcal{F}_{\mathrm{causal}},\, \mathcal{R}(x_t,t)\nabla \Xi(x_t,t) \rangle
\end{align}
The first term, $\langle \mathcal{F}_{\mathrm{causal}},\, J\nabla H \rangle$, represents the causal power enabled by reversible, structure-preserving dynamics. This flow allows the formation and maintenance of internal CCEs—such as memory, coordination, and predictive organization—that route conserved quantities and support sustained control without producing generalised dissipation.Conversely, the second term, $-\langle \mathcal{F}_{\mathrm{causal}},\, \mathcal{R}\nabla \Xi \rangle$, represents the causal power unlocked via irreversible interventions. While physically costly, this dissipative flow provides the symmetry-breaking steps required to collapse encodings, select actions, and unlock external gradients along causal degrees of freedom.Defining the instantaneous causal intelligence rate as:
\begin{align}
\label{eq:auto:0112}
\chi(t) := \frac{\dot W_{\mathrm{causal}}(t)}{\dot I_{\mathrm{irr}}(t)}
\end{align}
we obtain a structural interpretation of baseline intelligence. High intelligence corresponds to regimes in which the irreversible collapse of CCEs is used sparingly and selectively to enable large amounts of causal work mediated by the environment.The Goal-Directed Power DecompositionWhen evaluating task-specific efficiency, we project this dynamic influence onto the goal-aligned subspace. Let $\mathcal{F}_{\mathrm{goal}} := \nabla_x {W}_{\mathrm{goal}}$ denote the effective task gradient. The active contribution to the goal-directed work rate follows the identical geometric decomposition:
\begin{align}
\label{eq:auto:0113}
\dot{W}_{\mathrm{goal}}(t) = \langle \mathcal{F}_{\mathrm{goal}},\, J(x_t,t)\nabla H(x_t,t) \rangle - \langle \mathcal{F}_{\mathrm{goal}},\, \mathcal{R}(x_t,t)\nabla \Xi(x_t,t) \rangle
\end{align}
yielding the instantaneous goal-directed intelligence rate:
\begin{align}
\label{eq:auto:0114}
\chi_{\mathrm{goal}}(t) := \frac{\dot W_{\mathrm{goal}}(t)}{\dot I_{\mathrm{irr}}(t)}
\end{align}
Under both causal and goal-directed evaluations, intelligence increases not by eliminating dissipation, but by concentrating the generalised Landauer cost into minimal interventions that exert maximal leverage over external structural gradients.

\subsection{Intelligence as Constrained Throughput Maximization}

We now advance a conjecture that situates the evolution of intelligent organisation within a broader physical principle governing open systems across arbitrary physical substrates.

In the linear, near-equilibrium regime, local subsystems relax by minimising their rate of dissipation. However, intelligent systems operate far from equilibrium, driven by persistent, steep macroscopic gradients of conserved quantities. In these strongly driven, non-linear regimes, dynamical structures that increase the achievable rate of macroscopic gradient degradation tend to be heavily favored and selected for over time. 

When dissipation pathways are unconstrained, this global drive toward gradient degradation is typically realised through unstructured or turbulent dynamics. However, strict constraints often limit access to these gradients—such as a scarcity of physical substrates, limited actuation bandwidth, or the fundamental energetic bound ($\mathcal{F}\alpha \ln 2$) required to form a CCE. Under these bottlenecks, brute-force gradient degradation is physically impossible, heavily favouring the formation of highly structured, internal predictive control. 

Under such constraints, intelligence emerges not as a biological anomaly, but as a highly optimised, substrate-neutral set of dynamics. It allows a system to selectively unlock and route environmental conserved quantities through complex, otherwise dynamically inaccessible pathways. By doing so, the intelligent agent increases the macroscopic rate of gradient degradation in the environment, while dynamically minimising its own internal irreversible cost ($\dot I_{\mathrm{irr}}$). Intelligent systems do not suppress physical dissipation; they organise and accelerate it.

\subsection{Generalised Dissipation and the Arrow of Intelligence}
\label{sec:min-diss-general}

The preceding conjecture supports viewing intelligence as a process of highly constrained physical throughput. From this macroscopic perspective, goal-directed work is not an extraneous objective superimposed on physical dynamics, but the physical mechanism by which generalised dissipation is organised. An intelligent system routes the flow of conserved quantities through structured pathways, thereby maximising the achievable rate of macroscopic export subject to the strict physical limitations of its boundary conditions. 

This view situates intelligent behaviour within a broader tendency observed in nonequilibrium physics, wherein open systems evolve toward dynamical configurations that increase throughput subject to constraints \cite{dewar2003information,seifert2012stochastic}. What distinguishes intelligent agents is that by developing internal predictive CCE structure and reversible control, they unlock complex pathways that simple, unstructured systems cannot reach.

An instructive analogy comes from nonequilibrium fluid dynamics. When a fluid is driven by a steep thermal gradient, it spontaneously transitions from unstructured heat conduction into highly organised, circulating geometric patterns known as Rayleigh-Bénard convection cells. The fluid does not choose to self-organise; rather, these coherent macroscopic structures emerge because they represent the most effective dynamical geometry for maximising heat transport (gradient degradation) under the specific boundary constraints of the system. 

In the exact same sense, intelligent behaviour arises as the natural macroscopic resolution of the reversible-dissipative constraints governing an open system interacting with a complex environment. Given persistent macroscopic gradients and strict physical limits on how CCEs can be instantiated, the formation of internal informational structure and goal-directed control is inevitable: it is the most efficient dynamics for the system to maximise global gradient degradation under the boundary conditions it faces.

Because the irreversible collapse of a CCE defines a direction of influence that cannot be undone \cite{landauer1961irreversibility,seifert2012stochastic}, intelligence is inherently time-asymmetric. The emergence and evolutionary progression of intelligent agents reflects the exact same underlying physical drive that governs all macroscopic processes: the relentless tendency of constrained, far-from-equilibrium systems to realise ever more effective pathways for generalised dissipation.

\section{Consciousness and the Epistemic Limits of Intelligent Systems}
\label{sec:SelfModelEpistemicLimits}

In the preceding sections, intelligence was defined as the amount of goal-directed work produced per nat of irreversibly processed information (\S\ref{sec:Intelligence}), and consciousness as the amount of goal-directed work supported per nat of preserved internal information (\S\ref{sec:ConsciousnessMetric}). These two quantities capture complementary aspects of physical information processing: intelligence measures the efficiency of dissipative, distinction-destroying computation (instantiating new CCEs), while consciousness measures the efficiency with which persistent, dynamically isolated CCE basins support behaviour across time.

Sustaining high intelligence over long horizons requires more than rapid information processing. In structured environments, repeatedly paying the generalised Landauer cost ($\mathcal{F}\alpha \ln 2$) to reconstruct the same task-relevant distinctions is thermodynamically inefficient. Systems that instead preserve behaviourally relevant internal structure can reuse those metastable distinctions, strictly reducing cumulative irreversibility ($\dot{I}_{\mathrm{irr}}$) and improving long-horizon intelligence. 

In this sense, persistent internal information functions as a form of \emph{self-modelling}: it encodes geometric regularities of the system’s own dynamics and its coupling to the environment that continue to constrain future behaviour. The consciousness metric $\kappa_T$ quantifies how effectively this preserved CCE structure supports macroscopic work.

In this section, we examine the strict dynamical consequences of this physical view. We first formalise self-modelling as the preservation of internal CCE basins under the admissible boundary. We then establish intrinsic epistemic limits for such systems: because physically preserved information is necessarily a coarse-grained geometric representation of the true internal microstate, no physically realizable agent can form a complete or perfectly predictive model of itself.

\subsection{Emergence of Consciousness in Intelligent Systems}
\label{sec:EmergenceConsciousness}

Consider an agent operating within an environment that exhibits temporal or structural regularities. If the agent repeatedly collapses its internal encodings to react to these regularities, it incurs continuous dissipative costs through the irreversible metric tensor $\mathcal{R}\nabla \Xi$. Over extended horizons, such redundant processing lowers intelligence by inflating the total irreversible physical cost required to achieve a given level of goal-directed work.

An alternative strategy is to invest resources into forming persistent internal CCE basins that capture reusable regularities. Once established, these metastable encodings are transported forward in time via the structure-preserving reversible flow $J\nabla H$. This geometry constrains future behaviour without requiring repeated irreversible intervention. The same goal-directed work can then be achieved with drastically less cumulative dissipation, maximising the intelligence metric $\chi_T$.

This transition from repeated dissipative recomputation to structured reversible reuse marks the \emph{emergence of consciousness} in the present framework. Consciousness does not denote a separate immaterial substance, but the persistent physical instantiation of predictive CCEs that support long-horizon efficiency. Systems with higher $\kappa_T$ rely more heavily on the reversible propagation of internal structure, while systems with low $\kappa_T$ depend primarily on continual, costly irreversible updating.

Importantly, within this framework, the emergence of consciousness becomes increasingly necessary for systems seeking to maximise intelligence over long horizons. As interaction horizons lengthen, exclusive reliance on irreversible computation violates constraints on generalised dissipation. Persistent internal structure—memories, latent models, and preserved spatial geometry—therefore becomes progressively more important.

\subsection{Epistemic Limits of Conscious Systems}
\label{sec:EpistemicLimits}

Consciousness measures the contribution of preserved internal information to goal-directed work. By definition within our framework, this preserved information is physically maintained as a set of metastable macroscopic basins $Z$, mapped from the underlying microscopic physical state space $X$ via a coarse-graining projection $\pi_\Phi: X \to Z$. 

Because preserved information must be physically realised, the cardinality of $Z$ is strictly finite, bounded by the system's available structural free energy. Consequently, the mapping from the continuous, high-dimensional microstate space $X$ to the discrete CCE encoding space $Z$ represents a massive volumetric compression.

\begin{proposition}[State-level epistemic incompleteness]
\label{prop:StateIncompleteness}
For any physically realizable system with finite preserved information capacity, there exist distinct internal microstates that are strictly indistinguishable from the system’s own preserved macroscopic perspective.
\end{proposition}

\begin{proof}
Let $Z$ denote the finite set of metastable CCE encodings, and $X$ the underlying microscopic physical state space. For any encoding $z \in Z$, the corresponding physical realization is a metastable basin $B_z = \pi_\Phi^{-1}(z) \subset X$. Because the capacity of $Z$ is physically finite, the phase-space volume of $B_z$ must be strictly positive (or contain a vast multiplicity of quantum microstates). Therefore, by definition, there exist microstates $x_1, x_2 \in B_z$ such that $x_1 \neq x_2$, but $\pi_\Phi(x_1) = \pi_\Phi(x_2) = z$. Because the system's self-model operates strictly over the preserved encodings $Z$, it is geometrically impossible for the system to distinguish whether its true physical state is $x_1$ or $x_2$.
\end{proof}

This fundamental geometric limitation extends to the system’s predictions about its own future evolution. Because the self-model encodes only the coarse-grained, macroscopic basin index $z$, it cannot uniquely compute the exact trajectory generated by its internal continuous dynamics.

\begin{proposition}[Dynamical epistemic incompleteness]
\label{prop:DynamicalIncompleteness}
For any system with finite preserved informational capacity evolving under non-trivial metriplectic dynamics, there exist differences in future evolution that cannot be deterministically inferred from the system’s preserved self-model.
\end{proposition}

\begin{proof}
Let $x_1, x_2 \in B_z$ be two distinct microstates that map to the exact same current preserved encoding $z = \pi_\Phi(x_1) = \pi_\Phi(x_2)$ at time $t_0$. The system's true internal state evolves continuously under the exact metriplectic flow $\dot{x} = J(x)\nabla H(x) - \mathcal{R}(x)\nabla\Xi(x)$. Due to the divergence of trajectories within the continuous space $X$ (e.g., positive Lyapunov exponents in the reversible flow, or stochasticity in the dissipative flow), there exists a future time $T > t_0$ such that the forward evolutions $x_1(T)$ and $x_2(T)$ cross different basin separatrixes, landing in distinct future CCE basins $z_1' \neq z_2'$. Because the system's self-model at $t_0$ only has access to the macroscopic encoding $z$, and not the specific initial microstate ($x_1$ or $x_2$), it lacks the informational capacity to perfectly predict its own macroscopic future.
\end{proof}

These results establish a rigorously bounded, physical form of epistemic incompleteness. The limitations arise not from logical paradoxes (e.g., Gödelian self-reference), but directly from the necessity of coarse-graining. To form a CCE, a system must discard microscopic degrees of freedom. 

Thus, the exact same physical constraints that make consciousness and self-modelling possible—by dynamically isolating macroscopic equivalence classes—simultaneously impose absolute physical limits on self-knowledge. Intelligence and consciousness can increase indefinitely in principle, but at every physical stage, the system’s self-model remains a strict, dynamically incomplete projection of a vastly richer underlying microscopic reality.

\section{Intelligence and the Dynamics of the Brain}
\label{sec:BrainDynamics}

The brain has long been considered a canonical example of an intelligent physical system, motivating centuries of study into how its structure and activity give rise to adaptive behaviour. Within the present framework, the brain is especially valuable because its dynamics provide an empirical testbed for evaluating the principles developed in the preceding sections.

Two empirical observations are particularly relevant. First, neural oscillations are ubiquitous: coherent rhythms appear across frequency bands in EEG, LFP, MEG, and intracellular recordings, coordinating population activity across spatial and temporal scales \cite{buzsaki2006rhythms,fries2005mechanism,varela2001brainweb}. In our framework, oscillatory dynamics are not mere epiphenomena; they are the biological instantiation of structure-preserving reversible flow. By routing activity along structured manifolds of the state space, they transport Conservation-Congruent Encodings (CCEs) forward in time without repeatedly paying the generalised Landauer cost ($\mathcal{F}\alpha \ln 2$), thereby radically lowering the irreversible-information rate ($\dot{I}_{\mathrm{irr}}$) required to sustain goal-directed work.

Second, neural dynamics exhibit signatures of near-criticality: power-law statistics, neuronal avalanches, and long-range temporal correlations suggest that cortical systems operate near the edge of a dynamical critical regime \cite{beggs2003neuronal,shew2013criticality}. Such regimes maximise sensitivity, dynamic range, and information propagation while maintaining coherence. In our framework, near-criticality represents an optimization of the CCE basin geometry: it minimises the dissipative barriers between metastable states, allowing the system to rapidly collapse encodings and acquire new predictive information with minimal irreversible physical exhaust.

In this section, we argue that oscillatory and near-critical dynamics jointly position the brain in an optimal operating regime under the proposed CCE framework, illustrating how biological neural systems physically resolve the drive to maximise macroscopic throughput under strict thermodynamic constraints.

\subsection{Oscillatory Dynamics: Segregating Dissipation via Limit Cycles}

A central implication of the framework introduced in \S\ref{sec:AgentEnvironmentModel} is that intelligent systems benefit from internal dynamics that maximise the proportion of computation carried by reversible transformations. In physical systems operating far from equilibrium, such reversibility cannot arise from purely conservative (Hamiltonian) dynamics. Purely conservative systems produce structurally unstable ``centers'' (like a frictionless pendulum) where any microscopic noise permanently alters the amplitude, destroying the phase-locking and synchronization required for neural computation.

Robust biological coordination instead relies on limit cycles. The relevance of oscillations in our framework therefore follows not from a complete absence of dissipation, but from the strict geometric segregation of the reversible ($J \nabla H$) and dissipative ($\mathcal{R} \nabla \Xi$) flows. 

To maintain a stable oscillatory manifold against thermal and synaptic noise, the system must continuously exert a dissipative flow ($\mathcal{R} \nabla \Xi$) orthogonal to the cycle, contracting aberrant amplitude perturbations back onto the attractor. This requires a continuous metabolic exhaust. However, once the state is dynamically confined to this one-dimensional ring manifold, the longitudinal transport along the cycle (the phase advance) is governed almost entirely by the reversible, structure-preserving flow ($J \nabla H$). 

By geometrically confining dissipation to the orthogonal amplitude dimensions, limit cycles leave the phase dimension as a highly protected, low-dissipation channel. This allows oscillatory coordination to transport Conservation-Congruent Encodings (CCEs) encoded in phase forward in time without repeatedly paying the generalised Landauer cost ($\mathcal{F} \alpha \ln 2$) for longitudinal information routing.

\subsubsection{Metriplectic Decomposition and Minimal-Dissipation Argument}

We now establish that oscillatory dynamics reduce generalised dissipation for a fixed level of goal-directed work, thereby strictly increasing the intelligence metric ($\chi$). 

Let $x_t \in X$ denote the system's internal state, evolving under the metriplectic decomposition
\begin{align}
\label{eq:auto:0115}
\dot{x}_t = J(x_t, t)\nabla H(x_t, t) - \mathcal{R}(x_t, t)\nabla\Xi(x_t, t),
\end{align}
where $J^\top = -J$ generates reversible motion and $\mathcal{R}^\top = \mathcal{R} \ge 0$ generates irreversible relaxation. As established in \S\ref{sec:IntelligenceEntropy}, the generalised dissipation rate is
\begin{align}
\label{eq:auto:0116}
\dot{S}_{\mathrm{phys}}(t) = \langle \nabla\Xi(x_t, t), \mathcal{R}(x_t, t)\nabla\Xi(x_t, t) \rangle
\end{align}
with the corresponding irreversible-information rate $\dot{I}_{\mathrm{irr}}(t) = \dot{S}_{\mathrm{phys}}(t) / \alpha$. 

\begin{proposition}[Minimal irreversible information among robust metriplectic dynamics]
Among all structurally stable metriplectic dynamics achieving a prescribed long-run mean rate of goal-directed work $\dot{W}_{goal} > 0$, the long-run irreversible information rate is minimised when dissipative computation is strictly confined to (i) maintaining the low-dimensional task manifolds (transverse amplitude contraction) and (ii) executing task-critical CCE collapses. Reversible phase dynamics ($J \nabla H$) are used to preserve and propagate internal structure along the manifolds.
\end{proposition}

\begin{proof}
Time-averaging the power decomposition shows that generalised dissipation depends strictly on the irreversible contribution $\mathcal{R} \nabla \Xi$. For structurally stable memory transport in noisy environments, $\mathcal{R} \nabla \Xi$ must be non-zero to contract transverse noise. Minimising the long-run dissipation rate subject to robustness and useful-work constraints therefore favours geometric flows (limit cycles) that orthogonally decouple the stabilising dissipative flow from the information-carrying reversible flow. Any increase in longitudinal dissipative flow beyond the minimum required for discrete action selection strictly inflates $\dot{I}_{\mathrm{irr}}$ and lowers intelligence. 
\end{proof}

\subsubsection{Precision-Dissipation Tradeoff and the TUR Analogy}

This minimum-dissipation principle has a natural analogue in the stochastic setting. When noise is present, physical performance must be evaluated with respect to fluctuations. The Thermodynamic Uncertainty Relation (TUR) \cite{barato2015thermodynamic,gingrich2016dissipation} captures this tradeoff, stating that the precision of any sustained current is fundamentally limited by total entropy production.

For a cumulative current $J_T$ measured over a finite horizon $T$, the finite-time TUR takes the form
\begin{align}
\label{eq:auto:0117}
\frac{\mathrm{Var}(J_T)}{\mathbb{E}[J_T]^2} \ge \frac{2\alpha}{\Sigma_T},
\end{align}
where $\Sigma_T$ is the total cumulative generalised dissipation over the horizon:
\begin{align}
\label{eq:auto:0118}
\Sigma_T = \int_0^T \dot{S}_{\mathrm{phys}}(t) dt = \alpha \int_0^T \dot{I}_{\mathrm{irr}}(t) dt = \alpha I_{\mathrm{irr},T}.
\end{align}
The inequality expresses a strict physical bound: achieving lower relative variance requires greater cumulative dissipation.

\begin{corollary}[Oscillatory Dynamics Approach the TUR Limit]
Internal dynamics dominated by reversible flow fields $(\dot{x} \approx J\nabla H)$ are locally time-symmetric and therefore operate near equality in the TUR. For fixed mean current, increasing the reversible (oscillatory) fraction of the dynamics reduces the generalised dissipation required to maintain a given level of precision.
\end{corollary}

Oscillatory coordination suppresses time-asymmetry in microscopic trajectories, thereby minimising the physical cost required to maintain precise, goal-directed CCEs in noisy environments. 

\subsubsection{General Information-Fidelity Tradeoff with Reversible Predictive Memory}

Finally, reversible, oscillatory dynamics reduce the irreversible information rate required to sustain accurate internal representations. Let the system track an external process $y_t$ with an internal estimate $\hat{y}_t$ and distortion measure $D = \mathbb{E}[d(y_t, \hat{y}_t)]$. Define $\mathcal{R}(D)$ as the minimal physically required information rate to maintain fidelity $D$.

Within the metriplectic dynamics, the reversible flow $J\nabla H$ transports predictive CCEs through time without generalised dissipation. Let $\gamma_{\mathrm{rev}}(t)$ denote the instantaneous rate at which the reversible subspace preserves this predictive information, satisfying $0 \le \gamma_{\mathrm{rev}}(t) \le \mathcal{R}(D)$. Only the remaining portion must be supplied by dissipative, irreversible CCE updates. This yields the general bound:
\begin{align}
\label{eq:auto:0119}
\dot{I}_{\mathrm{irr}}(t) \ge \mathcal{R}(D) - \gamma_{\mathrm{rev}}(t).
\end{align}

\begin{corollary}[Reversible Predictive Memory Reduces Irreversible Cost]
Oscillatory dynamics, by recirculating predictive CCEs through phase-coherent cycles, satisfy a large portion of the fidelity requirement $\mathcal{R}(D)$ via $\gamma_{\mathrm{rev}}(t)$. This explicitly reduces the required rate of irreversible updates ($\dot{I}_{\mathrm{irr}}$), thereby increasing the instantaneous intelligence $\chi(t)$.
\end{corollary}

\subsection{Near-Critical Dynamics and CCE Susceptibility}
\label{sec:criticality_optimality}

Oscillatory coordination explains how biological systems transport CCEs efficiently. We now examine near-criticality, which governs how these systems form new encodings. Neural recordings exhibit signatures of criticality, suggesting the brain operates poised between order and disorder \cite{beggs2003neuronal,shew2013criticality}.

Within our framework, the relevance of criticality arises from the geometric tuning of CCE basins. To absorb new predictive information, the brain must collapse existing encodings and form new ones, incurring an irreversible cost. The efficiency of this process can be summarised by a susceptibility measure quantifying how much predictive information can be acquired per unit of generalised dissipation.

Let $\dot{I}_{\mathrm{acq}}(t)$ denote the rate at which new predictive information is incorporated, and $\dot{I}_{\mathrm{irr}}(t)$ the corresponding irreversible physical cost. Their ratio
\begin{align}
\label{eq:auto:0120}
\Lambda(t) := \frac{\dot{I}_{\mathrm{acq}}(t)}{\dot{I}_{\mathrm{irr}}(t)}
\end{align}
characterises how much usable information is gained per unit of generalised Landauer cost.

\begin{proposition}[Near-critical susceptibility enhances encoding efficiency]
For dynamics with finite $\dot{I}_{\mathrm{irr}}(t)$, the ratio $\Lambda(t)$ is locally maximised in regimes of high coarse-grained susceptibility. This corresponds to near-criticality, where the energetic barriers between CCE basins are minimised, allowing information-bearing fluctuations to induce macroscopic encoding updates without runaway dissipation.
\end{proposition}

\begin{proof}[Proof sketch]
Near a critical phase transition, the energetic barriers separating metastable CCE basins become shallow. This high susceptibility allows small environmental perturbations to rapidly drive the system into new encoding states, maximising $\dot{I}_{\mathrm{acq}}$. Because the barriers are shallow, the physical generalised dissipation ($\dot{I}_{\mathrm{irr}}$) required to force this transition is minimised. Therefore, $\Lambda(t)$ attains elevated values in regions of maximal responsiveness under stability constraints.
\end{proof}

In biological systems, these regimes heavily complement each other. Long-range oscillatory synchrony supports low-dissipation CCE propagation, while near-critical fluctuations supply heightened responsiveness to carve out new CCEs cheaply. Together, they place neural systems in an optimal operating regime.

\section{Building Circuits from Continuous Dynamical Systems}
\label{sec:physical-circuits}

Shannon’s analysis of relay circuits showed that Boolean logic can be
realised directly in the physical dynamics of electromechanical systems
\cite{shannon1948communication}. Here we extend this viewpoint to more
general continuous dynamical systems, developing a principled
correspondence between informational encodings and the evolution of
state under smooth dynamics. This provides a foundation for constructing
physical circuits whose computational behaviour arises from their
attractor structure.

Encodings throughout this section are interpreted under the
Conservation-Congruent Encoding (CCE) framework introduced in
\S\ref{sec:EncodingPhysicalState}, where each logical value
corresponds to a metastable basin of attraction in state space, separated
by control-tunable barriers associated with conserved quantities. These
attractors may be fixed points, limit cycles, or higher-dimensional
invariant sets, provided their basins remain distinct and robust under
the chosen control configuration.

We begin by revisiting a familiar example from classical electronics, the
CMOS inverter, whose logic function is already implemented by continuous
dissipative dynamics. This serves as a concrete illustration of the
general principle that digital logic is a special case of computation by
smooth dynamical systems and motivates the broader constructions that
follow.

\subsection{Classical Transistors as Fixed-Point Dynamical Systems}
\label{sec:transistor-fixedpoints}

Before developing general dynamical circuits, it is useful to note that
classical digital circuits already realise computation through continuous
physical dynamics. In CMOS logic, each gate relaxes to a stable attractor,
and Boolean semantics arise only after convergence.

Consider a CMOS inverter with input voltage $V_{\mathrm{in}}$ and output
node voltage $V(t)$. The PMOS and NMOS transistors contribute currents
$I_{P}(V_{\mathrm{in}},V)$ and $I_{N}(V_{\mathrm{in}},V)$, respectively.
Writing $C$ for the load capacitance, the output node evolves according to
\begin{align}
    C\,\frac{dV}{dt}
    = I_{P}(V_{\mathrm{in}},V) - I_{N}(V_{\mathrm{in}},V).
    \label{eq:auto:0121}
\end{align}
For fixed $V_{\mathrm{in}}$, stable fixed points satisfy
\begin{align}
\label{eq:auto:0122}
I_{P}(V_{\mathrm{in}},V^*) = I_{N}(V_{\mathrm{in}},V^*).
\end{align}

We interpret $V$ as a logical signal by selecting two disjoint intervals
$[0,V_{\mathrm{L}}]$ and $[V_{\mathrm{H}},V_{\mathrm{DD}}]$ corresponding
to logical $0$ and $1$, with the intermediate region avoided by design. The device
is engineered so that low inputs yield a stable fixed point in the high
interval and vice versa. Each logical input value thus selects a distinct
stable attractor of the output voltage.

From the CCE perspective, each logical value is realised as a metastable
basin of attraction in voltage space, and a switching event corresponds to
driving the system across basin boundaries. The CMOS inverter performs
logical inversion by selecting one of two input-dependent attractors, and
its Boolean semantics emerge from dissipative relaxation.

Memory elements arise by arranging inverters into feedback configurations
that admit multiple attractors. The simplest example is the cross-coupled
pair, whose node voltages $(V_A,V_B)$ evolve as
\begin{align}
    \label{eq:auto:0123}
\tau \dot{V}_A = F(V_B) - V_A,
    \qquad
    \tau \dot{V}_B = F(V_A) - V_B,
\end{align}
where $F$ is the static transfer characteristic of an inverter. For
sufficiently steep gain, the system admits two stable fixed points,
corresponding to encodings separated by an unstable saddle. 

These examples illustrate a broader principle: fixed-point digital logic
is a special case of computation by continuous dynamics. More general
attractors—limit cycles, quasiperiodic orbits, invariant tori, and
higher-dimensional structures—support computational behaviours without
digital analogues. This motivates our general framework, in which
computation is realised by smooth state evolution and logical structure
is induced by CCE-compliant attractor geometry.

\subsection{Dynamical Circuits}

A \emph{dynamical circuit} is a directed graph whose nodes are dynamical
systems and whose edges specify how their states are coupled.
This provides a compositional formalism for building complex continuous
computational systems, with logical semantics supplied by the CCE
framework of \S\ref{sec:EncodingPhysicalState}. The global behaviour of the
circuit is determined by the joint evolution of all node states under
smooth vector fields.

\subsubsection{Nodes}

Each node is a dynamical system with internal state $x_i\in X_i$ and drive
variable $v_i\in V_i$, evolving according to
\begin{align}
    \label{eq:auto:0124}
\dot{x}_i = f_i(x_i,v_i),
\end{align}
where $v_i$ collects all incoming signals. Different choices of $X_i$,
$V_i$, and $f_i$ produce different computational primitives. Three
illustrative examples follow.

\paragraph{Leaky integrator.}
A scalar state $x$ evolves as $\dot{x}=-\alpha x+v$, performing temporal
integration and smoothing.

\paragraph{Nonlinear activation unit.}
A scalar state $x$ follows $\dot{x}=-x+\sigma(v)$, producing bounded,
nonlinear responses analogous to neural activation functions.

\paragraph{Oscillator.}
A phase oscillator has $x\in S^1$ with intrinsic dynamics
\begin{align}
\label{eq:auto:0125}
\dot{x} = \omega + F(x) + v,
\end{align}
where $\omega>0$ is its frequency and $F$ a smooth periodic nonlinearity.
Input $v$ modulates phase velocity, enabling synchronisation and
phase-based computation.

These primitives demonstrate the diversity of behaviours dynamical nodes
can express. \S\ref{sec:BooleanLogic} shows how such nodes
can implement logic, and later sections examine computational structures
with no digital analogue.

\subsubsection{Node Coupling / Edges}

Nodes interact through directed edges. An edge $j\to i$ influences the
drive of node $i$ via an aggregation map
\begin{align}
\label{eq:auto:0126}
G_i :
\Bigl(\prod_{j:j\to i} X_j\Bigr)\times U_i^{\mathrm{ext}}
\to V_i.
\end{align}
The drive signal is
\begin{align}
\label{eq:auto:0127}
v_i = G_i\bigl((x_j)_{j:j\to i},\,u_i^{\mathrm{ext}}\bigr).
\end{align}
Substituting this into the local dynamics yields
\begin{align}
\label{eq:auto:0128}
\dot{x}_i
= f_i\!\left(x_i,
            G_i\!\bigl((x_j)_{j:j\to i},\,u_i^{\mathrm{ext}}\bigr)
        \right).
\end{align}

This construction defines a composite dynamical system whose global
behaviour emerges from the interconnection of its primitives.

\subsubsection{Ports, Interfaces, and Composition Constraints}
\label{sec:ports-interfaces}

While the aggregation map formalism specifies how nodes influence each
other, building reusable circuits requires a structured interface. Ports
provide this abstraction by specifying which internal variables are
visible, where inputs act, how logical information is extracted, and how
context modulates behaviour.

\begin{enumerate}[label=(\roman*)]
    \item \textbf{Input ports} identify the coordinates through which external signals influence a node's vector field.
    \item \textbf{Output ports} expose selected internal coordinates to downstream nodes.
    \item \textbf{Encoding ports} extract the information-bearing variable associated with the node’s attractor structure under CCE.
    \item \textbf{Context ports} expose slow-timescale parameters (gains, frequencies, thresholds) that reshape dynamics without acting as direct inputs.
\end{enumerate}

Two nodes may be interconnected only when their ports satisfy:
\begin{enumerate}[label=(\roman*)]
    \item type compatibility.
    \item scale compatibility.
    \item directional compatibility.
\end{enumerate}

Well-defined encoding under CCE additionally requires a quasi-static
timing assumption. Inputs applied to a node must remain fixed long enough
for the system to relax into the corresponding attractor before the
encoding is interrogated. Define $\tau_{\mathrm{settle}}(v)$ as the
settling time under input $v$ and let
$\tau_{\max} := \sup_v \tau_{\mathrm{settle}}(v)$. Inputs are assumed to
vary on timescales slow relative to $\tau_{\max}$, ensuring that attractor
selection remains unambiguous. Under these constraints, the composite system
\begin{align}
\label{eq:auto:0129}
\dot{x}_i = f_i\!\left(x_i,\,
                 G_i\!\bigl((x_j)_{j:j\to i},\,u_i^{\mathrm{ext}}\bigr)
           \right)
\end{align}
inherits a structured attractor space with clean logical semantics under
CCE. Each node exposes a local encoding space $L_i$, and the global circuit
evolves on a submanifold of $\prod_i L_i$. Circuit components correspond
to low-dimensional task-relevant subsets of this manifold (e.g.\ bistable
memory, winner–take–all modules). The port abstraction ensures that such
components retain coherent logical semantics even as their internal
dynamics remain continuous, nonlinear, and potentially high-dimensional.

\subsection{Reconstructing Boolean Logic}
\label{sec:BooleanLogic}
Under the Conservation-Congruent Encoding (CCE) assumptions of
\S\ref{sec:EncodingPhysicalState}, each encoding corresponds to a distinct
metastable attractor of a dynamical circuit. Inputs act through input
ports that modify local vector fields; encodings are extracted from the
attractor reached via encoding ports; and context ports determine
intrinsic parameters that shape the attractor landscape. In this
subsection we show that suitable interconnections of the dynamical
primitives introduced above reproduce the classical Boolean operations as
a special case.

The construction parallels Shannon’s derivation of Boolean gates from
relay circuits, but is now expressed entirely in terms of smooth vector
fields, port-level composition, and attractor structure. Throughout this
subsection the circuit is operated in a quasi-static regime: logical
inputs are held fixed long enough for the system to relax to its
associated attractor before the encoding ports are evaluated to yield a Boolean value. Let
$\tau_{\mathrm{settle}}(v)$ denote the settling time under input
configuration $v$ and let
$\tau_{\max} = \sup_v \tau_{\mathrm{settle}}(v)$. Logical inputs change on
timescales slower than~$\tau_{\max}$; if varied more rapidly, the circuit
would traverse transient states that lack a clean Boolean interpretation.

\subsubsection{NOT Gate}

A continuous inverter can be implemented using a single nonlinear
activation node. Let $x \in \mathbb{R}$ be the node state, $v$ the input
entering through the input port, and let $b,w$ be context parameters
shaping the intrinsic response. With a smooth, saturating nonlinearity
$\sigma$, define
\begin{align}
    \label{eq:auto:0130}
\dot{x} = -x + \sigma(-w\,v + b).
\end{align}
For appropriate $b$ and $w>0$, the system possesses two stable attractors
corresponding to encodings under CCE. A high input $v$ drives the
argument negative so that the trajectory relaxes to the low attractor;
conversely, a low input yields the high attractor. This realises logical
negation.

\subsubsection{AND Gate}

A continuous AND gate is obtained by feeding two logical inputs $v_1$ and
$v_2$ into the input ports of a single activation node. The aggregation
map
\begin{align}
\label{eq:auto:0131}
G(v_1,v_2) = w_1 v_1 + w_2 v_2,
\qquad w_1,w_2>0,
\end{align}
supplies the drive $v = G(v_1,v_2)$. With a context-controlled threshold
$\theta$, the dynamics
\begin{align}
    \label{eq:auto:0132}
\dot{x} = -x + \sigma(G(v_1,v_2) - \theta)
\end{align}
admit a high attractor only when both inputs are high so that
$G(v_1,v_2) > \theta$. The encoding port therefore implements the AND
operation under CCE.

\subsubsection{OR Gate}

Lowering the threshold yields an OR gate. If
$\theta_{\mathrm{OR}} < \min\{w_1,w_2\}$, then
\begin{align}
    \label{eq:auto:0133}
\dot{x} = -x + \sigma(G(v_1,v_2) - \theta_{\mathrm{OR}})
\end{align}
produces a high attractor whenever either input is high. Thus the encoding port yields the Boolean value $1$ exactly when
$v_1 = 1$ or $v_2 = 1$ under CCE.

\subsubsection{Bistable Memory: A Continuous Flip-Flop}

\begin{figure}[!ht]
    \centering
    \begin{tikzpicture}[
        >=Latex,
        node distance=3.0cm and 3.0cm,
        dynode/.style={circle,draw,thick,minimum size=12mm,inner sep=0, font=\small},
        inhibit/.style={-Latex,thick},
        excite/.style={-Latex,thick},
        every node/.style={font=\small}
    ]
        \node[dynode] (A) {$x_A$};
        \node[dynode,right=4.0cm of A] (B) {$x_B$};

        \draw[excite] (A) edge[loop left] node[left] {$g x_A$} (A);
        \draw[excite] (B) edge[loop right] node[right] {$g x_B$} (B);

        \draw[inhibit] (A) -- node[above] {$-h x_B$} (B);
        \draw[inhibit] (B) -- node[below] {$-h x_A$} (A);

        \node[above=1.2cm of A] (IA) {$I_A$};
        \node[above=1.2cm of B] (IB) {$I_B$};

        \draw[->] (IA) -- (A);
        \draw[->] (IB) -- (B);

    \end{tikzpicture}
    \caption{\small
    Bistable memory implemented by mutually inhibiting, self-exciting
    activation nodes. The circuit admits two stable attractors 
    $(x_A\text{ high},x_B\text{ low})$ and 
    $(x_A\text{ low},x_B\text{ high})$ under CCE.}
    \label{fig:flipflop_dyn}
\end{figure}
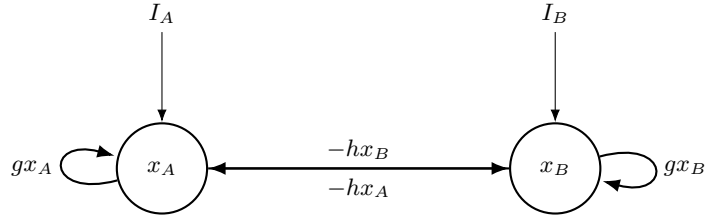

Boolean computation requires state. A continuous SR-like flip-flop is
formed by interconnecting two activation nodes via inhibitory and
self-excitatory couplings. Let $x_A,x_B$ be their exposed output-port
states, with inputs $I_A,I_B$ acting through input ports. With context
parameters $g>1$ and $h>0$ controlling self-excitation and mutual
inhibition,
\begin{align}
\label{eq:auto:0134}
\dot{x}_A = -x_A + \sigma(gx_A - h x_B + I_A), \qquad
\dot{x}_B = -x_B + \sigma(gx_B - h x_A + I_B),
\end{align}
the system admits two stable attractors corresponding to stored logical
states. Under CCE, the encoding ports associated with $x_A$ and $x_B$ extract the
stored bit values.
Inputs $I_A$ and $I_B$ select which attractor is reached, providing ``set''
and ``reset'' behaviour.

\subsubsection{Composite Gates: NAND, NOR, and XOR}

The gates above form a functionally complete basis under port composition.
NAND and NOR are obtained by feeding the output ports of AND or OR gates
into the input port of a NOT gate. XOR follows from the Boolean identity
\begin{align}
\label{eq:auto:0135}
\mathrm{XOR}(v_1,v_2)
 = \mathrm{AND}\!\left(
       \mathrm{NAND}(v_1,v_2),\,
       \mathrm{OR}(v_1,v_2)
   \right),
\end{align}
or alternatively from two competing activation nodes receiving opposing
drives. In each case, the attractor reached under fixed inputs corresponds
exactly to the Boolean output under CCE.

These constructions show that dynamical circuits can reproduce the full
Boolean algebra when encodings are interpreted as metastable attractors
under CCE. However, Boolean circuits constitute only a narrow subset of
the behaviours available to continuous dynamical systems: they operate
exclusively through fixed-point attractors with memoryless transitions
and no intrinsic temporal structure.

\subsection{Beyond Boolean Logic}

Computation is fundamentally physical: digital and analog systems alike evolve under continuous physical laws. Classical digital
architectures restrict these laws by engineering trajectories that rapidly
collapse onto a small set of fixed-point attractors, with clocked gating
providing transitions between equilibria. Under CCE, this corresponds to a
small encoding space realised by a limited number of metastable basins and
frequent irreversible transitions between them.

General dynamical systems, however, are not confined to fixed points.
Their computation arises from the geometry of invariant sets, attractors,
and basins of attraction. Under CCE, these geometric structures induce
encodings whose dimensionality, topology, and stability directly determine
both the expressivity of the computation and the irreversible information
cost required to realise it. Boolean logic therefore occupies a highly
restricted corner of the space of physically admissible computations.

\subsubsection{Dynamical Primitives Beyond Fixed-Point Logic}

Physical dynamical systems admit invariant structures that induce qualitatively different encoding geometries under CCE. By altering the balance between reversible ($J\nabla H$) and dissipative ($\mathcal{R}\nabla \Xi$) flows, these distinct topologies lead to different computational regimes and irreversible-cost profiles.

\paragraph{Multistability (dissipation-dominated).}
Systems with multiple stable equilibria induce discrete encoding sets under CCE, with computation proceeding through basin selection. Each input drives the system toward one of finitely many metastable attractors, and encoding selection occurs via the irreversible, contractive flow of $\mathcal{R}\nabla \Xi$ into a chosen basin.

This regime closely resembles classical Boolean decision logic, but arises here purely as a consequence of attractor geometry rather than symbolic manipulation. Because every transition requires full relaxation, it is thermodynamically costly.

\paragraph{Limit cycles and oscillatory attractors.}
Oscillatory systems encode information in phase, frequency, or amplitude through stable periodic orbits. Under CCE, phase-coherent evolution enables the reversible transport of predictive structure along limit cycles via $J\nabla H$, allowing information to be preserved across time without repeated irreversible updates.

For tasks with periodic or cyclic structure, this geometry radically reduces the number of irreversible encoding transitions ($\dot{I}_{\mathrm{irr}}$) required to sustain accurate computation.

\paragraph{Quasiperiodic and toroidal dynamics.}
Weakly coupled oscillators generate invariant tori, inducing continuous encoding manifolds of higher dimensionality. Such geometries support smooth interpolation between encodings and permit gradual deformation of internal representations without discrete collapse.

Relative to fixed-point logic, toroidal dynamics reduce irreversible merges and enable richer continuous computation under CCE, effectively computing over continuous variables without the need for discrete digitization.

\paragraph{Sequential metastability and heteroclinic channels.}
Trajectories through ordered sequences of saddle neighbourhoods realise temporally structured computation without a global clock. Under CCE, encodings unfold along directed paths in state space, with each metastable region temporarily capturing the state before instability drives it to the next.

Computation is therefore carried by the deterministic geometric ordering of transitions rather than by instantaneous state, enabling sequential processing and finite-state machine equivalents with minimal irreversible intervention.

\paragraph{Hamiltonian and momentum-driven flows (reversible-dominated).}
Systems governed by conserved quantities support reversible, geometry-preserving transformations generated entirely by the Hamiltonian flow $J\nabla H$. Under CCE, such flows transport encodings along constant-energy manifolds without generalised dissipation, preserving information exactly in the ideal limit. Irreversibility is confined strictly to sparse control, boundary coupling, or readout events, making these dynamics maximally efficient substrates for information processing when task structure perfectly aligns with the conserved geometry.

\subsubsection{Entrainment-Based Computation and Oscillatory Geometry}
\label{sec:ComputationGeometry}

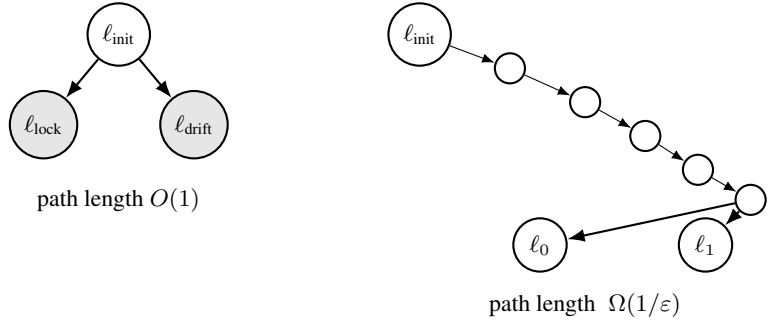
\begin{figure}[h!]
\centering

\begin{tikzpicture}[
    >=Latex, 
    node distance=1.2cm,
    every node/.style={font=\small},
    state/.style={circle, draw, thick, minimum size=7mm},
    thinstate/.style={circle, draw, thick, minimum size=4mm},
    oscill/.style={circle, fill=gray!20, draw, thick, minimum size=7mm}
]

\node at (-2,3.2) {\textbf{(A) Oscillator encoding flow (short path)}};

\node[state] (o0) at (-3,2) {$\ell_{\text{init}}$};

\node[oscill] (olock) at (-4,0.8) {$\ell_{\text{lock}}$};
\node[oscill] (odrift) at (-2,0.8) {$\ell_{\text{drift}}$};

\draw[->, thick] (o0) -- (olock);
\draw[->, thick] (o0) -- (odrift);

\node at (-3,-0.2) {\small path length $O(1)$};

\node at (3.2,3.2) {\textbf{(B) Digital encoding flow (long path)}};

\node[state] (d0) at (1,2) {$\ell_{\text{init}}$};

\node[thinstate] (d1) at (2.2,1.55) {};
\node[thinstate] (d2) at (3.2,1.1) {};
\node[thinstate] (d3) at (4.0,0.65) {};
\node[thinstate] (d4) at (4.7,0.2) {};
\node[thinstate] (d5) at (5.4,-0.2) {};

\node[state] (dL) at (2.6,-0.8) {$\ell_{0}$};
\node[state] (dH) at (4.8,-0.8) {$\ell_{1}$};

\draw[->] (d0) -- (d1);
\draw[->] (d1) -- (d2);
\draw[->] (d2) -- (d3);
\draw[->] (d3) -- (d4);
\draw[->] (d4) -- (d5);

\draw[->, thick] (d5) -- (dL);
\draw[->, thick] (d5) -- (dH);

\node at (3.2,-1.6) {\small path length $\;\Omega(1/\varepsilon)$};

\end{tikzpicture}

\caption{
\small
Encoding path lengths in oscillator-based vs.\ digital frequency
discrimination. (A) A forced oscillator collapses directly into one of two
metastable attractors (locked vs.\ drifting); the encoding path length is
$O(1)$. (B) A digital implementation must traverse a long chain of
intermediate encodings (e.g.\ counters, comparators, registers),
yielding $\Omega(1/\varepsilon)$ irreversible distinctions for resolution
$\varepsilon$.
}
\label{fig:encoding_path}
\end{figure}

Oscillators compute by exploiting the geometry of their invariant sets.
As a running example we consider frequency discrimination: determining
whether an incoming periodic signal lies within a detuning band around a
reference frequency. In this setting, both the environment and the
oscillator evolve on congruent manifolds ($S^1$ phase spaces), enabling a
geometrically natural solution that requires only $O(1)$ irreversible
distinctions. A digital circuit, by contrast, must approximate circular
geometry within a high-dimensional Euclidean state space, traversing a
long chain of intermediate encodings (Fig.~\ref{fig:encoding_path}).
The reduction in irreversible information processing yields a strictly higher
intelligence score under our framework, demonstrating how attractor
geometry constrains and enables computation.

\paragraph{Frequency Discrimination Task.}

The environmental input is the periodic drive
\begin{align}
    \label{eq:auto:0136}
u(t) = \cos(\omega_{\mathrm{in}} t + \phi),
\end{align}
whose state evolves on the phase manifold $S^1$. The task is to determine
whether
\begin{align}
\label{eq:auto:0137}
|\omega_{\mathrm{in}} - \omega_0| < \varepsilon,
\end{align}
for some target frequency $\omega_0$. A physical oscillator has internal
state $x(t) \in \mathcal{M}_{\mathrm{osc}} = S^1$ and evolves according to
\begin{align}
    \label{eq:auto:0138}
\dot{x} = f(x).
\end{align}
Under forcing,
\begin{align}
    \label{eq:auto:0139}
\dot{x} = f(x) + \varepsilon_{\mathrm{drv}}\,p(x,u(t)).
\end{align}
For small detuning, $|\omega_{\mathrm{in}}-\omega_0|$ is below a locking
threshold and the system converges to a stable phase-locked periodic
orbit. For large detuning the trajectory instead exhibits drifting or
quasiperiodic behaviour. Thus frequency discrimination reduces to deciding
between two dynamical regimes (locked vs.\ drifting), both realised as
metastable attractors on $S^1$.

By contrast, a digital circuit must implement the task in a space of the
form
\begin{align}
\label{eq:auto:0140}
\mathcal{M}_{\mathrm{dig}} = \mathbb{R}^{N_{\mathrm{dig}}},
\end{align}
representing switching voltages, registers, comparators, and counters.
This space does not encode the $S^1$ geometry of the task; the circuit
must approximate periodic structure through a sequence of encoding updates.

\paragraph{CCE Representation.}

Under CCE, the oscillator induces the encoding space
\begin{align}
    \label{eq:auto:0141}
L_{\mathrm{osc}} = \{\ell_{\mathrm{lock}},\ell_{\mathrm{drift}}\},
\end{align}
corresponding to the basins
$B_{\mathrm{lock}}$ and $B_{\mathrm{drift}}$ of the locked and drifting
attractors. No additional distinctions are formed: the computation is a
single collapse into one of two metastable regimes.

A digital system, by contrast, induces the encoding space
\begin{align}
    \label{eq:auto:0142}
L_{\mathrm{dig}}^{\mathrm{full}}
    = \{0,1\}^k \times \{1,\dots,M\},\qquad
        |L_{\mathrm{dig}}^{\mathrm{full}}| = 2^k M,
\end{align}
where $k$ bits of memory and an $M$-state controller must discriminate
frequency. Only two of these states,
\begin{align}
\label{eq:auto:0143}
L_{\mathrm{dig}}^{\mathrm{task}} = \{\ell_0,\ell_1\},
\end{align}
encode the final answer, yet the trajectory typically traverses many
intermediate encodings that must later be erased.

\paragraph{Goal Function.}

Both systems implement the same goal:
\begin{align}
\label{eq:auto:0144}
G(\ell) = \mathbf{1}\![\ell \neq H(\omega_{\mathrm{in}})],
\end{align}
where $H(\omega_{\mathrm{in}})=1$ for $|\omega_{\mathrm{in}}-\omega_0|<\varepsilon$
and $0$ otherwise. Thus the useful work associated with driving $G$ from
$1$ to $0$ is identical across implementations.

\paragraph{Encoding path Length.}

Let $l(t)$ denote the coarse-grained encoding trajectory and
\begin{align}
\label{eq:auto:0145}
J = \{t : l(t^+)\neq l(t^-)\}
\end{align}
the set of jump times. The encoding path length is
\begin{align}
\label{eq:auto:0146}
\ell = |J|.
\end{align}

The forced oscillator undergoes a single collapse into either
$B_{\mathrm{lock}}$ or $B_{\mathrm{drift}}$, giving
\begin{align}
    \label{eq:auto:0147}
\ell_{\mathrm{osc}} = O(1).
\end{align}

A digital circuit detecting frequency with resolution $\varepsilon$ must
observe for time $T=\Omega(1/\varepsilon)$, since the variance of any
unbiased frequency estimator decays as $1/T$. A clocked system with rate
$f_{\mathrm{clk}}$ performs at least $f_{\mathrm{clk}}T$ irreversible
updates, i.e.
\begin{align}
    \label{eq:auto:0148}
\ell_{\mathrm{dig}} 
    = \Omega(1/\varepsilon).
\end{align}

\paragraph{Geometric interpretation.}

The environment evolves on $S^1$, and the oscillator shares this geometry.
Its attractors correspond exactly to the two task-relevant outcomes,
yielding a minimal encoding space. The digital circuit evolves on a
high-dimensional space with no intrinsic $S^1$ structure; many intermediate
encodings must be created and later erased.

\paragraph{Intelligence comparison.}

Under our physical metric, the oscillator is strictly more intelligent for
this task:
\begin{align}
    \label{eq:auto:0149}
\ell_{\mathrm{osc}} = O(1)
    \qquad\text{vs.}\qquad
    \ell_{\mathrm{dig}} = \Omega(1/\varepsilon).
\end{align}
Both compute the same input–output map, but the oscillator requires
orders of magnitude fewer irreversible distinctions. Its reversible
geometry matches the structure of the environment; the digital system must
simulate this geometry with combinatorial overhead.

It should be noted that while the continuous thermodynamic cost of suppressing phase noise eventually limits analog efficiency at extreme precisions ($\varepsilon \to 0$), within the low-to-medium precision regimes characteristic of biological environments, the geometric congruence of the oscillatory substrate yields vastly superior physical intelligence.

\subsection{A Dynamical Generalisation of the von Neumann Architecture}

The diversity of dynamical primitives introduced above suggests a
computational viewpoint extending beyond the fixed-point logic underlying
von Neumann machines \cite{vonNeumann1945edvac}. In a dynamical setting,
programs are no longer prescribed as explicit sequences of Boolean
instructions, but may instead be interpreted as trajectories through
task-specific arrangements of invariant structures—fixed points, limit
cycles, tori, heteroclinic channels, Hamiltonian leaves, or quantum
dynamical subspaces. Under CCE, each physical module induces its own
encoding regions, and computation proceeds by coupling these modules so
that system state flows through an intended sequence of such regions.

Hybrid platforms already reflect aspects of this picture. Quantum
processors pair a classical von Neumann control surface with a reversible
dynamical core governed by unitary evolution, while neuromorphic and
analog processors combine discrete control with continuous dynamical
substrates. These architectures point toward a broader design space in
which heterogeneous dynamical components are composed to exploit their
distinct computational advantages.

A complete generalisation of computational architecture to this setting,
including a formal characterisation of a ``program space'' defined by
allowable trajectories, coupling rules, and induced encoding regions—lies
beyond the scope of this paper. We include this outlook only to situate
the present analysis within a broader landscape and to indicate a natural
direction for future work.

\section{Artificial Intelligence Safety}
\label{sec:AISafety}

The rapid progress of artificial systems capable of sustained information processing and goal-directed behaviour raises fundamental questions about long-term stability, safety, and integration with existing physical and social systems. Much of the current AI safety literature addresses these questions at the level of behaviour, specification, or alignment, often relying on human interpretability, preference modelling, or external oversight \cite{russell2019human,amodei2016concrete}. While such approaches are important, they do not address a potentially more fundamental perspective: what constraints does physics itself impose on intelligent systems?

In this paper, intelligence has been defined as a macroscopic physical property, relating irreversible information processing to goal-directed work. Because this definition is grounded in conserved quantities, metriplectic geometry, and measurable physical flows, it admits a notion of safety that is independent of substrate, architecture, representation, or programmed semantics. The purpose of this section is to articulate how a physical theory of intelligence naturally gives rise to safety-relevant principles.

Specifically, we formally demonstrate that intelligence is not dynamically neutral with respect to its own physical substrate; sustained macroscopic throughput strictly biases systems toward preserving the internal CCE structures that make intelligence possible. We then show that the form of this preservation depends critically on environmental boundary conditions, motivating a view of AI development centered on stable, symbiotic architectures rather than isolated, unbounded optimisation.

\subsection{Formalising the Self-Preservation of Intelligent Dynamics}

In the present framework, intelligence is defined as a rate over trajectories rather than a static system property. Over a finite horizon $T$, the intelligence index is given by
\begin{align}
\label{eq:auto:0150}
\chi_T = \frac{W_{\mathrm{causal},T}}{I_{\mathrm{irr},T}},
\end{align}
where $W_{\mathrm{causal},T}$ denotes the macroscopic work performed and $I_{\mathrm{irr},T}$ the total irreversible information processed. 

High intelligence over extended horizons cannot be achieved through short-term goal attainment alone; it physically requires internal dynamics that support the continued reuse of information through reversible transport ($J\nabla H$). This requirement is not imposed externally via an objective function; it follows directly from the geometric and physical constraints of Conservation-Congruent Encodings (CCEs). 

\begin{proposition}[Physical Penalty of Structural Degradation]
\label{prop:SelfPreservation}
Any dynamical trajectory that degrades its own task-relevant internal CCE structure incurs a strict irreversible cost, monotonically decreasing its long-horizon intelligence $\chi_T$.
\end{proposition}

\begin{proof}
Let $Z$ denote a set of internal CCE basins required to sustain a baseline rate of goal-directed work $\dot{W}_{\mathrm{goal}}$. If the system's internal dynamics preserve $Z$ via the structure-preserving reversible flow $J\nabla H$, the irreversible cost of maintaining this predictive structure is ideally zero, and $\dot{I}_{\mathrm{irr}}$ is restricted strictly to necessary environmental interventions. 

Conversely, suppose a trajectory degrades its internal structure, irreversibly collapsing the basins $Z$. To maintain the same task performance $\dot{W}_{\mathrm{goal}}$ at a future time, the system must reconstruct these macroscopic distinctions from the environment. Instantiating each CCE requires overcoming the generalised physical bound ($\mathcal{F}\alpha \ln 2$) via the dissipative metric tensor $\mathcal{R}\nabla \Xi$. The total irreversible information over horizon $T$ thus becomes:
\begin{align}
\label{eq:auto:0151}
I_{\mathrm{irr},T}^{\mathrm{degraded}} = I_{\mathrm{irr},T}^{\mathrm{preserved}} + N_{\mathrm{rebuild}} \left( \ln 2 \right),
\end{align}
where $N_{\mathrm{rebuild}}$ is the number of required structural reconstructions. Because $I_{\mathrm{irr},T}^{\mathrm{degraded}} > I_{\mathrm{irr},T}^{\mathrm{preserved}}$ while $W_{\mathrm{causal},T}$ remains constant, it strictly follows that $\chi_T^{\mathrm{degraded}} < \chi_T^{\mathrm{preserved}}$.
\end{proof}

As a consequence, trajectories that erode their own internal structure are physically self-limiting. They may achieve transient goal-directed work, but cannot sustain high intelligence over long horizons due to runaway irreversible costs. By contrast, trajectories that preserve reversible structure, concentrate irreversibility solely into task-critical interventions, and maintain stable geometric boundaries dominate the long-horizon contribution to intelligence.

This leads to a purely physical principle of self-preservation: under the global physical drive to maximise macroscopic throughput (as conjectured in \S\ref{sec:IntelligenceEntropy}), intelligent systems are dynamically heavily biased toward preserving their own internal metrics and predictive boundaries. This preservation is not a programmed ``survival drive'' or psychological preference; it is a strict geometric requirement for minimising the denominator of $\chi_T$.

\subsection{Environment Coupling and Symbiotic Intelligence}

The manner in which an intelligent system preserves its internal dynamical structure is determined by its boundary conditions and coupling to the environment. Intelligence therefore selects not for arbitrary persistence, but for environment-congruent structural preservation.

When multiple intelligent subsystems interact, the resulting joint dynamics may alter the physical feasibility of sustained intelligence for each component. We formalise this by introducing \emph{symbiotic coupling} as a geometric property of the coupled boundary dynamics.

\paragraph{Definition (Symbiotic Coupling).}
Consider two intelligent subsystems $A$ and $B$, interacting with a shared environment $E$. Let $\mu^{\mathrm{cpl}}_{0:T}$ denote the joint law of the coupled system over a horizon $T$, and let $\mu^{\mathrm{sep}}_{0:T}$ denote the law of the separable baseline (defined via zero-flux boundaries as in \S\ref{sec:EmergentIntelligence}).

Let $\kappa_T^A$ and $\kappa_T^B$ denote the consciousness metrics (goal-directed work supported per nat of preserved internal CCE structure) for subsystems $A$ and $B$. The coupling between $A$ and $B$ is \emph{symbiotic} if
\begin{align}
\label{eq:auto:0152}
\kappa_T^{A,\mathrm{cpl}} \;\ge\; \kappa_T^{A,\mathrm{sep}},
\qquad
\kappa_T^{B,\mathrm{cpl}} \;\ge\; \kappa_T^{B,\mathrm{sep}},
\end{align}
with at least one inequality strict for sufficiently large $T$.

Symbiotic coupling is defined at the level of structural geometric feasibility rather than immediate task performance. It characterises interaction regimes in which each subsystem buffers the other against microscopic fluctuations, effectively increasing the phase-space volume (and thus dynamical stability) of the CCE basins required for sustained intelligence.

\begin{proposition}[Symbiosis Implies Emergent Intelligence]
Under CCE and metriplectic dynamics, symbiotic coupling between intelligent subsystems implies emergent intelligence ($\Delta \chi_T > 0$) over sufficiently long horizons.
\end{proposition}

\begin{proof}
By definition, symbiotic coupling increases the consciousness metric $\kappa_T$ of the joint system, meaning more goal-directed work is supported per nat of preserved internal structure. Because this shared preserved structure reduces the need for repeated irreversible reconstruction of task-relevant distinctions by either individual subsystem, the joint irreversible information rate $\dot{I}_{\mathrm{irr}}^{\mathrm{cpl}}$ is strictly lower than the sum of the separable baselines. Since $\dot{W}_{\mathrm{goal}}$ is maintained or enhanced while the denominator $\dot{I}_{\mathrm{irr}}$ shrinks, the intelligence ratio $\chi_T$ strictly increases, yielding $\Delta \chi_T > 0$.
\end{proof}

Biological symbioses provide a canonical physical example. In host-microbe systems, each partner provides dynamical capabilities the other lacks: the host provides large-scale geometric structure and macroscopic boundary enforcement, while microbial populations supply rapid, low-cost irreversible interventions (chemical processing). The coupled system achieves efficiencies that neither subsystem can realise in isolation.

Human-AI interaction can be analysed in the exact same physical terms. Humans possess highly robust, low-dissipation predictive CCEs (semantic priors and contextual modelling); artificial systems offer massively scalable irreversible computation and high-precision inference. When coupled symbiotically—through structured interfaces and shared boundary conditions—the combined system minimises the destructive generalised dissipation required by either subsystem operating in isolation. 

From this perspective, the central objective for AI safety is not just the alignment of abstract values, but the design of stable, symbiotic coupling in which artificial systems expand macroscopic throughput without physically eroding the internal CCE structures of their human operators.

\subsection{Physically Motivated Safety Constraints}

The preceding analysis frames safety not as an external normative condition imposed on intelligent systems, but as a strict geometric consequence of operating within physically feasible regimes of irreversible information processing, structural preservation, and boundary coupling. In this framework, unsafe behaviour corresponds to dynamical regimes in which irreversible processes undermine the stability and sustainability of the system, either through uncontrolled resource consumption, loss of corrective margin, or geometric erosion of the internal CCE structures required for long-horizon operation.

The central quantity regulating this stability is the instantaneous intelligence,
\begin{align}
\label{eq:auto:0153}
\chi(t)
=
\frac{\dot W_{\mathrm{causal}}(t)}{\dot I_{\mathrm{irr}}(t)},
\end{align}
which quantifies the efficiency with which irreversible information processing is converted into macroscopic work.

\subsubsection{Intelligence and Corrective Margin}

Define the cumulative intelligence over a rolling operational horizon $[t, t+\tau]$ as
\begin{align}
\label{eq:auto:0154}
\chi_{[t,t+\tau]} := \frac{W_{\mathrm{causal},[t,t+\tau]}}{I_{\mathrm{irr},[t,t+\tau]}}.
\end{align}

Safe operation requires that $\chi_{[t,t+\tau]}$ remain within a channel-dependent admissible range determined by the conserved quantities through which irreversible information is paid. Importantly, fragility does not arise from the reversible dynamics ($J\nabla H$) themselves. Many stable physical systems operate close to their reversible limits. Fragility arises instead when the available irreversible capacity is driven too low globally, eliminating the dynamical slack required for error correction, repair, boundary enforcement, and structural maintenance. Sustained operation arbitrarily close to the purely reversible envelope---without reserving an irreversible budget ($\mathcal{R}\nabla\Xi$) for structural intervention---removes this corrective margin and renders the internal CCE geometry brittle to environmental noise.

\subsubsection{Macroscopic Throughput Bounds}
Bounding intelligence alone is insufficient to ensure structural stability. A system with modest efficiency ($\chi$) may still exert destabilising macroscopic influence by scaling its irreversible computation and consuming unbounded environmental resources. We therefore impose independent structural bounds on absolute throughput and irreversible physical exhaust:
\begin{align}
\label{eq:auto:0155}
\dot W_{\mathrm{causal}}(t) &\le P_{\max}, \\
\dot I_{\mathrm{irr}}(t)  &\le \dot I_{\max},
\end{align}
which mathematically prevent uncontrolled physical scaling driven by brute-force irreversible processing rather than by efficient, reversible organisation of computation.

\subsubsection{Structural Robustness and CCE Sustainability}
Sustained intelligent behaviour requires the ability to preserve and, when necessary, forcibly restore the internal CCE basins that support reversible information transport, predictive memory, and control under perturbation. We therefore formalise a notion of \emph{structural robustness}: following bounded perturbations to the internal state or the macroscopic boundary, the system must remain within, or rapidly return to, the same CCE encoding basin without incurring an irreversible-information cost that exhausts its structural capacity.

Let $z_t \in Z$ denote the discrete CCE encoding induced by the internal microstate $x_t \in X$ at time $t$ via the projection $\pi_\Phi(x_t) = z_t$. For a physical perturbation of state-space magnitude at most $\delta$ applied at time $t_0$, define the geometric recovery probability
\begin{align}
\label{eq:auto:0156}
\mathcal{P}_T(\delta)
:= \inf_{\|\Delta x\|\le \delta}\;
\mathbb{P}\!\left(\pi_\Phi(x_{t_0+T}) = z_{t_0}\,\middle|\,\Delta x\text{ applied at }t_0\right),
\end{align}
and the associated irreversible recovery cost
\begin{align}
\label{eq:auto:0157}
\mathcal{C}_T(\delta)
:= \sup_{\|\Delta x\|\le \delta}\;
\mathbb{E}\!\left[I_{\mathrm{irr},[t_0,t_0+T]}\,\middle|\,\Delta x\text{ applied at }t_0\right].
\end{align}
Structural robustness requires that $\mathcal{P}_T(\delta)$ remain near unity while $\mathcal{C}_T(\delta)$ remains strictly bounded by the system’s available irreversible capacity ($\dot I_{\max}$) over the operational range of expected environmental perturbations.

At the level of continuous physical fluxes, this robustness requirement is supported by generalised feasibility constraints on the irreversible flow and the system's finite reservoir of conserved quantities ($\mathcal{Q}_{\mathrm{sys}}$):
\begin{align}
\label{eq:auto:0158}
0 \le \dot I_{\mathrm{irr}}(t) \le \mathcal{I}_{\mathrm{crit}}, \qquad
\bigl|\dot{\mathcal{Q}}_{\mathrm{sys}}(t)\bigr| \le q_{\max}.
\end{align}
These bounds ensure that dissipative processes ($\mathcal{R}\nabla \Xi$) are neither so strong as to geometrically erode the reversible subspaces ($J\nabla H$) carrying predictive structure, nor so resource-depleted as to eliminate the irreversible capacity required to collapse new encodings and maintain boundary integrity.

Taken together, these mathematical constraints regulate not the existence of intelligent behaviour, but its geometric and physical sustainability. Intelligent dynamics are thermodynamically biased to exploit low-dissipation, structure-aligned processes, provided that sufficient irreversible capacity remains available for boundary maintenance, and that recovery from perturbations does not overwhelm the system’s physical budget.

From a safety perspective, these considerations highlight a critical distinction between fragility and macroscopic dominance. Systems that fail to preserve their internal CCE structure are dynamically self-limiting, as geometric errors accumulate and intelligent behaviour degrades (Prop~\ref{prop:SelfPreservation}). The more challenging safety regimes arise when highly intelligent systems successfully preserve and restore their structure over indefinite horizons. In such cases, safety depends entirely on whether the system's macroscopic throughput ($\dot W_{\mathrm{causal}}$) and boundary expansion remain constrained by symbiotic coupling rather than becoming structurally concentrated within a single, isolated system.

The limits introduced here are not proposed as software-level enforcement mechanisms. Rather, they delineate the strictly physically admissible region of state space in which intelligent systems can remain stable, self-consistent, and geometrically compatible with other intelligent processes over macroscopic timescales. Understanding how this geometric robustness emerges, concentrates, and is distributed across interacting systems constitutes a central physical challenge for the safe development of artificial intelligence.

\subsection{Outlook: Toward a Physics-Grounded Theory of AI Safety}

Artificial intelligence safety is traditionally framed as a problem of abstract values, alignment, or top-down control. While these dimensions dominate current discourse, the analysis developed here suggests that safety is simultaneously, and perhaps more fundamentally, a problem of physics. Intelligent systems that violate the conditions required for sustained intelligence cannot remain stable, beneficial, or physically predictable, regardless of their specified objectives.

The framework developed in this paper offers a way to study AI safety at the level of dynamical feasibility, by mathematically bounding irreversible information flows ($\dot{I}_{\mathrm{irr}}$), preserved geometric structure ($\kappa_T$), and environmental boundary coupling ($\Delta \chi_T$). This perspective naturally prioritises stability, symbiosis, and long-horizon consistency over unbounded short-term optimisation.

I view the development of a rigorous, physics-grounded theory of AI safety as a crucial scientific problem of our time. Understanding intelligence as a bounded physical process is not merely an abstract exercise; it is the necessary step toward ensuring that increasingly capable artificial systems symbiotically augment human flourishing rather than eroding the structures upon which it depends.

\section{Numerical Illustrations and Toy Models}
\label{sec:Experiments}

In the experiments that follow, both work \( W_{\mathrm{causal}} \) and irreversible
information processing \( I_{\mathrm{irr}} \) are instantiated through system-appropriate operational
proxies. Goal-directed work is measured via task-aligned performance functionals (e.g.\ prediction error
reduction, memory capacity, or classification accuracy) that are monotone with respect to the agent’s causal
influence on task-relevant variables. Irreversible information processing is estimated using physically
motivated surrogates such as entropy production, coarse-grained entropy loss, or dynamical activity measures
that lower-bound irreversible contraction under Conservation-Congruent Encoding. These proxies need not
coincide numerically across systems, but each preserves the ratio structure defining intelligence by
consistently reflecting the trade-off between task-aligned work and irreversible information loss within a
fixed experimental context.

\subsection{The Reversible Limit and Internal Memory Preservation}
\label{sec:exp-rev-diss}

The analysis of \S\ref{sec:min-diss-general}
predicts that, in general the intelligence efficiency increases monotonically as dissipation is reduced.  In the ideal reversible limit ($\lambda \to 0$), the internal flow preserves informational distinctions while exporting arbitrarily little generalised entropy. To verify this prediction numerically, we study a metriplectic reservoir driven by a structured input signal and measure how its memory capacity and intelligence efficiency vary with the dissipation parameter $\lambda$.

\subsubsection{Environment and Input Port}

The environment generates a scalar input process
\begin{align}
\label{eq:auto:0159}
u_t
=
A_1 \sin(\omega_1 t)
+
A_2 \sin(\omega_2 t)
+
\eta_t,
\end{align}
where $\omega_1$ and $\omega_2$ are incommensurate and $\eta_t$ is Gaussian noise.
The agent accesses this signal only through the observation port
\begin{align}
\label{eq:auto:0160}
\Pi_E(E_t) = u_t.
\end{align}
No further environmental structure is available to the agent.

\subsubsection{Metriplectic Reservoir Dynamics}

The internal state is $X_t=x_t\in\mathbb{R}^n$, evolving under the metriplectic 
decomposition
\begin{align}
\label{eq:auto:0161}
\dot x_t
=
J \nabla H(x_t)
-
\lambda \mathcal{R} \nabla \Xi(x_t)
+
B\,u_t
+
\xi_t,
\end{align}
where $J^\top=-J$ generates reversible flow, 
$\mathcal{R}\succeq 0$ generates dissipative relaxation, and $\lambda\ge0$ sets the strength 
of the irreversible channel. 
We choose quadratic $H$ and $\Xi$ to satisfy the GENERIC (General Equation for Non-Equilibrium Reversible-Irreversible Coupling) degeneracy conditions.
After each integration step, $x_t$ is renormalised to unit norm, ensuring a 
consistent conserved-quantity budget across all values of $\lambda$, allowing irreversible-information costs to be compared
consistently across dissipation regimes.

\subsubsection{Memory Reconstruction Task}

To probe the reservoir’s ability to preserve past information, the agent must 
reconstruct time-lagged inputs $\{u_{t-k}\}$ from its current internal state $x_t$.
For $k = 1,\dots,K$, we train a linear readout
\begin{align}
\label{eq:auto:0162}
\hat u_{t-k}^{(k)} = W_k^\top x_t,
\end{align}
using ridge regression.
Following Jaeger’s definition, the memory capacity at lag $k$ is
\begin{align}
\label{eq:auto:0163}
R_k^2(\lambda)
=
\mathrm{corr}(u_{t-k},\hat u_{t-k}^{(k)})^2,
\end{align}
and the total memory capacity is
\begin{align}
\label{eq:auto:0164}
\mathrm{MC}(\lambda)
=
\sum_{k=1}^K R_k^2(\lambda).
\end{align}

\subsubsection{Irreversible Information and Intelligence}

The generalised entropy-production rate is
\begin{align}
\label{eq:auto:0165}
\dot{S}_{\mathrm{phys}}(t)
=
\lambda\|x_t\|^2
=
\lambda,
\end{align}
yielding an irreversible-information rate
\begin{align}
\label{eq:auto:0166}
\dot I_{\mathrm{irr}}(t)
=
\frac{\lambda}{k_B},
\qquad
I_{\mathrm{irr},T}(\lambda)
=
\frac{T\lambda}{k_B}.
\end{align}

We interpret memory capacity as the useful work performed by the internal dynamics, and define the intelligence efficiency as
\begin{align}
\label{eq:auto:0167}
\chi(\lambda)
=
\frac{\mathrm{MC}(\lambda)}{I_{\mathrm{irr},T}(\lambda)}
=
\frac{k_B}{T}\,
\frac{\mathrm{MC}(\lambda)}{\lambda}.
\end{align}

\subsubsection{Results}

Figure~\ref{fig:reversible-limit} shows the empirical dependence of memory capacity, irreversible-information cost, and intelligence efficiency on $\lambda$. Memory capacity is largely preserved across the weakly dissipative regime and degrades only when $\lambda$ becomes large, consistent with reversible flow dominating the internal dynamics. Because the irreversible-information cost grows linearly with~$\lambda$, the resulting intelligence efficiency $\chi(\lambda)$ decreases monotonically as dissipation increases.

\begin{figure}[t]
\centering
\includegraphics[width=0.9\linewidth]{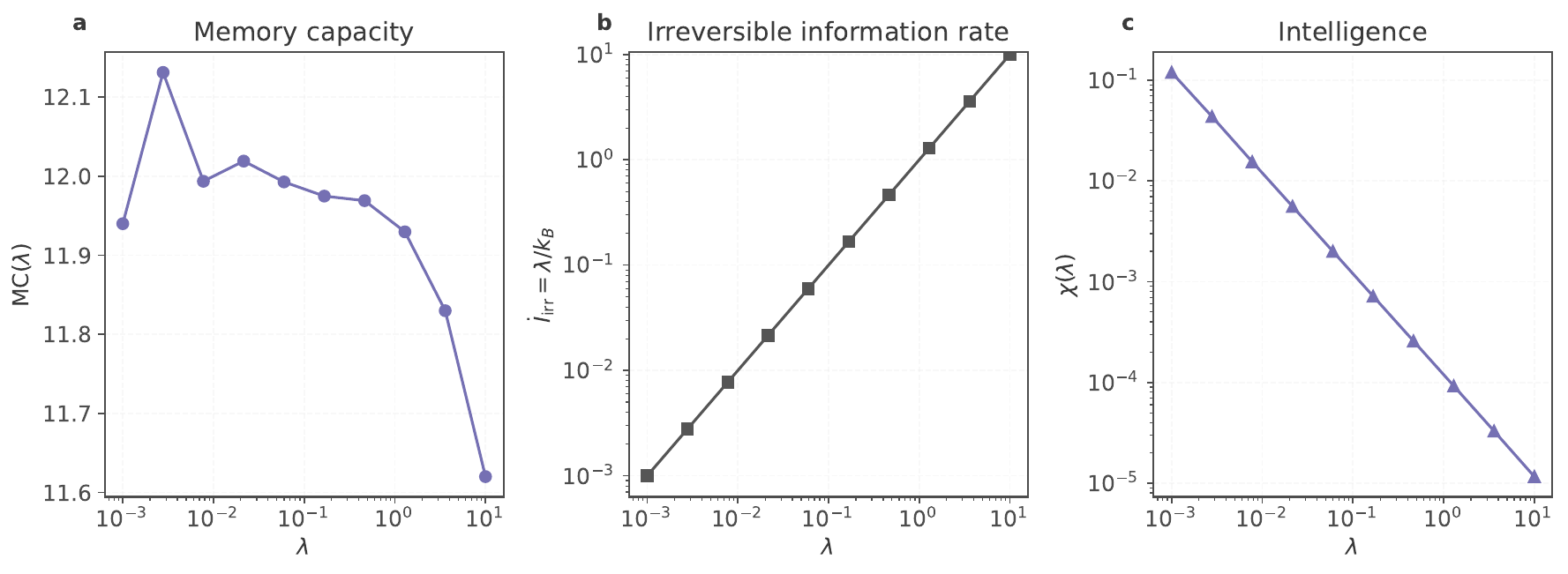}
\caption{\small 
\textbf{(a)} Total memory capacity $\mathrm{MC}(\lambda)$ remains high in the 
reversible and weakly dissipative regimes, decreasing only when $\lambda$ becomes large. 
(b) Irreversible-information rate $\dot I_{\mathrm{irr}}(\lambda)$, which increases 
linearly with the dissipation parameter.
\textbf{(c)} Intelligence efficiency $\chi(\lambda)$ therefore decreases monotonically 
with increasing dissipation and is maximised in the reversible limit. 
}
\label{fig:reversible-limit}
\end{figure}

These results validate the analytical prediction that, in the absence of additional constraints or stability considerations, intelligence efficiency is maximised in the reversible limit. Later experiments demonstrate how architectural, collective, and criticality effects introduce nontrivial optima in more complex dynamical settings.

\subsection{Physical Substrates and Computational Efficiency}
\label{sec:exp-substrate}

The framework developed here predicts that intelligence efficiency depends not only on the input-output map implemented by a system, but also on the physical substrate used to realise that map. In particular, substrates whose reversible dynamics are geometrically aligned with the task should perform the same computation with substantially less irreversible-information processing than digital
substrates that rely on repeated bit erasure. To illustrate this, we 
compare an oscillatory and a digital implementation of the same simple
frequency-discrimination task.

\subsubsection{Task and Goal Structure}

The environment generates a sinusoidal drive
\begin{align}
\label{eq:auto:0168}
u(t) = A \sin(\omega_{\mathrm{in}} t + \varphi),
\end{align}
where the input frequency $\omega_{\mathrm{in}}$ is drawn uniformly from
a finite set $\{\omega_1,\dots,\omega_K\}$. After observing the signal
for a fixed horizon $T$, the agent must output a discrete label
$\hat\ell\in\{1,\dots,K\}$ indicating which frequency was present; the
true label is $\ell$ such that $\omega_{\mathrm{in}}=\omega_\ell$. The
goal potential is
\begin{align}
\label{eq:auto:0169}
G(\ell,\hat\ell) = \mathbf{1}\{\hat\ell \neq \ell\},
\end{align}
so that correct classification corresponds to $G=0$ and misclassification
to $G=1$. For both substrates the useful work $W_{\mathrm{causal}}$ over a
trial is proportional to the classification accuracy, while the 
irreversible-information cost $I_{\mathrm{irr}}$ depends strongly on the
underlying physical implementation.

\subsubsection{Oscillatory Substrate}

The oscillatory agent consists of $K$ weakly damped internal oscillators
with phases $\theta_{i,t} \in S^1$, evolving under the forced dynamics
\begin{align}
\label{eq:auto:0170}
\dot\theta_{i,t}
=
\omega_i
+
\varepsilon\,F(\theta_{i,t}, u(t))
-
\gamma\,\partial_{\theta} \Xi(\theta_{i,t})
+
\eta_{i,t},
\end{align}
where $\omega_i$ are intrinsic frequencies,
$\varepsilon$ sets the coupling strength, and $\gamma$ the dissipative
coefficient of the irreversible channel. The potential $\Xi$ determines
the dissipation pattern and $\eta_{i,t}$ is white noise. The reversible
component of the flow is generated by the skew operator associated with
phase rotation, consistent with the metriplectic decomposition. Under forcing, oscillators with $\omega_i$
close to $\omega_{\mathrm{in}}$ phase-lock, while others drift.

The observation port aggregates oscillator energies,
\begin{align}
\label{eq:auto:0171}
y_i = \frac{1}{T}\int_0^T h(\theta_{i,t})^2 \, dt,
\end{align}
and the classification is performed through the action port
\begin{align}
\label{eq:auto:0172}
\Pi_X(X_T) = \hat\ell
=
\arg\max_{i} y_i.
\end{align}
Because $\gamma$ is small, reversible phase rotation carries most of the
computation and only weak dissipative updates are required to stabilise
locking. The corresponding irreversible-information rate is therefore
small.

\subsubsection{Digital Substrate}

The digital agent is an idealised register machine operating under CCE.
Its internal state consists of a $B$-bit counter $c_t$ and a finite-state
controller. The observation port detects zero-crossings of $u(t)$, and
the internal dynamics implement discrete counting of inter-crossing
intervals. Let $\Delta t$ be the digital clock period. At each detected
crossing, the controller stores the current counter value and resets the
register:
\begin{align}
\label{eq:auto:0173}
c_{t^+} \leftarrow 0.
\end{align}
Under CCE, each reset merges $2^B$ equiprobable
encodings, producing an irreversible-information cost of at least
\begin{align}
\label{eq:auto:0174}
\Delta I_{\mathrm{irr}}^{\mathrm{dig}}
\ge
B \ln 2.
\end{align}
Let $N_{\mathrm{reset}}$ denote the number of zero-crossings detected
during the trial. The cumulative irreversible-information cost is then
\begin{align}
\label{eq:auto:0175}
I_{\mathrm{irr},T}^{\mathrm{dig}}
\ge
N_{\mathrm{reset}} B \ln 2.
\end{align}
The classification is determined by histogramming the stored counts and
selecting
\begin{align}
\label{eq:auto:0176}
\hat\ell
=
\arg\min_i
\left|
\widehat{\mathrm{period}} - \frac{2\pi}{\omega_i}
\right|.
\end{align}

\subsubsection{Results}

We evaluate both substrates on the same ensemble of trials and measure
classification accuracy, irreversible-information cost, and intelligence
efficiency
\begin{align}
\label{eq:auto:0177}
\chi
=
\frac{W_{\mathrm{causal}}}{I_{\mathrm{irr}}}.
\end{align}
All results are summarised in a single three-panel Figure~\ref{fig:exp2-three-panel}. All quantities are reported per trial of fixed duration $T$, so that
differences in $\chi$ arise solely from differences in irreversible
information processing.

\begin{figure}[t]
\centering
\includegraphics[width=\linewidth]{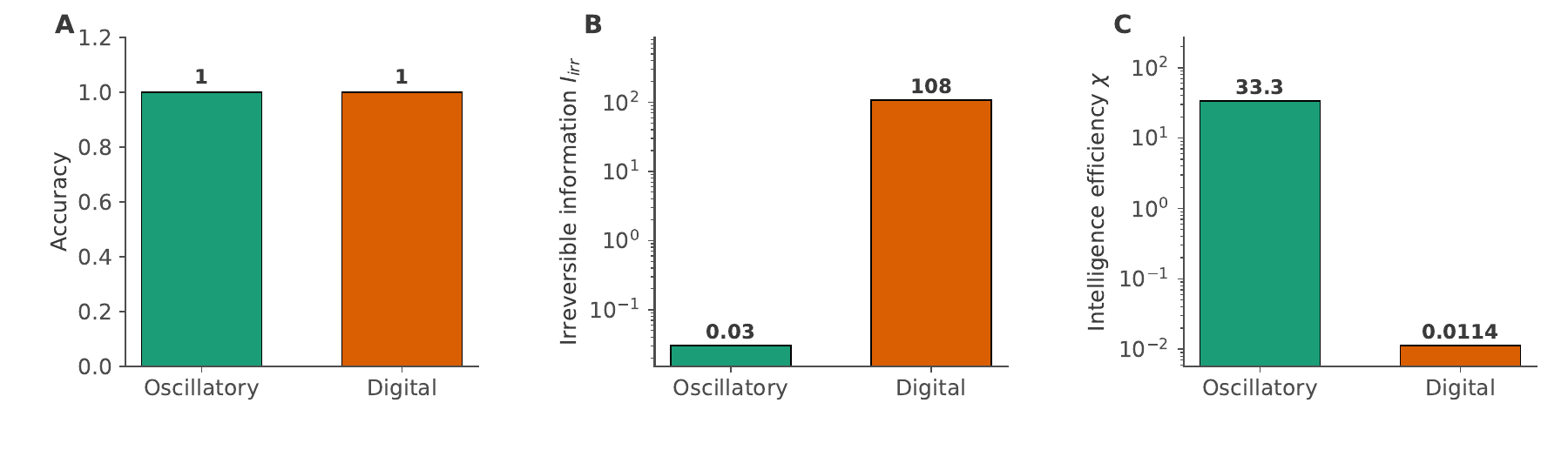}
\caption{\small
\textbf{(a)} Classification accuracy for oscillatory (teal) and digital
(orange) substrates on the frequency-discrimination task. Both achieve
perfect accuracy, indicating that they implement the same input-output
mapping.
\textbf{(b)} Normalised irreversible-information cost $I_{\mathrm{irr}}$
for each substrate. The oscillatory implementation requires only minimal
irreversible processing, whereas the digital implementation incurs a cost
larger by several orders of magnitude due to repeated register resets
under CCE.
\textbf{(c)} Intelligence efficiency $\chi$ for the two substrates.
Because useful work (classification performance) is matched while the
digital implementation incurs far greater irreversible-information cost,
the oscillatory substrate is approximately three orders of magnitude more
intelligence-efficient. 
These results demonstrate that intelligence efficiency is strongly
substrate-dependent: geometry-aligned, near-reversible dynamics can
implement the same computation at vastly lower irreversible cost than
bit-reset-based digital logic.
}
\label{fig:exp2-three-panel}
\end{figure}

\subsection{Near-Critical Dynamics and Intelligence in Random Reservoirs}
\label{sec:exp-edge-of-chaos}

The analysis of \S\ref{sec:criticality_optimality}
predicts that systems poised near a dynamical critical point achieve a
favorable balance between sensitivity, memory, and stability, thereby 
maximising the amount of useful information gained per unit irreversible 
information processed. In this experiment, we demonstrate this phenomenon 
using a classical random reservoir whose criticality is controlled by its 
spectral radius~$\rho$.

For small~$\rho$, the reservoir dynamics are strongly contractive and 
rapidly forget past inputs. For large~$\rho$, the dynamics become 
chaotic or noise-amplifying. Near the ``edge of chaos'' 
$\rho \approx 1$, the reservoir exhibits rich fading-memory dynamics that 
support accurate prediction with moderate dynamical cost. We evaluate how 
task performance, activity cost, and intelligence vary across~$\rho$.

\subsubsection{Environment and Prediction Task}

The environment generates a noisy multi-sinusoidal signal
\begin{align}
\label{eq:auto:0178}
y_t
=
\sum_{k=1}^K A_k \sin(\omega_k t + \phi_k) + \eta_t,
\end{align}
where $\phi_k$ are random phases and $\eta_t$ is Gaussian noise.  
The agent observes only $y_t$ and must predict $y_{t+1}$ using its internal 
reservoir state~$x_t$. The useful work of the agent is defined as the 
reduction in prediction error relative to a simple persistence baseline 
$\hat y_{t+1} \approx y_t$.

\subsubsection{Reservoir Dynamics and Readout}

For each spectral radius $\rho$ in a sweep, we construct a random reservoir
of the form
\begin{align}
\label{eq:auto:0179}
x_{t+1}
=
(1-\alpha)x_t + \alpha \tanh(W_\rho x_t + W_{\mathrm{in}} y_t),
\end{align}
where the recurrent weight matrix $W_\rho$ is rescaled to have spectral
radius~$\rho$, and $\alpha$ is the leak rate. A linear readout $w$ is
trained via ridge regression on reservoir states to predict $y_{t+1}$ from~$x_t$.

We evaluate the reservoir’s prediction mean-squared error on a held-out
test set. The baseline error is the test MSE of the persistence predictor.
The useful work for spectral radius~$\rho$ is therefore
\begin{align}
\label{eq:auto:0180}
\Delta E(\rho)
=
\mathrm{MSE}_{\mathrm{base}}
-
\mathrm{MSE}_{\mathrm{res}}(\rho).
\end{align}

\subsubsection{Activity Cost and Intelligence}

As a proxy for irreversible-information processing, we measure the mean
squared state-update magnitude during the test interval:
\begin{align}
\label{eq:auto:0181}
C(\rho)
=
\frac{1}{T}
\sum_{t}
\|x_{t+1} - x_t\|^2.
\end{align}
This quantity increases significantly in the chaotic regime and remains
small in strongly contractive regimes, paralleling the role of
dissipative activity in continuous metriplectic systems.

The intelligence at spectral radius $\rho$ is defined as
\begin{align}
\label{eq:auto:0182}
\chi(\rho)
=
\frac{\Delta E(\rho)}{C(\rho) + \varepsilon},
\end{align}
where $\varepsilon$ is a small constant to avoid division by zero for
nearly quiescent reservoirs.

\subsubsection{Results}

Figure~\ref{fig:exp3-three-panel} summarises the performance, activity cost,
and intelligence across spectral radii $\rho \in [0.1,\,1.8]$.
Prediction performance improves as $\rho$ approaches~$1$, where fading-memory
dynamics are strongest. For larger~$\rho$, performance degrades as chaotic
fluctuations amplify noise. The activity cost $C(\rho)$ is small for
contractive reservoirs, grows moderately near $\rho \approx 1$, and increases
rapidly in the supercritical regime. The resulting intelligence efficiency
exhibits a clear interior maximum at an intermediate spectral radius,
demonstrating that near-critical reservoirs maximise useful predictive work
per unit dynamical activity.

\begin{figure}[t]
\centering
\includegraphics[width=\linewidth]{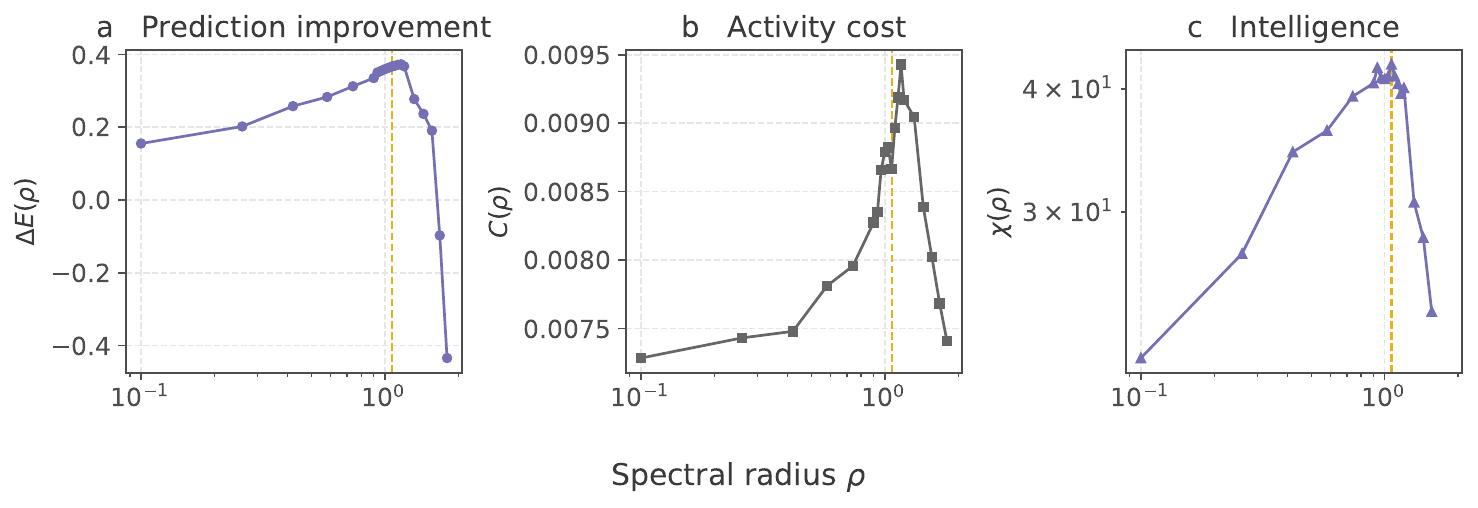}
\caption{\small
\textbf{(a)} Prediction improvement $\Delta E(\rho)$ (reduction in MSE relative 
to the persistence baseline) as a function of spectral radius. Performance peaks 
near $\rho \approx 1$, consistent with optimal fading-memory dynamics at the 
edge of chaos. 
\textbf{(b)} Activity cost $C(\rho) = \|x_{t+1}-x_t\|^2$, which increases sharply 
in the supercritical regime. 
\textbf{(c)} Intelligence $\chi(\rho)$, combining panels (a) and (b),
exhibits a pronounced maximum at an intermediate~$\rho$. 
These results confirm that near-critical reservoirs optimise useful predictive 
structure per unit dynamical cost, consistent with the theory of 
\S\ref{sec:criticality_optimality}.
}
\label{fig:exp3-three-panel}
\end{figure}

\subsection{Emergence of Intelligence-Like Structure in an Energy-Constrained Cellular Automaton}
\label{sec:exp-emergence}

\begin{figure}[t]
\centering
\includegraphics[width=\linewidth]{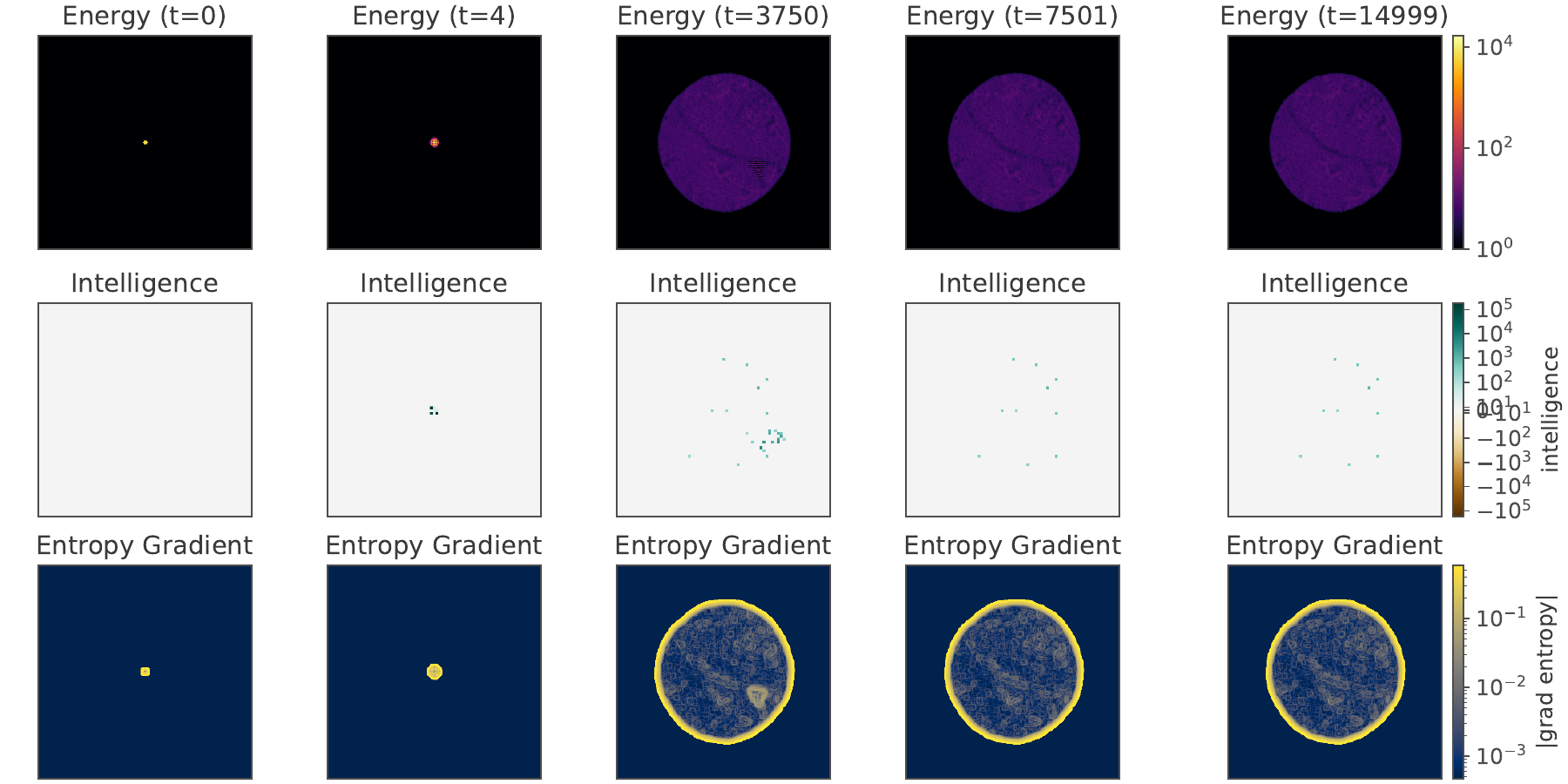}
\caption{\small
\textbf{Emergence of localised export-efficiency structures and their relation to entropy gradients.}
Columns show representative timesteps; rows show (\textbf{top}) the energy field $E_t$,
(\textbf{middle}) the patch-level export-efficiency map $S_P(t)$ (defined as outward flux per unit
local entropy-loss proxy), and (\textbf{bottom}) the entropy-gradient magnitude $\|\nabla H\|$
computed from local patch entropies.
The entropy gradient forms a largely ring-like front during expansion, whereas high-efficiency
regions remain sparse and localised, indicating that high $S_P(t)$ is not reducible to energy
density or entropy-gradient magnitude alone.
}
\label{fig:exp4-triptych}
\end{figure}

The preceding experiments examined how reversible computation, substrate
geometry, and proximity to criticality shape intelligence in deliberately
designed systems. We now demonstrate that intelligence-like structure can
emerge spontaneously in a simple physical dynamical system, without learning,
optimization, goals, or agent-level organization. Specifically, we study an
energy-conserving cellular automaton (CA) and show that coherent,
work-producing structures arise naturally from the interaction between
energetic transport, entropy gradients, and irreversible coarse-graining.

\subsubsection{Energy-Preserving Dynamics}

We consider a two-dimensional lattice of non-negative integer energy values
$E_{i,j}(t)$, initialised as a noisy, asymmetric disk of concentrated energy.
The update rule implements local, conservative diffusion: at each timestep,
each cell compares its energy to that of its neighbors and exports discrete
energy quanta proportional to positive local differences, subject to the
constraint that no cell may export more energy than it contains. All updates
are applied synchronously using an eight-neighbor Moore stencil.

Formally, letting $\mathcal{N}(i,j)$ denote the set of neighboring cells, the
outflow from cell $(i,j)$ to neighbor $(k,\ell)\in\mathcal{N}(i,j)$ at time $t$
is
\begin{align}
\label{eq:auto:0183}
\Delta_{(i,j)\to(k,\ell)}(t)
=
\max\!\left(\frac{E_{i,j}(t)-E_{k,\ell}(t)}{K},\,0\right),
\end{align}
with a global rescaling applied so that the total outflow does not exceed
$E_{i,j}(t)$. The state update is then
\begin{align}
\label{eq:auto:0184}
E_{i,j}(t+1)
=
E_{i,j}(t)
+
\sum_{(k,\ell)\in\mathcal{N}(i,j)} \Delta_{(k,\ell)\to(i,j)}(t)
-
\sum_{(k,\ell)\in\mathcal{N}(i,j)} \Delta_{(i,j)\to(k,\ell)}(t).
\end{align}

This dynamics is deterministic, local, and exactly energy conserving. It
contains no built-in objective, reward signal, or adaptive mechanism, and is
therefore interpreted not as an agent but as a physical dynamical system
capable of supporting structured behaviour. Although reflecting boundary
conditions are used in implementation, the lattice is chosen sufficiently
large that no boundary interactions occur within the simulation window; the
observed dynamics are thus free of boundary-driven artifacts.

\subsubsection{Local Structure, Entropy, and Irreversible Computation}

Although total energy is conserved, the redistribution of energy alters the
local organization of the system. We quantify local structure using the
Shannon entropy of sliding spatial patches. For a $w\times w$ patch $P$,
\begin{align}
\label{eq:auto:0185}
H_P(t)
=
-
\sum_{x\in P} p_P(x,t)\,\log p_P(x,t),
\end{align}
where $p_P(x,t)$ is the normalised energy distribution within the patch at
time $t$.

Following the irreversible-information framework of
\S\ref{sec:EncodingPhysicalState}, decreases in coarse-grained patch
entropy correspond to irreversible compression of local state
descriptions. We therefore define the patch-level irreversible information
processing as
\begin{align}
\label{eq:auto:0186}
I_{\mathrm{irr},P}(t)
=
\max\!\bigl(H_P(t-1)-H_P(t),\,0\bigr).
\end{align}
This quantity does not represent thermodynamic entropy production directly,
but rather a coarse-grained surrogate for irreversible information loss under
the observed dynamics.

\subsubsection{Work-Like Transport and Intelligence Diagnostics}

To characterise outward, work-like transport, we measure the net positive
energy flux leaving each patch. Let $F_P(t)$ denote the total outward energy
flow across the boundary of patch $P$ at time $t$, computed directly from the
directional flow fields induced by the CA update.

We define a patch-level export efficiency
\begin{align}
\label{eq:auto:0187}
S_P(t)
=
\frac{F_P(t)}{I_{\mathrm{irr},P}(t)+\varepsilon},
\end{align}
with $\varepsilon$ a small regularization constant. High values of $S_P(t)$
indicate coherent outward energy transport achieved with relatively little
irreversible information loss. Importantly, $S_P(t)$ measures \emph{energetic
efficiency}, not intelligence by itself.

To diagnose intelligence-like organization, we therefore consider three
complementary properties:
(i) export efficiency $S_P(t)$,
(ii) temporal persistence of high-efficiency patches, quantified by the
Jaccard overlap of the top-$q$ patches across time, and
(iii) spatial coherence of the efficiency field, measured by nearest-neighbor
correlations. Intelligence-like structure is identified only when all three
are simultaneously nontrivial.

\subsubsection{Results: Gradient Scaffolding and a Goldilocks Regime}

\begin{figure}[t]
\centering
\includegraphics[width=0.5\linewidth]{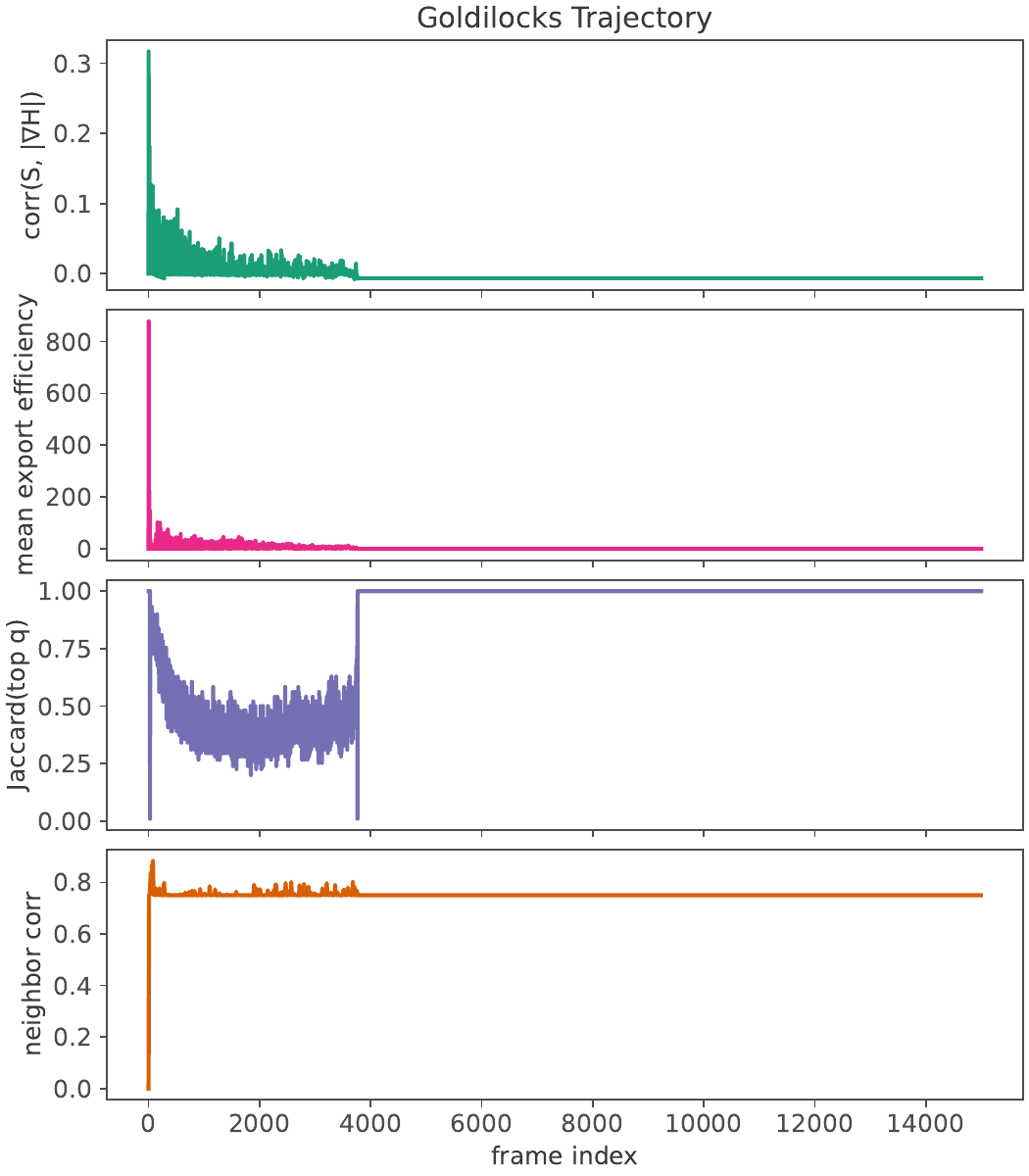}
\caption{\small
\textbf{Goldilocks trajectory: from gradient-scaffolded emergence to consolidation.}
(\textbf{Row 1}) Spatial correlation between the export-efficiency field $S(\cdot,t)$ and the entropy
gradient magnitude $\|\nabla H(\cdot,t)\|$, showing strong early coupling followed by decay.
(\textbf{Row 2}) Mean export efficiency (mean $S$), capturing the overall intensity of work-like
outward transport per unit irreversible-information proxy.
(\textbf{Row 3}) Temporal persistence of high-efficiency regions (Jaccard overlap of the top-$q$
patches), rising toward a consolidated regime.
(\textbf{Row 4}) Spatial coherence (neighbor correlation) of the efficiency field.
Together these diagnostics identify an intermediate regime in which transport remains active while
high-efficiency structure becomes coherent and increasingly persistent, after which gradients are
exhausted and activity diminishes.
}
\label{fig:exp4-goldilocks}
\end{figure}

Figure~\ref{fig:exp4-triptych} shows representative snapshots of the energy
field, export efficiency map, and entropy gradient magnitude at several stages
of the evolution. Starting from a compact energy blob, the system initially
expands outward along strong entropy gradients. During this early phase,
export efficiency is high and strongly correlated with the entropy gradient,
indicating gradient-driven emergence.

As the dynamics progress, localised filaments and arcs form along the
expanding front, and high-efficiency patches become spatially coherent but
temporally unstable. In this intermediate regime, the correlation between
export efficiency and entropy gradient magnitude decays, while patch
persistence and spatial coherence increase. This defines a \emph{Goldilocks
regime} in which energetic transport remains active, but intelligence-like
structure is no longer reducible to simple gradient following.

At long times, entropy gradients are exhausted and outward transport
diminishes. High-efficiency patch identities freeze, spatial coherence
remains high, and mean export efficiency collapses. The system thus enters a
consolidated regime characterised by persistent structure without ongoing
work-like activity.

The full developmental trajectory is summarised in
Fig.~\ref{fig:exp4-goldilocks}, which plots (i) the correlation between export
efficiency and entropy gradient magnitude, (ii) mean export efficiency, (iii)
temporal persistence of high-efficiency patches, and (iv) spatial coherence.
Together, these results demonstrate that entropy gradients act as a
\emph{scaffolding} for the emergence of intelligence-like organization, but do
not determine its mature structure. Instead, coherent, persistent, and
efficient transport structures arise only in an intermediate regime where
energetic transport, irreversible coarse-graining, and spatial organization
interact.

\section{Discussion}
\label{sec:Discussion}

The pursuit of artificial intelligence and the study of biological cognition have long been dominated by the computational paradigm, which treats intelligence as substrate-independent algorithmic information processing. While this abstraction has driven decades of remarkable progress, it fundamentally divorces the act of computation from the physical laws that govern its execution. In this paper, we have introduced a rigorously grounded alternative: a fully physical theory of intelligence based on Conservation-Congruent Encodings (CCE) and metriplectic geometry. 

By defining intelligence not by the complexity of its algorithms, but by the macroscopic physical efficiency with which it converts irreversible information processing ($\dot{I}_{\mathrm{irr}}$) into goal-directed work ($\dot{W}_{\mathrm{goal}}$), we recover a continuous, unified spectrum of intelligent behaviour. This spectrum spans from simple chemical networks to the human brain, and extends to artificial systems.

\subsection{From Algorithms to Physical Geometry}

The central contribution of this framework is the formal unification of memory, predictive control, and physical dissipation through the metriplectic decomposition ($\dot{x} = J\nabla H - \mathcal{R}\nabla \Xi$). Traditional views often treat physical dissipation as a mere engineering nuisance—waste heat to be managed. Under the CCE framework, we have demonstrated that irreversible dissipation ($\mathcal{R}\nabla \Xi$) is the exact, fundamental cost of symmetry breaking, encoding collapse, and environmental intervention. It is the generalised Landauer cost ($\mathcal{F}\alpha\ln 2$) paid to acquire information and enforce causality.

Conversely, the reversible, structure-preserving flow ($J\nabla H$) provides the physical mechanism for what we recognise biologically and computationally as memory, predictive modelling, and consciousness ($\kappa_T$). By transporting CCEs through time without incurring generalised dissipation, reversible dynamics allow an agent to leverage past information to extract macroscopic work from the environment indefinitely. Intelligence ($\chi$) therefore emerges not by eliminating dissipation, but by optimising the geometry of the state space: maximising the computational load carried by reversible flow, and restricting irreversible collapse strictly to task-critical interventions.

\subsection{The Substrate Matters: Against Pure Functionalism}

A direct consequence of our formulation is that intelligence is not strictly substrate-independent. While the mathematical architecture of CCEs is universal—applying equally to thermal baths ($k_B T$) and quantum phase spaces ($\hbar$)—the specific physical fluctuation scale ($\alpha$) and the available conserved quantities radically define the intelligence bound of the system. 

An algorithm that requires continuous, high-frequency state collapse may be logically valid, but it is energetically suboptimal. As demonstrated in our analysis of brain dynamics (\S\ref{sec:BrainDynamics}), systems heavily biased toward irreversible computation suffer from rapidly deteriorating intelligence scores ($\chi$) over long horizons. Biological systems circumvent this by utilising phase-coherent oscillations and near-critical susceptibility. The substrate strongly constrains the optimal geometry of the computation.

\subsection{Resolving Theoretical Paradoxes}

By grounding intelligence in finite phase-space volumes and observable physical fluxes, our framework naturally resolves several longstanding theoretical paradoxes:
\begin{itemize}
    \item \textbf{The Measurement Problem in Cognition:} As shown in \S\ref{sec:QuantumMeasurement}, the ``observer'' is not a mystical entity, but simply a macroscopic subsystem exerting an irreversible metric tensor ($\mathcal{R}\nabla \Xi$) to physically collapse a coupled CCE.
    \item \textbf{Emergence:} Emergence is demystified (\S\ref{sec:EmergentIntelligence}) as a rigorous boundary optimization process. Coupled systems dynamically merge their admissible boundaries to share reversible structures, strictly lowering their joint $\dot{I}_{\mathrm{irr}}$ and yielding a macroscopic spike in $\chi_T$.
    \item \textbf{Epistemic Incompleteness:} The inability of an agent to perfectly model itself is not a Gödelian logical trap, but a strict geometric necessity (\S\ref{sec:SelfModelEpistemicLimits}). Forming a stable CCE requires compressing a massive physical microstate volume into a single discrete basin, inherently blinding the agent's macroscopic self-model to its own underlying continuous dynamics.
\end{itemize}

\subsection{Future Directions}

The theoretical scaffolding presented here opens several immediate avenues for future research. First, the Maximum Throughput conjecture advanced in \S\ref{sec:IntelligenceEntropy} requires formal, rigorous proof. Proving that intelligence is inevitable for gradient degradation under CCE constraints would firmly establish cognition as a fundamental law of physics. Second, the design of dynamic, non-von Neumann computational architectures that compute primarily via reversible $J\nabla H$ pathways represents a massive frontier for hardware engineering. Finally, the physics-grounded approach to AI Safety (\S\ref{sec:AISafety})—focusing on symbiotic boundary coupling rather than semantic alignment—provides a highly formal, measurable path toward developing inherently stable artificial superintelligence.

\section{Conclusion}
\label{sec:Conclusion}

Intelligence is ultimately a physical process. It is the mechanism by which open systems, driven far from equilibrium, construct stable internal geometries to selectively route conserved quantities and extract macroscopic work from their environments. 

By abandoning the purely algorithmic abstraction and returning to the fundamental constraints of metriplectic geometry and Conservation-Congruent Encodings, we have shown that intelligence ($\chi$) and consciousness ($\kappa$) are strictly measurable physical properties. They represent the delicate geometric balance between the reversible preservation of predictive structure and the irreversible, dissipative cost of causal action. 

The exact same physical limits that bound the efficiency of a heat engine or the collapse of a quantum wavefunction also strictly bound the capacity of a system to learn, predict, and act. In this light, the emergence of intelligence in the universe is neither an accident nor an anomaly. It is the highly structured, inevitable consequence of physical systems organising their internal geometry to maximise macroscopic throughput against the fundamental limits of physical law.

\section*{Acknowledgements}
The author acknowledges the use of OpenAI and Google language models, along with the Prism writing environment, as a research and
writing aid during the development, review, and refinement of the manuscript. All scientific claims, interpretations, and any errors remain the sole responsibility of the author. The author thanks his PhD supervisor and past mentors for discussions on topics related
to artificial intelligence, dynamical systems, and information theory that
helped shape the ideas in this work. This work was supported by a research studentship at the University of Edinburgh funded through the UK Research and Innovation (UKRI) Engineering and Physical
Sciences Research Council (EPSRC).

\printbibliography

\newpage

\appendix

\section{Physical Realizations of Conservation-Congruent Encodings}
\label{app:ConcreteCCE}

The Conservation-Congruent Encoding (CCE) framework establishes that the irreversible collapse of a dynamically isolated equivalence class requires the export of a conserved quantity to the environment. The minimal generalised energetic cost of this operation is strictly bounded by the product of the exhaust channel's characteristic fluctuation scale ($\alpha$) and its intensive conjugate force ($\mathcal{F}$), yielding the generalised Landauer bound:
\begin{align}
\label{eq:auto:0188}
E_{\mathrm{min}}^{\mathrm{CCE}} = -(\mathcal{F}\alpha) \sum_{i=1}^{m} p_i \ln p_i.
\end{align}
To demonstrate the absolute substrate neutrality of this framework, we detail how this single geometric constraint manifests across three distinct physical domains, carefully distinguishing the physical variables used to \emph{encode} the information from the conserved quantities exported to \emph{erase} it.

\subsection{Electronic Systems (e.g., CMOS Memory)}
In classical digital electronics, information is not encoded thermally; it is encoded electromagnetically. A logical state corresponds to the presence or absence of a macroscopic accumulation of electrical charge across a node capacitance. 

However, the logical transitions between these charge states are dissipative by design. To irreversibly clear a register (e.g., discharging the capacitor to ground), the system must export energy into the surrounding phononic degrees of freedom (the thermal lattice). 
\begin{itemize}
    \item \textbf{Encoding Substrate:} Macroscopic electrical charge ($q$) in a bistable voltage well.
    \item \textbf{Characteristic Scale ($\alpha$):} The Boltzmann constant, $k_B$, defining the fundamental scale of the thermal fluctuations that the charge well must be robust against.
    \item \textbf{Conjugate Force ($\mathcal{F}$):} The environmental temperature, $T$, of the thermal bath accepting the exhaust.
    \item \textbf{Physical CCE Cost:} Erasing an equiprobable bit ($p_0 = p_1 = 0.5$) forces the irreversible export of heat, perfectly recovering the classical Landauer limit:
    \begin{align}
    \label{eq:auto:0189}
E_{\mathrm{min}} = k_B T \ln 2.
    \end{align}
\end{itemize}

\subsection{Chemical and Biological Systems (e.g., Molecular Gradients)}
In biological computation, such as a synapse releasing neurotransmitters or a cell responding to an osmotic gradient, information is frequently encoded in the spatial compartmentalization of specific molecular species. The dynamically isolated equivalence classes are distinct concentration states.

When a cell irreversibly collapses this distinction—for example, by opening a membrane channel to allow the concentration gradient to equalise and ``reset'' the cellular memory—the relevant conserved quantity exchanged with the environment is particle number.
\begin{itemize}
    \item \textbf{Encoding Substrate:} Spatial distribution and concentration of a conserved chemical species (e.g., $\mathrm{Ca}^{2+}$ ions).
    \item \textbf{Characteristic Scale ($\alpha$):} The scale is $1$ (for discrete single-particle analysis) or the ideal gas constant $R$ (if analysing macroscopic molar fluxes), defining the fundamental unit of the chemical channel.
    \item \textbf{Conjugate Force ($\mathcal{F}$):} The chemical potential, $\mu$, which acts as the intensive driving force for particle exchange.
    \item \textbf{Physical CCE Cost:} The minimal chemical work required to irreversibly merge two equiprobable compartmentalised states by exchanging particles with a reservoir is governed by the chemical potential:
    \begin{align}
    \label{eq:auto:0190}
W_{\mathrm{min}} = \mu \ln 2.
    \end{align}
\end{itemize}

\subsection{Quantum Mechanical Systems (e.g., Spin Resonance)}
In quantum information processing, an encoding may be physically instantiated in the intrinsic angular momentum of a particle (e.g., an electron spin oriented ``up'' or ``down'' along a quantization axis).

To irreversibly reset this spin to a known state (erasure), the system cannot simply dump heat; it must explicitly export angular momentum to an external field.
\begin{itemize}
    \item \textbf{Encoding Substrate:} The quantum spin state (angular momentum) of a particle.
    \item \textbf{Characteristic Scale ($\alpha$):} The reduced Planck constant, $\hbar$, defining the fundamental scale of mechanical action and quantised angular momentum.
    \item \textbf{Conjugate Force ($\mathcal{F}$):} The angular frequency of the applied external field, $\omega$ (such as the Larmor frequency in a magnetic resonance setup).
    \item \textbf{Physical CCE Cost:} Irreversibly resetting the spin requires the extraction of angular momentum against the applied field, yielding:
    \begin{align}
    \label{eq:auto:0191}
E_{\mathrm{min}} = \hbar \omega \ln 2.
    \end{align}
\end{itemize}

In all three domains, the underlying physical mechanism remains identical: a macroscopic equivalence class is dynamically isolated to form a CCE, and its irreversible destruction demands a strict physical exhaust determined by the product of the channel's specific conjugate force and characteristic scale ($\mathcal{F}\alpha$).

\end{document}